\documentclass{article}

\usepackage{microtype}
\usepackage{graphicx}
\usepackage{subcaption}
\usepackage{booktabs} 

\usepackage{hyperref}

\usepackage[preprint]{arxiv/icml2026}

\usepackage{amsmath}
\usepackage{amssymb}
\usepackage{mathtools}
\usepackage{amsthm}

\usepackage[textsize=tiny]{todonotes}
\usepackage{enumitem}
\usepackage[many]{tcolorbox}    	
\tcbuselibrary{listings, breakable}
\newtcolorbox{promptBox}[1]{
    title=\textbf{#1},
    breakable,
    fonttitle=\bfseries,
    boxrule = 1pt,
    toprule = 3pt, 
    colframe = RoyalBlue,
    enhanced,
    rounded corners,
    arc = 2pt,   
    top=1mm,bottom=1mm,left=1mm,right=1mm    
}

\newtcolorbox{outputbox}[1]{
    title=\textbf{#1},
    breakable,
    fonttitle=\bfseries,
    boxrule = 1pt,
    toprule = 3pt,
    colframe = blue!60!black, 
    enhanced,
    rounded corners,
    arc = 2pt,
    top=1mm,
    bottom=1mm,
    left=1mm,
    right=1mm
}

\usepackage[capitalize,noabbrev]{cleveref}

\theoremstyle{plain}
\newtheorem{theorem}{Theorem}[section]
\newtheorem{proposition}[theorem]{Proposition}
\newtheorem{lemma}[theorem]{Lemma}

\theoremstyle{definition}

\theoremstyle{remark}
\newtheorem{remark}[theorem]{Remark}

\usepackage{multirow}
\newcommand{\algname}{\emph{AutoAdapt}}

\usepackage[dvipsnames, table]{xcolor}
\usepackage[many]{tcolorbox}    	

\usepackage[textsize=tiny]{todonotes}

\usepackage{tcolorbox}
\usepackage{adjustbox}

\usepackage{wrapfig}

\usepackage{rotating} \usepackage{booktabs} \usepackage{array}
\usepackage{caption}
\usepackage{placeins}

\title{AutoAdapt: An Automated Domain Adaptation Framework for LLMs}
\icmltitlerunning{AutoAdapt: An Automated Domain Adaptation Framework for LLMs}

\begin{document}

\twocolumn[
  \icmltitle{AutoAdapt: An Automated Domain Adaptation Framework \\
  for Large Language Models}
    


  \icmlsetsymbol{equal}{*}

  \begin{icmlauthorlist}
    \icmlauthor{Sidharth Sinha}{yyy}
    \icmlauthor{Anson Bastos}{yyy}
    \icmlauthor{Xuchao Zhang}{comp}
    \icmlauthor{Akshay Nambi}{sch}
    \icmlauthor{Chetan Bansal}{comp}
    \icmlauthor{Saravan Rajmohan}{comp}
  \end{icmlauthorlist}

  \icmlaffiliation{yyy}{Microsoft, M365 Research, India}
  \icmlaffiliation{comp}{Microsoft, M365 Research, Redmond}
  \icmlaffiliation{sch}{Microsoft Research, India}

  \icmlcorrespondingauthor{Anson Bastos}{ansonbastos@microsoft.com}
  \icmlcorrespondingauthor{Xuchao Zhang}{xuchaozhang@microsoft.com}
  \icmlcorrespondingauthor{Akshay Nambi}{Akshay.Nambi@microsoft.com}

  \icmlkeywords{Machine Learning, ICML}

  \vskip 0.3in
]



\printAffiliationsAndNotice{}  

\begin{abstract}

Large language models (LLMs) excel in open domains but struggle in specialized settings with limited data and evolving knowledge. Existing domain adaptation practices rely heavily on manual trial-and-error processes, incur significant hyperparameter complexity, and are highly sensitive to data and user preferences, all under the high cost of LLM training. Moreover, the interactions and transferability of hyperparameter choices across models/domains remain poorly understood, making adaptation gains uncertain even with substantial effort.
To solve these challenges,
we present \algname{}, a novel end-to-end automated framework for efficient and reliable LLM domain adaptation. \emph{AutoAdapt} leverages curated knowledge bases from literature and open-source resources to reduce expert intervention. To narrow the search space, we design a novel multi-agent debating system in which proposal and critic agents iteratively interact to align user intent and incorporate data signals and best practices into the planning process. To optimize hyperparameters under tight budgets, we propose \emph{AutoRefine}, a novel LLM-based surrogate that replaces costly black-box search. Across 10 tasks, \emph{AutoAdapt} achieves a \textbf{25\%} average relative accuracy improvement over state-of-the-art Automated Machine Learning baselines with minimal overhead.

\end{abstract}

\section{Introduction}

Large Language Models (LLMs) have rapidly advanced the state of the art in natural language processing, enabling impressive performance across a wide spectrum of tasks and domains. However, despite their broad pre-training, LLMs often fall short when applied to specialized domains where data is scarce, proprietary, or subject to privacy constraints. Further, in scenarios where data evolves with time, pretrained LLMs would not have the updated knowledge.
In such settings, domain adaptation of LLMs is imperative by resorting to techniques like in-context prompting \cite{dong-etal-2024-survey} using retrieval from current data (RAG) \cite{RAG_NEURIPS2020_6b493230} or finetuning the models (SFT) \cite{wang2025parameterefficientfinetuninglargemodels} on new data for optimal performance.
The current process of adapting LLMs to a specific task is largely manual and remains limited to a set of best practices, assuming a one-template-fits-all approach across tasks and models. Further, effective fine tuning relies on expert knowledge because small choices in data, objectives, and training settings can strongly affect the performance.

Domain adaptation for LLMs involves a series of intricate steps: curating domain-specific datasets, selecting appropriate model architectures, configuring prompting or fine-tuning strategies, and iteratively optimizing hyperparameters for optimal performance. Each decision point, ranging from model selection, context engineering, training technique (e.g., supervised fine-tuning (SFT), parameter-efficient adaptation (PeFT) e.g. LoRA/Prefix Tuning, preference alignment e.g. DPO) to the choice of learning rates, batch sizes, and regularization techniques, can significantly impact the final model’s quality, stability and cost. The challenge of exploring the vast search space is compounded in LLMs due to the high time and cost of finetuning these large models. The combinatorial explosion of possible configurations, coupled with the high time and resource utility, means that even experienced teams must engage in extensive trial-and-error, often spending weeks or months to achieve robust, high-quality results.

Recent research \cite{dsagent_guo2024dsagentautomateddatascience, automlagent_trirat2025automlagent, mlcopilot_zhang2024mlcopilotunleashingpowerlarge, autoGPT_zhang2023automlgptautomaticmachinelearning} has explored automating general Machine Learning (ML) pipelines using agent-based frameworks offering partial solutions for model selection and hyperparameter tuning, albeit in the general ML domain. However, the following \textbf{challenges exist}:
1) Existing methods either rely on the pretraining knowledge or in-context learning ability of LLMs, which we show is insufficient for the task of LLM adaptation. Further, retrieval based methods \cite{dsagent_guo2024dsagentautomateddatascience, automlagent_trirat2025automlagent} lack a structured knowledge base of LLM use cases, limiting their usage to only classical ML methods. 
2) These approaches typically rely on exhaustive or heuristic search strategies, which are computationally expensive and ill-suited to the scale and complexity of LLM adaptation. 
3) Existing methods rely on the LLM agents to generate code. However, agents may have limited and outdated pretraining context on LLM adaptation. 
This results in low success rate for end-to-end runs as we show in our experiments. 
4) 
LLM adaptation pipelines often require multiple stages such as pretraining, followed by finetuning, alignment, in-context prompting using RAG techniques and so on. Existing methods fail in planning such a multi-stage pipeline due to a lack of a structured and well defined hierarchical plan. 
5) LLMs incur an exorbitant cost and latency for training and thus standard Hyperparameter Optimization (HPO) techniques employed in existing methods \cite{automlagent_trirat2025automlagent} will hardly scale to LLM adaptation. 
6) Most agents do not consider the data characteristics and explicit adherence to user preferences during the planning stage 
or merely loop on failing execution which degrades efficiency.

To address these challenges, we introduce \algname{}, an automated domain adaptation framework that proposes to solve the adaptation problem as a multi-agent system grounded in a knowledge-driven hierarchical workflow.
We further use Gaussian processes \cite{beckers2021introductiongaussianprocessmodels} for probabilistic function estimation, modeling uncertainty over the adaptation objective function.
Figure \ref{fig:concept_figure} shows the user flow of \algname{}, where the user inputs the task, data, and constraints and the system outputs the domain adapted LLM.

\begin{figure}[t!]
    \centering
    \includegraphics[width=0.95\linewidth]{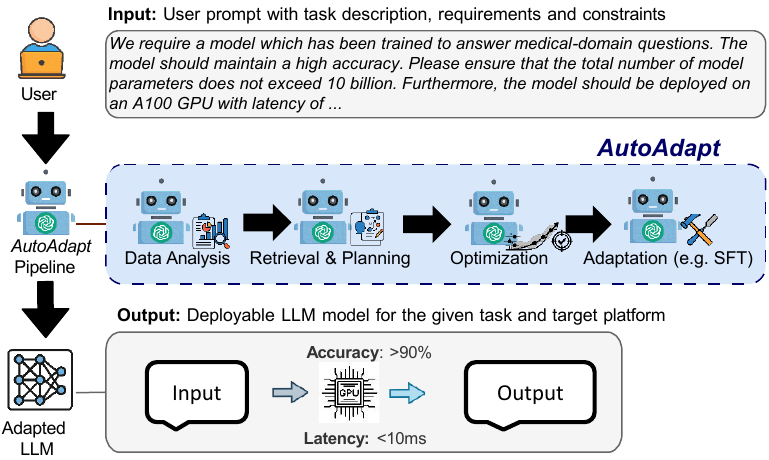}
    \caption{\algname{} takes as input the user instruction and dataset and delivers the domain adapted LLM based on user constraint}
    \label{fig:concept_figure}
\end{figure}

\algname{} leverages a structured directed acyclic graph, called ``Adaptation Configuration Graph" (ACG), that incorporates the LLM adaptation workflow knowledge. This graph contains hierarchical decision nodes (e.g. \texttt{RAG} vs \texttt{SFT}, \texttt{SFT} using \texttt{LoRA} or \texttt{Prefix Tuning} etc.) and parameter nodes (e.g. \texttt{learning rate, epochs, batch size} to explore for \texttt{SFT}).
The framework is composed of a system of proposal and critic agents to  generate, debate, verify and refine executable pipelines. We ground the multi-agent system by integrating best practices from thousands of open-source models and datasets, as well as user-specified requirements and data characteristics.  We further perform online retrieval from web sources and literature. Thus, \algname{} dramatically reduces the need for manual intervention and expert oversight. This produces an adaptation pipeline with narrowed search spaces (e.g. $1e^{-4} \leq \texttt{learning rate} \leq 1e^{-6}$) to optimize.
The framework further incorporates a feedback-driven optimization loop, called \emph{AutoRefine}, to optimize parameters (e.g. exact \texttt{SFT} parameters like $\texttt{learning rate} = 2e^{-5}$ from the narrowed search space). AutoRefine efficiently explores the hyperparameter space using LLM-based oracle that factors performance from prior runs.
However naive usage of LLMs to predict performance (surrogate modeling) \cite{agenthpo_liu2025largelanguagemodelagent} in the multi-dimensional LLM hyperparameter space, is prone to hallucination. We propose to solve this by grounding the LLM surrogate with function estimation techniques.

Through extensive empirical evaluation across diverse domains, including mathematical and general reasoning, medical question answering, legal analysis, programming, agent tool calling, e-commerce, and real world cloud incident diagnosis, we demonstrate that \algname{} consistently outperforms existing AutoML baselines, delivering significant accuracy gains with minimal additional computational overhead (w.r.t. default choices). By \emph{automating the end-to-end LLM adaptation process}, \algname{} makes robust, domain-specific LLMs accessible to a broader range of users and organizations, \emph{bridging the gap between cutting-edge research and practical real-world applications}.
In summary:
\begin{itemize}[leftmargin=9pt, nosep, noitemsep]
    \item We propose a novel method for end-to-end LLM domain adaptation guided by a structured adaptation workflow ensuring a high success rate in executing the pipeline. To the best of our knowledge, we are the first to propose an automated domain adaptation framework for LLMs.
    \item We curate a knowledge base for LLM domain adaptation from open source implementations, datasets, models, and best practices to ground the recommendations.
    \item We address the exponential search space in LLM adaptation, by designing a novel multi-agent framework that traverses the adaptation graph incorporating the user constraints, data signals, and best practices, to recommend adaptation pipelines and hyperparameter ranges.
    \item We formulate a first-of-its-kind probabilistic function estimation based technique in conjunction with LLMs to efficiently perform refinement of the parameter choices and ranges for the selected pipeline.
    \item We conduct extensive experiments on datasets spanning various benchmark and real world datasets across multiple domains to demonstrate the efficacy of our method in adapting LLMs to the specialized tasks. 
\end{itemize}

\section{Related Works}

\begin{table*}[t]
\centering
\vspace*{-0.25cm}
\caption{Comparison between \algname{} and existing LLM-based frameworks.} \label{table:related_work_comparison}
\resizebox{\textwidth}{!}{%
\begin{tabular}{@{}lcccccc@{}}
\toprule
\multirow{2}{*}{\textbf{Framework}} & \multicolumn{6}{c}{\textbf{Key Functionality}}                                                                           \\ \cmidrule(l){2-7} 
                           & \textbf{Planning} & \textbf{Verification} & \textbf{Efficient Optimization} & \textbf{Effective LLM Adaptation} & \textbf{Training-Free Search} & \textbf{With Retrieval} \\ \midrule \midrule
AutoML-GPT\,\citep{autoGPT_zhang2023automlgptautomaticmachinelearning} & \textcolor{red}{$\times$}        & \textcolor{red}{$\times$}            & \textcolor{red}{$\times$}             & \textcolor{red}{$\times$}                          & \textcolor{teal}{$\checkmark$}                    & \textcolor{red}{$\times$}              \\
Prompt2Model\,\citep{Prompt2Model_2023} & \textcolor{red}{$\times$}        & \textcolor{red}{$\times$}            & \textcolor{red}{$\times$}             & \textcolor{red}{$\times$}                         & \textcolor{red}{$\times$}                    & \textcolor{teal}{$\checkmark$}              \\
HuggingGPT\,\citep{hugginggpt_shen2023hugginggptsolvingaitasks} & \textcolor{teal}{$\checkmark$}        & \textcolor{red}{$\times$}            & \textcolor{red}{$\times$}             & \textcolor{red}{$\times$}                          & \textcolor{teal}{$\checkmark$}                    & \textcolor{teal}{$\checkmark$}              \\
CAAFE\,\citep{CAAFE_2023} & \textcolor{red}{$\times$}        & \textcolor{teal}{$\checkmark$}            & \textcolor{red}{$\times$}             & \textcolor{red}{$\times$}                         & \textcolor{red}{$\times$}                    & \textcolor{red}{$\times$}              \\
MLCopilot\,\citep{mlcopilot_zhang2024mlcopilotunleashingpowerlarge} & \textcolor{red}{$\times$}        & \textcolor{red}{$\times$}            & \textcolor{red}{$\times$}             & \textcolor{red}{$\times$}                          & \textcolor{teal}{$\checkmark$}                    & \textcolor{red}{$\times$}              \\
AgentHPO\,\citep{agenthpo_liu2025largelanguagemodelagent} & \textcolor{teal}{$\checkmark$}        & \textcolor{teal}{$\checkmark$}            & \textcolor{red}{$\times$}             & \textcolor{red}{$\times$}                          & \textcolor{red}{$\times$}                    & \textcolor{red}{$\times$}              \\
Data Interpreter\,\citep{datainterpretor_hong2024datainterpreterllmagent} & \textcolor{teal}{$\checkmark$}        & \textcolor{teal}{$\checkmark$}            & \textcolor{red}{$\times$}             & \textcolor{red}{$\times$}                          & \textcolor{red}{$\times$}                    & \textcolor{red}{$\times$}              \\
DS-Agent\,\citep{dsagent_guo2024dsagentautomateddatascience} & \textcolor{teal}{$\checkmark$}        & \textcolor{teal}{$\checkmark$}            & \textcolor{red}{$\times$}             & \textcolor{red}{$\times$}                          & \textcolor{red}{$\times$}                    & \textcolor{teal}{$\checkmark$}              \\
SELA\,\citep{SELA_2024} & \textcolor{teal}{$\checkmark$} & \textcolor{teal}{$\checkmark$} & \textcolor{red}{$\times$} & \textcolor{red}{$\times$}  & \textcolor{red}{$\times$} & \textcolor{red}{$\times$} \\
Agent\,K\,\citep{AgentK_2024} & \textcolor{teal}{$\checkmark$} & \textcolor{teal}{$\checkmark$} & \textcolor{red}{$\times$} & \textcolor{red}{$\times$}  & \textcolor{red}{$\times$} & \textcolor{teal}{$\checkmark$} \\
AutoMMLab\,\citep{AutoMMLab_AAAI} & \textcolor{red}{$\times$}        & \textcolor{teal}{$\checkmark$}            & \textcolor{red}{$\times$}              & \textcolor{red}{$\times$}                         & \textcolor{red}{$\times$}                    & \textcolor{red}{$\times$} \\
AutoMLAgent \citep{automlagent_trirat2025automlagent}                & \textcolor{teal}{$\checkmark$}        & \textcolor{teal}{$\checkmark$}            & \textcolor{red}{$\times$}              & \textcolor{red}{$\times$}                          & \textcolor{teal}{$\checkmark$}                    & \textcolor{teal}{$\checkmark$}              \\  \midrule
\textbf{\algname{}} (Ours)                 & \textcolor{teal}{$\checkmark$}        & \textcolor{teal}{$\checkmark$}            & \textcolor{teal}{$\checkmark$}             & \textcolor{teal}{$\checkmark$}                         & \textcolor{teal}{$\checkmark$}                    & \textcolor{teal}{$\checkmark$}              \\ \bottomrule
\end{tabular}%
}
\vspace*{-0.2cm}
\end{table*}

There exists a plethora of works in literature on automated machine learning. Here, we focus on agentic solutions to automate this process. In agentic literature, prior works such as \cite{autoGPT_zhang2023automlgptautomaticmachinelearning, hugginggpt_shen2023hugginggptsolvingaitasks, datainterpretor_hong2024datainterpreterllmagent, agenthpo_liu2025largelanguagemodelagent} focus on specific aspects of the ML training pipeline such as data preparation, analysis or HPO etc. \textbf{MLCopilot} \cite{mlcopilot_zhang2024mlcopilotunleashingpowerlarge} was one of the first agentic systems tackling general ML pipelines, using an offline stage to build knowledge from experience and an online retrieval stage. Similarly \textbf{DS-Agent} \cite{dsagent_guo2024dsagentautomateddatascience} proposes to use LLM agents along with case based reasoning (CBR) to come up with generic data science pipelines for a task. 
\textbf{AutoML-Agent} \cite{automlagent_trirat2025automlagent} employs a multi-agent LLM framework using retrieved knowledge bases to automate generic ML training pipelines.
These methods are not tailored to LLMs and fall short in reliably planning and executing complex, multi-stage LLM adaptation workflows. 
Unlike prior methods that typically support only single-stage training pipelines, \algname{} provides a comprehensive framework for full domain adaptation pipelines, including multi-stage processes such as in-context prompting (e.g. RAG), supervised fine-tuning (SFT), alignment (DPO) and so on.
\algname{} aims to \emph{reliably} and \emph{efficiently} generate user-aligned executable \emph{LLM adaptation} plans, overcoming the limitations of existing AutoML systems.
Please refer the Appendix \ref{apndx_related_works} for further details.

\section{Problem Formulation}

We formulate domain adaptation of large language models (LLMs) as a 
multi–stage optimization problem. Let 
$D = \{(x_i, y_i)\}_{i=1}^N$
denote the user–provided dataset for a downstream task, and let 
\(c\) represent the user context (e.g., hardware constraints, latency 
requirements, or user preferences such as PEFT vs.\ full fine–tuning).
An adapted model is constructed through a sequence of 
\(T\) transformation stages (e.g., continual pretraining, supervised 
fine–tuning, preference optimization, inference–time augmentation).  
Each stage \(t\) is parameterized by a configuration 
\(h_t \in \mathcal{H}_t\).  
The final model produced by the pipeline is denoted as $f_\theta(h_1,\ldots,h_T).$
\algname{} \emph{differs from existing AutoML techniques in this aspect of planning each component of the pipeline sequentially following the hierarchy of the adaptation graph.}
Our goal is to find a feasible configuration sequence  
\((h_1,\ldots,h_T)\) that minimizes expected task loss:
\[
\min_{h_1,\ldots,h_T \in \mathcal{H}(c)}
\;\; 
\mathbb{E}_{(x,y)\sim D}
\left[
    L\big(f_\theta(h_1,\ldots,h_T)(x), y\big)
\right],
\]
where the feasible set \(\mathcal{H}(c)\) encodes context‑dependent 
constraints such as GPU memory, training budget, or user‑specified 
preferences.

\textbf{Prior‑Guided Sequential Configuration Space: }
Each stage \(t\) has a set of hyperparameters, $h_t = \{h_{t,1}, \ldots, h_{t,n_t}\},$
and rather than selecting all hyperparameters jointly, we construct them 
sequentially.  
Let \(h_{<t} = (h_1,\ldots,h_{t-1})\) denote configurations from previous 
stages, and \(h_{t,<i} = (h_{t,1},\ldots,h_{t,i-1})\) denote earlier 
choices within stage \(t\).
Define a prior‑guided feasible region for stage \(t\) as
\[
\mathcal{P}_t(c, h_{<t})
=
\left\{
    h_t : 
    h_{t,i} \in \mathcal{P}_{t,i}(c, h_{<t}, h_{t,<i}),
    \; i = 1,\ldots,n_t
\right\},
\]
where each sub‑policy \(\mathcal{P}_{t,i}\) incorporates 
domain knowledge (e.g., best‑practice hyperparameter ranges), 
conditional dependencies (e.g., selecting LoRA activates 
\(\text{rank}, \alpha, \text{dropout}\)), and context constraints.
The multi–stage adaptation problem thus becomes:
\begin{align*}
\min_{h_1,\ldots,h_T}
\;\;
& \mathbb{E}_{(x,y)\sim D}
\left[
    L\big(f_\theta(h_1,\ldots,h_T)(x), y\big)
\right]
\quad \\
& \text{s.t.}
\quad
h_t \in \mathcal{P}_t(c, h_{<t}),\;
\forall t
\end{align*}

\begin{proposition}
Let \(\mathcal{H}(c)\) be the feasible configuration space and 
\(\mathcal{P}(c) = \mathcal{P}_1 \times \cdots \times \mathcal{P}_T\) the 
prior‑guided subspace.  
Denote the expected task loss as 
\(L_{\mathbb{E}}(h)\),  
the optimal loss over the full space is $L_{\mathcal{H}}^{*} = \min_{h\in \mathcal{H}(c)} L_{\mathbb{E}}(h)$ and the optimal loss over the prior‑guided space is $L_{\mathcal{P}}^{*} = \min_{h\in \mathcal{P}(c)} L_{\mathbb{E}}(h)$. Considering \(L_{\mathbb{E}}(h)\) is Lipschitz continuous in hyperparameters, the prior‑guided subspace satisfies the generalized 
\(\varepsilon\)-covering bound:
\[
L_{\mathcal{P}}^{*}
\;\le\;
L_{\mathcal{H}}^{*}
\;+\;
L \varepsilon,
\]
where \(L\) is the Lipschitz constant and \(\varepsilon\) is the distance between the optimal parameters in $\mathcal{H}(c)$ and $\mathcal{P}(c)$. 
\end{proposition} 
The above result is straightforward to verify after continuous relaxation \cite{boyd2004convex} of the discrete space. It implies that restricting search to the prior‑guided subspace incurs at most a controlled additive suboptimality while yielding dramatically improved tractability.
For the optimization over the constrained search space we propose to use LLMs together with function estimation (using Gaussian process) techniques. More details on Gaussian processes can be found in Appendix \ref{apndx_prelims} and \citealt{beckers2021introductiongaussianprocessmodels}. 

\noindent Note, the above formulation is intended solely for clarity of exposition and the proposed method does not depend on the stated assumptions.

\section{Method}

\begin{figure*}[ht!]
    \centering
    \includegraphics[width=0.95\linewidth]{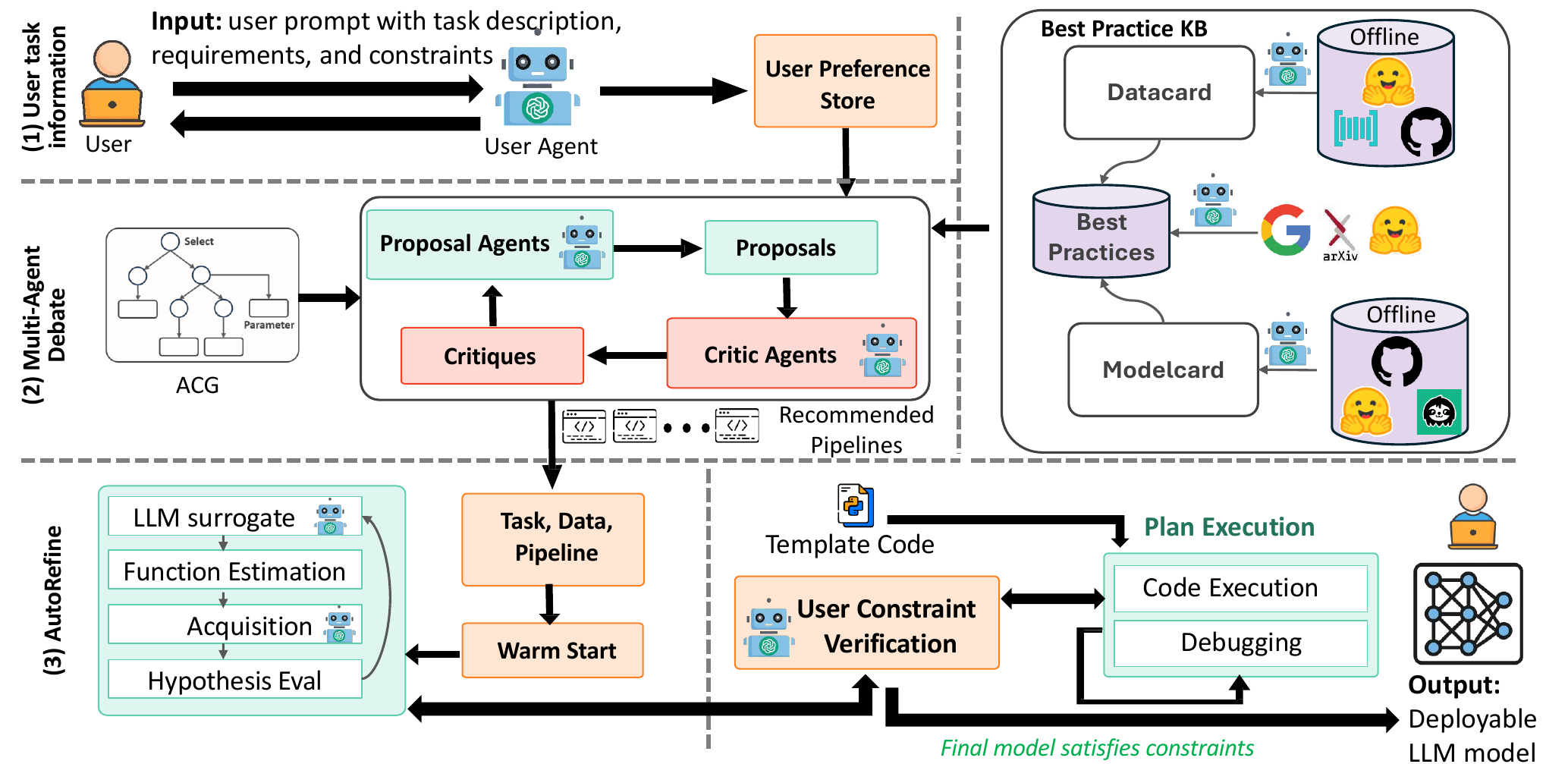}
    \caption{Framework Overview: \algname{} processes user data and task definitions, integrates best practices, generates an executable training pipeline, and refines an intermediate model to produce a deliverable model for users. }
    \label{fig:autoadapt_overview}
\end{figure*}

This section details the components of our method. Figure \ref{fig:autoadapt_overview} gives an overview of \algname{}. 
The user provides the task description, data and constraints. \algname{} then runs to find the optimal adaptation pipeline through the following workflow: 
(i) The \textbf{Adaptation Configuration Graph (ACG)} is executed to decompose the decision making in a structured workflow. (ii) A \textbf{multi‑agent debate pipeline} is run at every node of the ACG to propose an executable adaptation plan with narrowed hyperparameter ranges. The multi-agent system is grounded in \textbf{best practices} from literature (processed in a KB or queried real time from web search, huggingface, github etc.) (iii) On obtaining the adaptation pipeline from the multi-agent system, \textbf{AutoRefine} is run to obtain the exact parameter set. AutoRefine is a feedback‑driven optimization loop that combines LLM surrogate reasoning with probabilistic function estimation to efficiently and effectively select hyperparameters. 
(iv) Once the optimization is complete, the pipeline is executed to produce the adapted LLM for the user task.

\subsection{Best Practices Knowledge Base}
The multi-agentic system needs to be grounded in the real world LLM usage ecosystem. For this we curate the best-pratices corpus from public LLM ecosystems capturing model choices, training pipelines, and hyperparameters at scale. The sources include Hugging Face model/data cards and tags (e.g., sft, dpo, domain tags like medical), PwC leaderboards and notebooks (github, unsloth) and practitioner guides on web.
This corpus helps in the hierarchical decision of strategy selection (e.g., \texttt{RAG} vs \texttt{SFT}, \texttt{LoRA} vs. \texttt{Prefix Tuning}), model choice (e.g. \texttt{Llama-3.2-3B-Instruct} vs \texttt{Qwen2.5-3B}) and hyperparameter ranges (e.g., \texttt{learning\_rate}, \texttt{batch\_size} etc.).

We extract knowledge from $\sim 100$ K model and data cards available from the sources. 
We prompt an LLM to extract features such as model family, size, training pipeline and parameters if available. Similarly for the datasets we extract the dataset size, domain, best performing models among other details.
These structured extractions (training technique, dataset size, hyperparameters, evaluation metrics, source links) form the schema of the KB. 

\textbf{Extraction pipeline} The system applies prompt‑driven parsers over model/data cards and notebooks to obtain details like training technique, model/dataset size, hyperparameters, evaluation, model source etc. The entries are indexed in a vector store by task description, domain and dataset descriptors. The KB is then queried given an incoming task (e.g. medical QA), returning top‑k matched datasets/models and their reported training configurations to initialize
the planning. 
Please refer $\S$ \ref{apndx_bp} for details.

\subsection{Adaptation Configuration Graph (ACG)}

\begin{figure}
    \centering
    \includegraphics[width=0.9\linewidth]{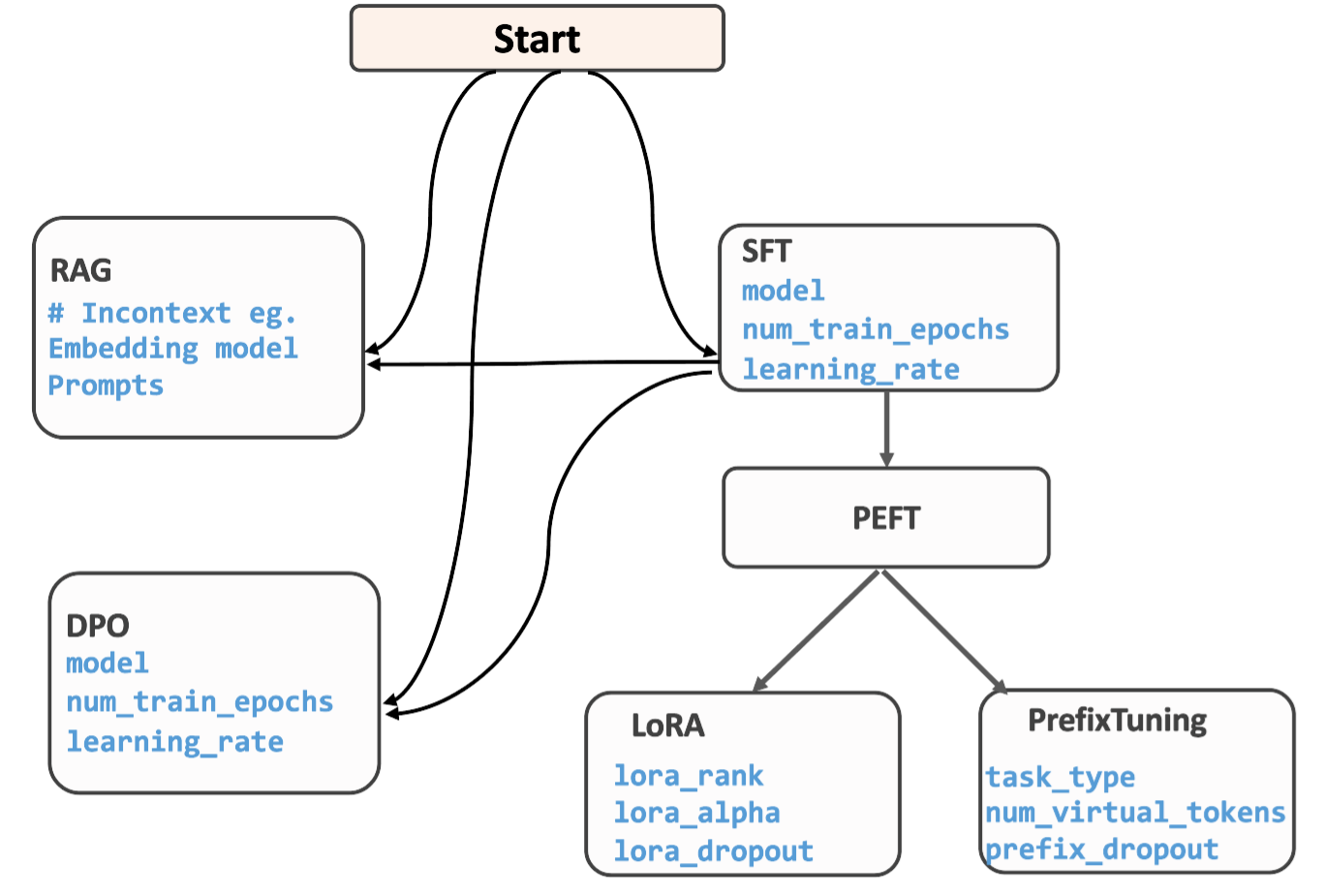}
    \caption{Hierarchical Configuration Determination Using a Domain Adaptation Configuration Graph (ACG) }
    \label{fig:acg}
    \vspace{-0.5cm}
\end{figure}

The ACG formalizes adaptation planning as a directed acyclic graph.
The nodes are of two types: selection nodes for discrete strategy/technique (e.g. \texttt{RAG} vs \texttt{SFT}) choices and parameter nodes for  hyperparameters within the selected technique (e.g. \texttt{learning\_rate} in \texttt{SFT}). 

\textbf{Construction:} We construct the ACG manually based on best practices. The start selection node selects between \texttt{RAG}, \texttt{SFT} and \texttt{Alignment} . Along \texttt{SFT}, the decision path includes \texttt{PEFT} and within \texttt{PEFT} the choices include \verb|{LoRA, DoRA, Prefix Tuning}|. The selection nodes for \texttt{Alignment} can be invoked post \texttt{SFT} or without \texttt{SFT} (using a pretrained model). The choices inside \texttt{Alignment} include \verb|{DPO, RLHF}|. The Parameter nodes, present at the leaf of a traversed path, represent the decision for the continuous/categorical parameter ranges for e.g. parameters for \texttt{SFT} include \texttt{epochs, learning\_rate, LoRA rank/alpha/dropout, batch\_size, gradient\_accumulation, quantization} etc. Edges encode dependencies e.g. choosing \texttt{LoRA} in \texttt{PEFT} selection node activates the dependent Parameter node to decide the search space for \verb|{lora_r, lora_alpha, lora_dropout}|.

\textbf{Hierarchical Decision over the ACG:} Figure \ref{fig:acg} shows a sample ACG workflow. The planner (multi-agent debate system outlined next) starts at high‑level decision nodes (e.g., \texttt{RAG}, \texttt{SFT} and \texttt{Alignment}) and traverses until leaf nodes (complete executable plan). At each node, the planner retrieves information from the best‑practice KB
and aggregates them with user constraints (compute/time/budget) to decide the pipeline.
This structured approach constrains search making the adaptation efficient and ensures reliable execution. The traversed path need not adhere to a linear chain graph but could also involve multiple paths e.g. \texttt{SFT} followed by \texttt{Alignment} or \texttt{RAG} as decided by the system.

\subsection{Multi-Agent Debate Planner}

In order to make the decision at the nodes of the ACG, we design a multi-agent pipeline.
The selection nodes of the ACG could also be treated as hyperparameters and evaluated. However this increases the search space which becomes infeasible for the case of LLMs. 
Thus we need an efficient planning system akin to expert humans selecting techniques from knowledge or literature.

The multi‑agent system consists of proposal and critic agents that debate in turns to come up with the final decision on the technique, model and parameter ranges. 
Below we briefly describe the agents (with details in Appendix \ref{apndx_multiagent_system})

\noindent \textbf{Proposal Agents: }
Proposal agents are responsible for constructing and iteratively refining candidate domain adaptation pipelines. 
These agents are (i) \emph{best practice agent}: consolidates empirical knowledge from prior training pipelines, model cards, and published benchmarks, and (ii) \emph{knowledge retrieval agent}: dynamically retrieves evidence from recent literature, technical reports, and online resources. (iii) \emph{research agent}: utilizes the research information retrieved from arxiv, KBs etc. to propose techniques. 
These agents initialize the search process in regions of the configuration space that historically yield strong performance while remaining responsive to emerging techniques.

\noindent \textbf{Critic Agents: }
Critic agents verify the proposed configurations against user, data and resource constraints. 
(i) The \emph{user agent} enforces user‑specified requirements, such as compute budgets, training preferences, methodological restrictions, and domain‑specific constraints.
In parallel, (ii) the \emph{data agent} performs data‑centric critique by analyzing statistics of the target dataset, including scale, token distributions, domain characteristics etc.,
and mapping these signals to concrete constraints on adaptation choices.

After each round of debate, an \emph{aggregator agent} aggregates the proposals and critiques.
The proposal agents further revise the proposals based on the critiques and unsatisfied constraints.
In this manner, each ACG node undergoes  \emph{proposal $\xrightarrow{}$ critique $\xrightarrow{}$ revision } iterations. 
The loop terminates when all active constraints are satisfied and the plan is deemed executable (with runnable code). The AutoRefine module is then invoked to decide the exact parameter values.

\subsection{AutoRefine}

This module decides the optimal hyperparameter set using an LLM surrogate method.
Prior works \cite{transformersfailforecasttimezhou2025} have shown that transformers (and thus LLMs) have limitations in time series forecasting. This issue is bound to be amplified in the multivariate forecasting by works using LLMs as surrogates \cite{llambo_liu2024large} to predict results. Thus we propose a principled approach grounded in function estimation theory to guide the surrogate modeling from prior knowledge of historic runs.
We use Gaussian Process (GP) model for function estimation.

Algorithm $\ref{alg:pseudo_code}$ provides a pseudo-code for AutoRefine.
We start by sampling a coreset \cite{jubran2019introductioncoresetsaccuratecoresets} of the data for efficient execution. We then perform few ($\sim 3$) execution runs on the selected pipeline (e.g. \texttt{SFT}) by invoking an LLM to sample initial parameters from the narrowed search space (lines 2-5). 
Inspired by \cite{llmprocess_NEURIPS2024_c5ec2271}, we prompt the LLM to predict the results along each parameter dimension from the evaluated points (lines 9-12). We reason that predicting the trend along each parameter axis will help reduce hallucinations compared to directly estimating the results on the higher dimensional space. These historic and predicted results are used to train a GP model to obtain a function estimation of the remaining points in the search space (line 13). Note the GP model requires hundreds of data points to learn the function and it is not possible to finetune/evaluate LLMs for many iterations. Hence, we obtain the parameter-axis wise trends of the results to augment data. Once the model of the results over the search space is obtained we need an acquisition function that would maximize the result of the acquired hyperparameter set. For acquisition, we find the optimal point of function using LBFGS \cite{Liu1989OnTL} (after continuous relaxation of the categorical variables), starting from each of the historic evaluated samples (lines 14-18). Post discretization, the LLM is prompted in-context with these optimal points to choose the next parameter set to evaluate on (lines 19-21), balancing exploration and exploitation. Thus the LLM surrogate is guided by the optimal hyperparameter as per the estimated function that is grounded in results from prior runs, reducing hallucinations. We allow a maximum of $\mathcal{B}$ (loops starting line 2,7) function evaluations (e.g. finetuning followed by testing on validation set) on the selected parameters.
Following theoretical result analyzes the regret (difference of result on selected parameter from true optimal) bounds of AutoRefine
(cf. Appendix \ref{apndx_regret_bound} for details).
The regret grows sub-linearly with the iterations and the constants are improved by the LLM augmentation factor.
\begin{lemma}
    Pick $\delta \in (0,1)$. Then with probability at least $1-\delta$, the cumulative regret of the algorithm in \ref{alg:pseudo_code} is bounded as $R_{T_a} = \mathcal{O}\left(\sqrt{\frac{T}{F_a}log^2\left(\frac{T}{F_a}\right)}\right)$. Further, to guarantee the average regret $r_T \le \varepsilon$ for some $\varepsilon > 0$, we have $T_a = \Omega\!\Bigg(\frac{\ln^2(1/\varepsilon)}{F_a\varepsilon^2}\Bigg)$.
    Here $T_a, T$ is the number of actual function evaluations and total iterations (actual evals + LLM predictions) of the algorithm respectively, $F_a$ is the LLM surrogate augmentation factor
\end{lemma}

\renewcommand{\algorithmiccomment}[1]{\textcolor{RoyalBlue}{$\triangleright$ #1}}

\renewcommand{\algorithmicrequire}{\textbf{Input:}}
\renewcommand{\algorithmicensure}{\textbf{Output:}}

\setlength{\textfloatsep}{0.3cm}
\begin{algorithm}[t!]
\caption{Overall Procedure of AutoRefine} \label{alg:pseudo_code}
\begin{algorithmic}[1]
\REQUIRE Pipeline $\mathcal{P}$, Parameter space $\Theta$, Coreset $C \subseteq \texttt{Data}$, LLM $\mathcal{L}$, GP kernel $\kappa$ (RBF), Budget $\mathcal{B}$, Augmentation Factor $F_a$
\ENSURE Best configuration $\theta^\star \in \Theta$
\STATE $\mathcal{M} = \varnothing$
\FOR[Warm start with $\mathcal{K}<\mathcal{B}$ runs]{$k$ \textbf{in} $\mathcal{K}$}
\STATE $\theta^{(k)} \gets \mathcal{L}(\mathcal{M}, \mathcal{P})$ \COMMENT{Sample from parameter space}
\STATE $\hat{y}^k \gets \mathrm{Eval}(\mathcal{P}, C, \theta^{(k)})$ \COMMENT{Evaluate pipeline and observe outcomes}
\STATE $\mathcal{M} \gets \mathcal{M} \cup (\theta^k, \hat{y}^k)$
\ENDFOR

\WHILE{$!\texttt{Converged}$ \textbf{or} $\texttt{niters} = \mathcal{B-K}$}
\STATE $\mathcal{M_{L}} \gets M$
\FOR[Univariate LLM forecasting along each parameter, keeping others fixed]{$j$ \textbf{in} $\texttt{dim}(\theta)$}
  \STATE $\{(\theta_j^{(s)}, \tilde{y}_j^{(s)})\} \gets \mathcal{L}(\mathcal{M}, \mathcal{P}, \texttt{parameter=j}, \texttt{s=}\lceil\frac{F_a}{\texttt{dim}(\theta)}\rceil)$
  \STATE $\mathcal{M_{L}} \gets \mathcal{M_{L}} \cup \{(\theta_j^{(s)}, \tilde{y}_j^{(s)})\}$
\ENDFOR

\STATE $\mathcal{S}(\theta), \mu(\theta), \sigma(\theta) \gets \mathrm{GP}(\mathcal{M_L}, \kappa)$ \COMMENT{Fit GP surrogate and obtain UCB score, mean, uncertainty over $\Theta$ }

\STATE $M_{\texttt{opt}} \gets \varnothing$
\FOR[Optimize the estimated density from each evaluated point]{$t \in \vert \mathcal{M} \vert$}
  \STATE $(\theta^t_{\texttt{opt}}, y^t) \gets \texttt{LBFGS}(S(\theta), \theta^{t})$
  \STATE $M_{\texttt{opt}} \gets M_{\texttt{opt}} \cup (\theta^t_{\texttt{opt}}, y^t)$
\ENDFOR

\STATE $\theta_{s} \gets \mathcal{L}(\mathcal{M, P, M_{\texttt{opt}}}, \theta)$ \COMMENT{Acquisition function to sample new set of parameters for evaluation}
\STATE Evaluate $\hat{y}_s \gets \mathrm{Eval}(\mathcal{P}, C, \theta_s)$
\STATE $\mathcal{M} \gets \mathcal{M} \cup \{(\theta_s, \hat{y}_s)\}$ 
\ENDWHILE
\STATE \textbf{return} $\theta^\star \gets \arg\max_{(\theta,y)\in \mathcal{M}} y$ 
\end{algorithmic}
\end{algorithm}

\section{Experiments}

\begin{figure*}[ht!]
    \centering
    \includegraphics[width=0.8\linewidth]{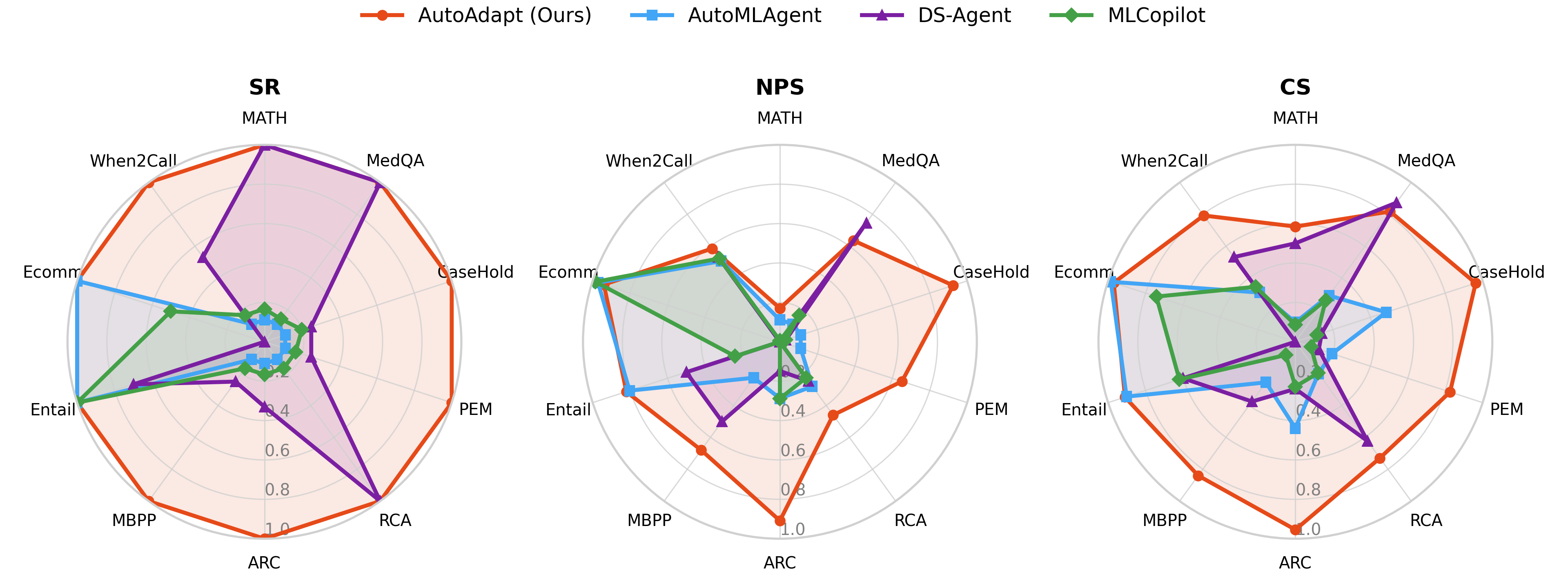}
    \caption{Success Rate (SR), Normalized Performance Score (NPS), Cumulative Score (CS) comparing \algname{} with baseline methods across datasets. Higher score indicates better results. \algname{} outperforms SoTA baselines across datasets. Detailed results are in $\S$ \ref{apndx_res}.}
    \label{fig:main_results} 
\end{figure*}

We design our experiments to evaluate the effectiveness of \algname{} and aim to answer the following research questions (RQ):
\textbf{RQ1:} Can \algname{} effectively recommend adaptation pipelines across real world and benchmark datasets?
\textbf{RQ2:} What is the importance of each component of \algname{}?
\textbf{RQ3:} How do the results converge on scaling the iterations?
\textbf{RQ4:} What is the effectiveness of AutoRefine module compared to Bayesian Optimization?
\textbf{RQ5:} How efficient is \algname{}?

\noindent \textbf{Datasets}
We conduct experiments on the following benchmark datasets (details in $\S$ \ref{apndx_datasets}):
\textbf{MATH} \cite{hendrycksmath2021}, \textbf{MedQA} \cite{medqa_jin2020disease}, \textbf{CaseHold} \cite{casehold_zheng2021doespretraininghelpassessing}, \textbf{ARC} \cite{allenai_arc}, \textbf{MBPP+} \cite{mbppplus}, \textbf{Ecomm} \cite{automlagent_trirat2025automlagent}, \textbf{Entail} \cite{automlagent_trirat2025automlagent}, \textbf{When2Call} \cite{when2call-ross-etal-2025}. In order to test the adaptation to real world setting, 
we also evaluate on proprietary datasets containing pre-engineering level math questions (\textbf{PEM}) and root cause analysis of cloud incidents (\textbf{RCA}).

\noindent \textbf{Metrics:}
We report the accuracy for task-specific benchmarks and compute overhead in GPU hours and estimated cost (\$ value). 
Additionally, we study convergence stability and the number of iterations to reach optimal configurations.

\noindent \textbf{Evaluation Protocol:}
\algname{} was run with its full pipeline: ACG based hierarchical decision making, Best practice KB, multi-agent debate planner, and AutoRefine. Baselines were executed using their recommended workflows and default hyperparameter tuning strategies. 
Where baselines fail to directly produce executable pipelines, we perform minimal corrective edits to ensure executability for each error, run and repeat till success. We report success rate (SR) as the reciprocal of the count of these iterations.
All experiments were repeated three times to ensure consistency and the average is reported.
\noindent \textbf{Baselines:}
We compare \algname{} against the following SoTA Agentic approaches:
1) \textbf{MLCopilot} \cite{mlcopilot_zhang2024mlcopilotunleashingpowerlarge}:
2) \textbf{DS-Agent} (\citet{dsagent_guo2024dsagentautomateddatascience}, ICML'24)
3) \textbf{AutoMLAgent} (\citet{automlagent_trirat2025automlagent}, ICML'25)

\textbf{Implementation Details}
We implement \algname{} as an end-to-end pipeline integrated with the Hugging Face model hub. 
Due to resource constraints, we perform adaptation on models less than 15B parameters from the following family: Qwen, Llama, Phi, Mistral.
All runs were executed on a single 4x A100 GPU with mixed precision enabled (bf16) to ensure reproducibility and resource efficiency.
In AutoRefine, the Gaussian process (GP) model for function estimation was trained using the GPyTorch framework \cite{gpytorchblackboxmatrixmatrixgaussian_gardner2021} for maximum of 50 iterations to avoid overfitting and to enhance efficiency.
Inspired by \cite{automlagent_trirat2025automlagent}, we report success rate (SR), Normalized Performance score (NPS) and a Comprehensive score (CS) averaging SR and NPS, for methods over k(=3) runs. NPS is the task-specific accuracy for our datasets. To select a coreset (by default 10\%) of the data for finetuning/evaluation in AutoRefine we make use of \cite{deepcore_guo2022deepcorecomprehensivelibrarycoreset}. Note finetuning/evaluation using the final pipeline is done on the entire train/test set.
All baselines and \algname{} were constrained by identical underlying LLM (GPT-4.1-mini), resource budgets and evaluated under the same conditions to ensure fairness.

\section{Results}

\textbf{Main Results (RQ1)}
Figure \ref{fig:main_results} shows the SR, NPS and CPS metrics in the \emph{template-free} (TF) setting where methods are allowed to choose any technique, model etc. for solving the task.
We report the average scores from 3 independent runs for
all evaluation metrics.
\algname{} consistently achieves the highest performance across nearly all datasets and metrics.
Table \ref{tab:template_aware_results} shows the \emph{template-aware} (TA) results with the model, technique (SFT) and epochs fixed to ensure comparison across methods for hyperparameter selection. We observe that \algname{} outperforms the generic AutoML methods by a large margin.
These results demonstrate that \algname{}’s structured planning, prior‑guided configuration, and iterative refinement enable more reliable and consistently superior domain adaptation compared to existing automated baselines, making it effective for LLM adaptation to downstream tasks.
Further we show the results of \algname{} with the constraints of choosing RAG and alignment (DPO) pipelines on the PEM dataset. We see from figure \ref{fig:rag_dpo} that compared to the baseline of human tuned or HF defaults our method performs significantly better,
indicating \algname{}'s effectiveness in optimizing the LLM adaptation pipeline irrespective of underlying technique.

\begin{table}[ht]
\centering
\caption{Performance comparison across datasets and methods in the template aware (TA) setting. \algname{} is better than the baselines even when model and technique (SFT) is fixed.}
\label{tab:template_aware_results}
\resizebox{\linewidth}{!}{%
\begin{tabular}{l l c c c c c}
\toprule
\textbf{Dataset} & \textbf{AutoAdapt} & \textbf{AutoMLAgent} & \textbf{MLCopilot} & \textbf{HF Defaults} & \textbf{DS-Agent} \\
\midrule
UC Berkeley MATH & $27.8 \pm 3.5$ & $26.8 \pm 2.4$ & $27.2 \pm 2.3$ & $26.2 \pm 0.06$ & $\mathbf{30.0 \pm 3.7}$ \\
MedQA            & $\mathbf{61.27 \pm 1.1}$ & $58.84 \pm 1.3$ & $57.50 \pm 1.2$ & $58.13 \pm 0.04$ & $59.86 \pm 1.0$ \\
CaseHold         & $\mathbf{89.22 \pm 0.9}$ & $84.72 \pm 1.1$ & $82.63 \pm 1.0$ & $86.21 \pm 0.08$ & $87.28 \pm 0.9$ \\
PW               & $\mathbf{67.69 \pm 1.3}$ & $64.53 \pm 1.4$ & $63.71 \pm 1.5$ & $65.78 \pm 0.02$ & $63.40 \pm 1.6$ \\
RCA              & $\mathbf{2.30 \pm 0.76}$ & $2.23 \pm 0.81$ & $2.18 \pm 0.77$ & $2.23 \pm 0.81$ & $2.25 \pm 0.66$ \\
ARC              & $\mathbf{83.79 \pm 1.2}$ & $77.39 \pm 1.1$ & $77.56 \pm 1.0$ & $78.24 \pm 0.03$ & $77.47 \pm 1.4$ \\
MBPP+            & $\mathbf{68.3 \pm 0.8}$ & $\mathbf{68.2 \pm 0.7}$ & $67.5 \pm 0.9$ & $\mathbf{68.0 \pm 0.07}$ & $66.9 \pm 0.8$ \\
When2Call & $\mathbf{54.5 \pm 0.9}$ & $\mathbf{45.4 \pm 1.2}$ & $48.1 \pm 1.1$ & $\mathbf{46.9 \pm 0.05}$ & $50.0 \pm 0.7$ \\
\bottomrule
\end{tabular}%
}
\end{table}

\begin{figure}[!t]
    \centering
    \subfloat[Variations of \algname{}.]{
        \includegraphics[width=0.5\linewidth]{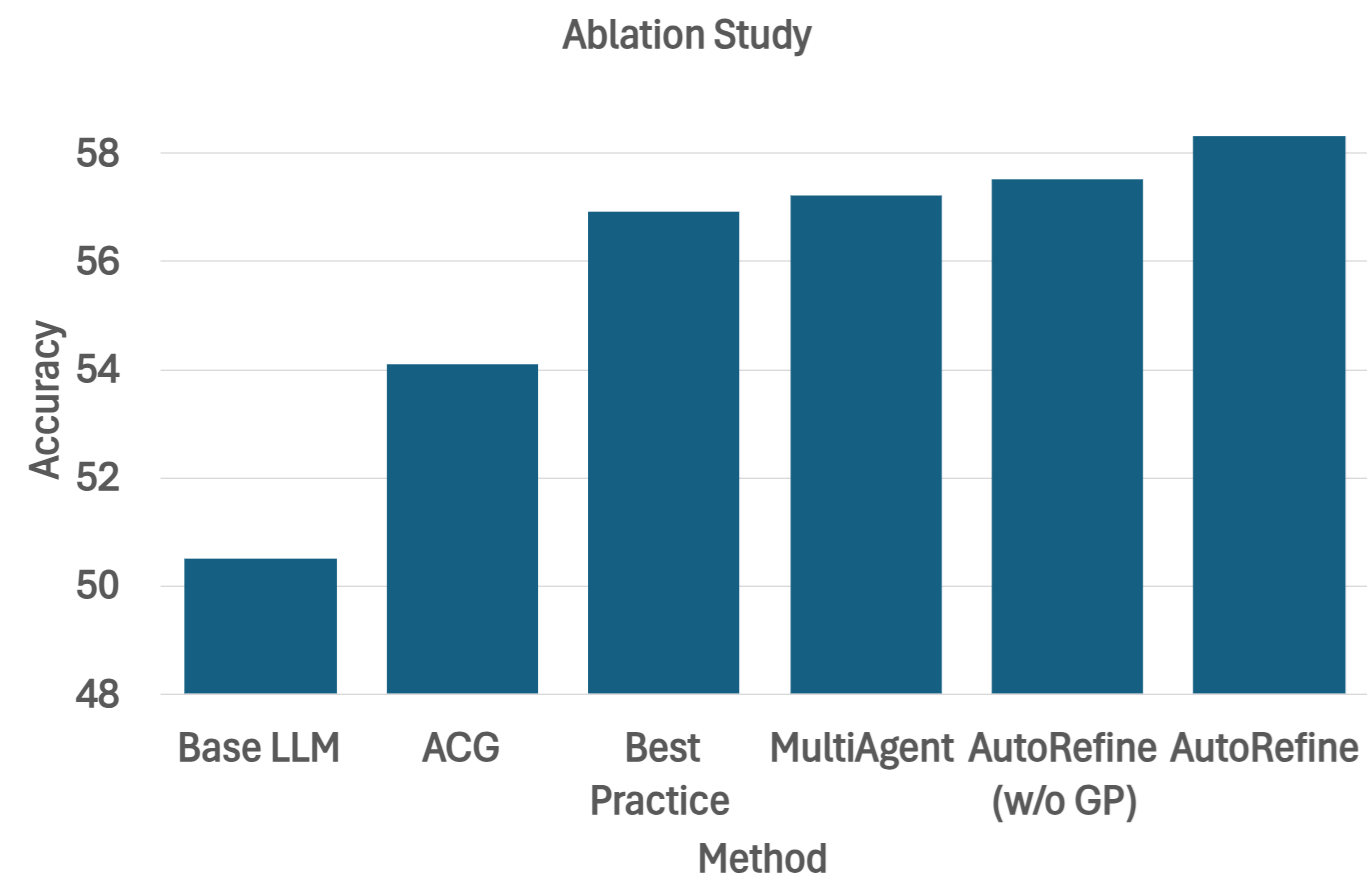}%
        \label{figure:ablation_figure}
    }
    \subfloat[Varying \# rounds/trials.]{
        \includegraphics[width=0.47\linewidth]{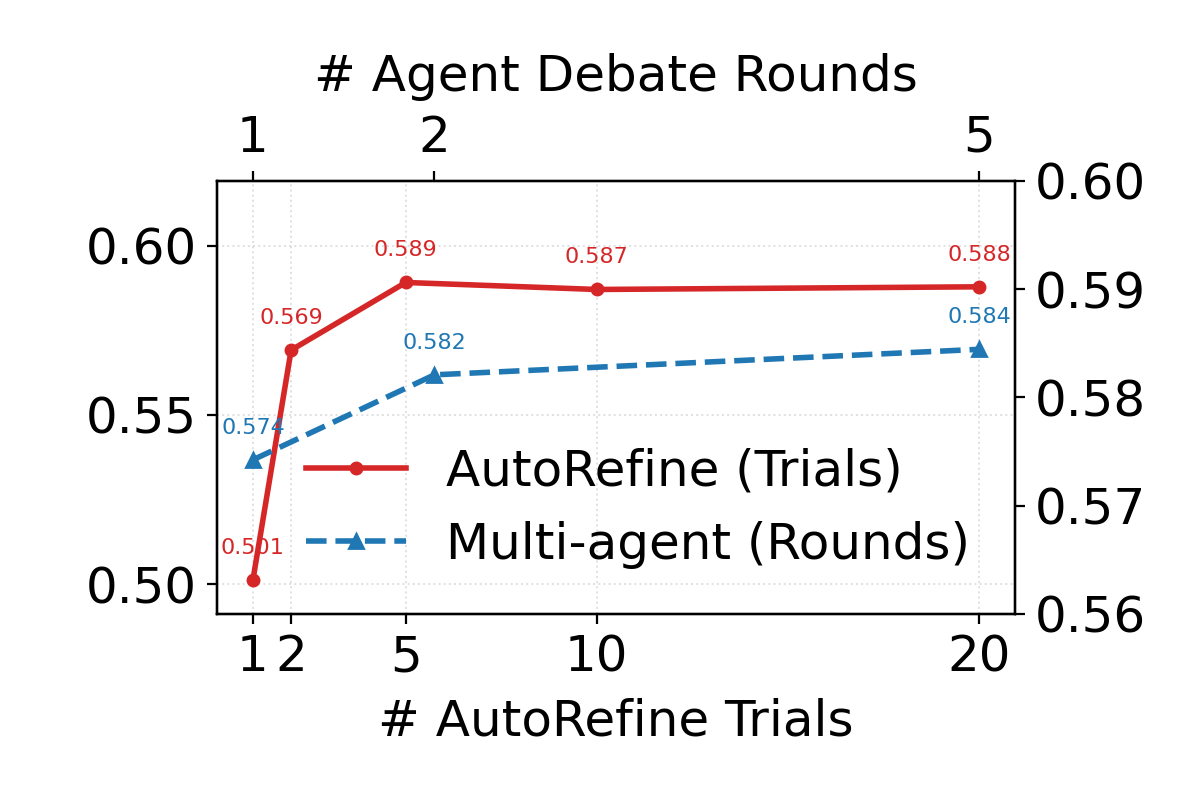}%
        \label{figure:num_trials_rounds_ablation}
    }
    \vspace*{-0.2cm}
    \caption{(a) Ablation and (b) hyperparameter study varying number of multi-agent debate rounds and AutoRefine trials.} \label{fig:ablation_hp_results}
\end{figure}

\textbf{Ablation Study (RQ2)}
Figure \ref{figure:ablation_figure} summarizes the incremental impact of AutoAdapt’s principal components on task accuracy, on the MedQA dataset. 
Specifically, we start with the base LLM, then add the ACG planning, best practices, multi-agent pipeline, AutoRefine (w/o and with GP). We note that on addition of every component, the results improve indicating the importance of the modules.

\textbf{Hyperparameter Study (RQ3)}
Figure \ref{figure:num_trials_rounds_ablation} examines the sensitivity of AutoAdapt to two parameters, namely number of AutoRefine trials and multi‑agent debate rounds, on the MedQA dataset. We observe a sharp gain for AutoRefine when increasing trials from 1 to 2, followed by further increase at 5 trials and then a plateau
This indicates the effectiveness of autorefine in efficiently converging to the optimal set of parameters in few trials. In parallel, multi‑agent rounds show steady, monotonic improvements 
suggesting that additional proposal–critique cycles systematically reduce planning bias and improve configuration quality. 
This is suggestive that a modest budget of 5-10 AutoRefine trials and 2 debate rounds is sufficient to yield optimal results.

\begin{figure}[!t]
    \centering
    \begin{subfigure}{0.18\textwidth}
        \centering
        \includegraphics[width=\linewidth]{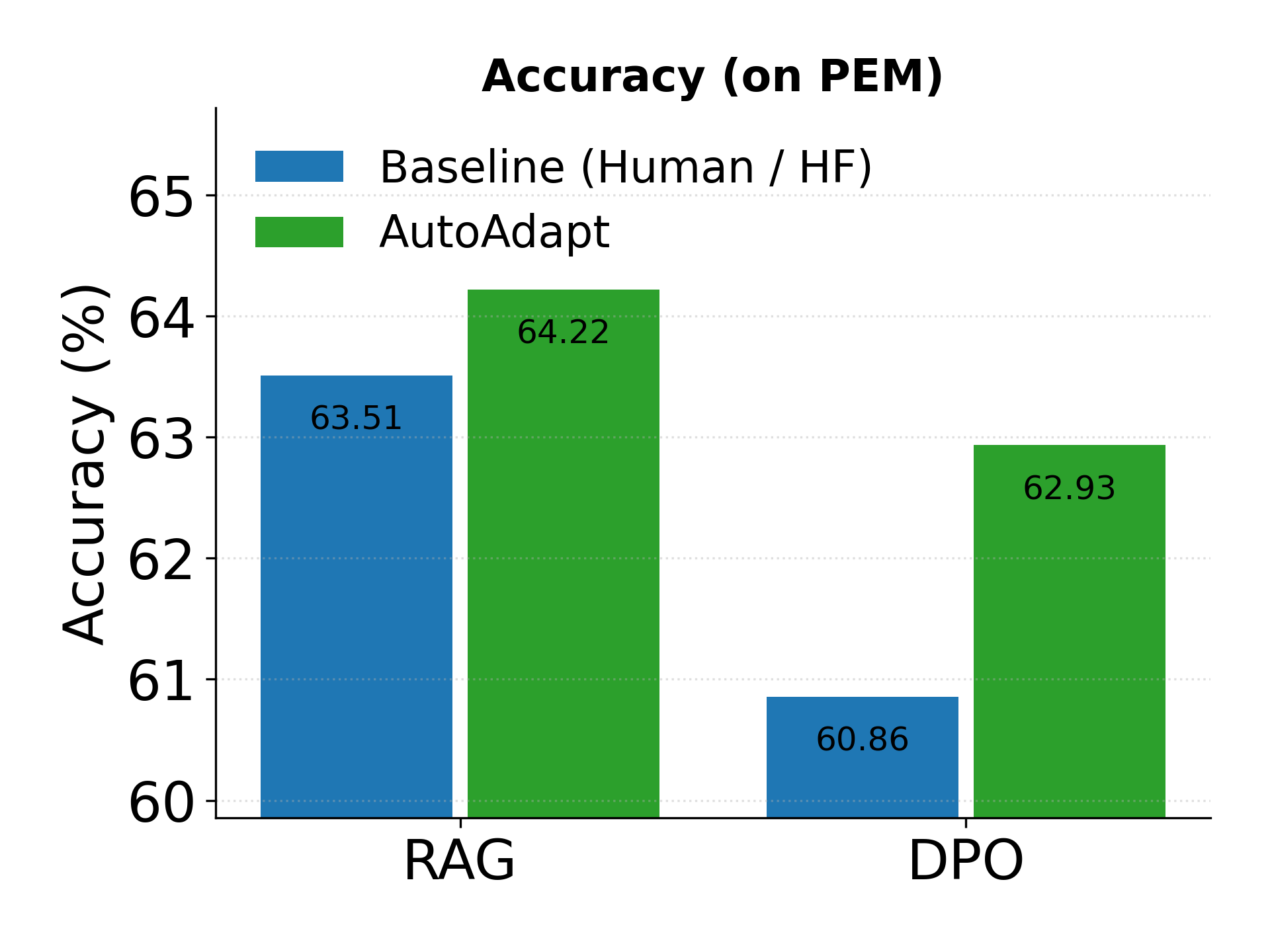}
        \caption{}
        \label{fig:rag_dpo}
    \end{subfigure}\hfill
    \begin{subfigure}{0.3\textwidth}
        \centering
        \includegraphics[width=\linewidth]{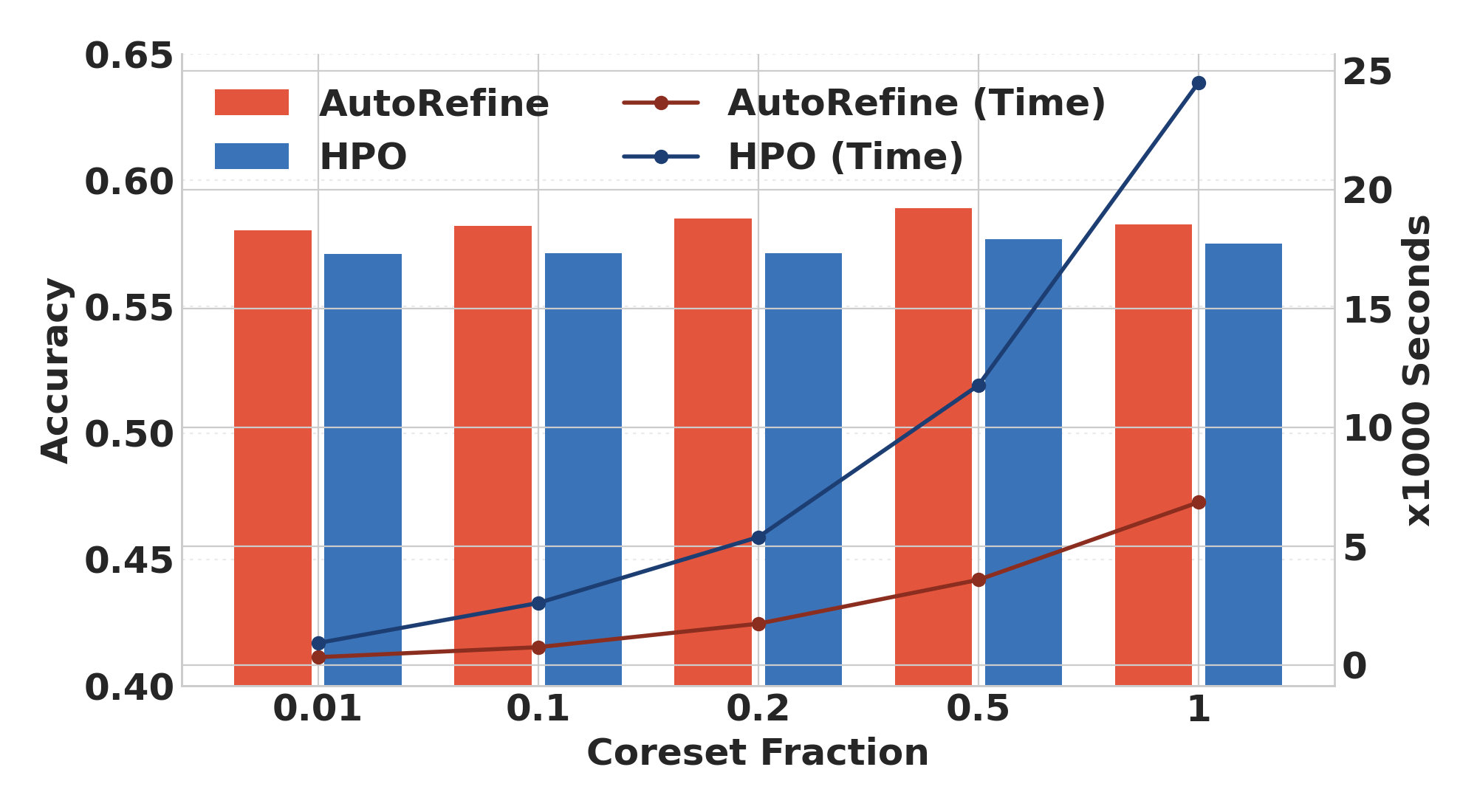}
        \caption{}
        \label{fig:accuracy_time}
    \end{subfigure}
    \caption{(a) RAG, DPO Results (b) Accuracy (bars) and wall-clock time (lines) by coreset sizes for HPO and AutoRefine.}
    \label{fig:accuracy_time_clustered}
\end{figure}

\textbf{Effectiveness of AutoRefine (RQ4)}
Figure \ref{fig:accuracy_time} compares AutoRefine with a standard Hyperparameter Optimization (HPO) baseline across increasing coreset fractions (0.01→1.0), reporting both accuracy and wall‑clock time. 
HPO was performed using the Optuna framework \cite{optuna} with \texttt{auto\_sampler},
till desired accuracy was achieved or maximum budget was attained.
AutoRefine consistently achieves higher or comparable accuracy at all fractions, with a stable performance while being efficient.
In practice, these results support using small coresets with AutoRefine to capture most of the achievable accuracy at a fraction of the cost of standard HPO. 

\begin{figure}[t!]
    \centering
    \subfloat[Cost breakdown (time and money) in multi-agent planning]{
        \includegraphics[width=0.98\linewidth]{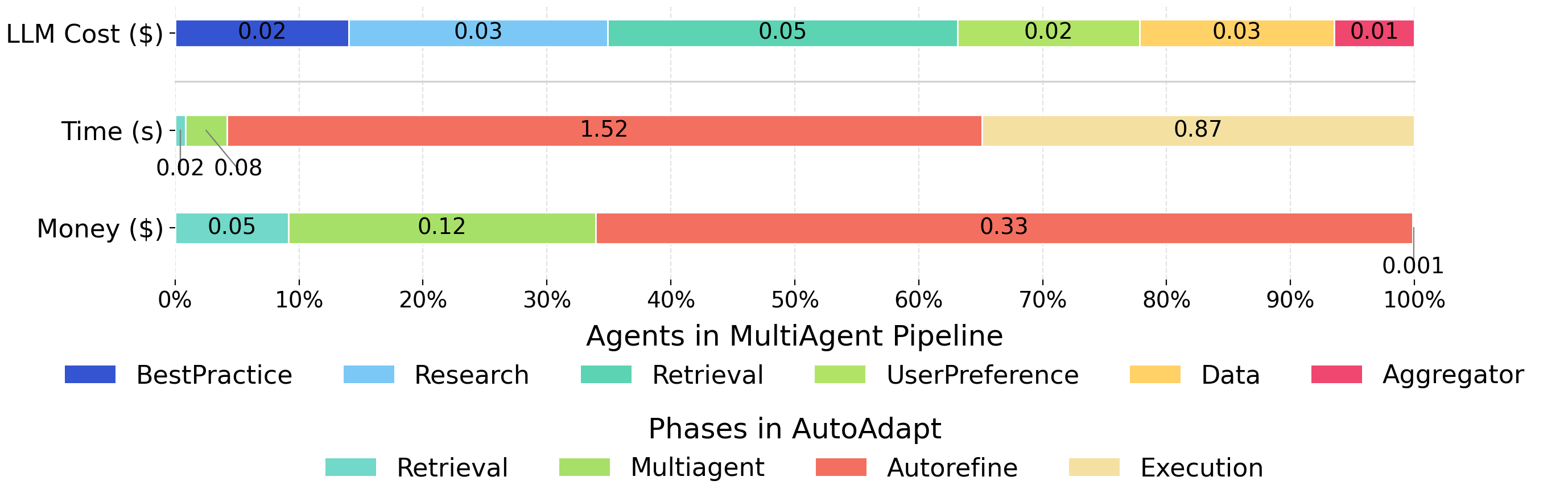}%
        \label{figure:cost_breakdown}
    }
    
    \subfloat[Time Overhead]{
        \includegraphics[width=0.49\linewidth]{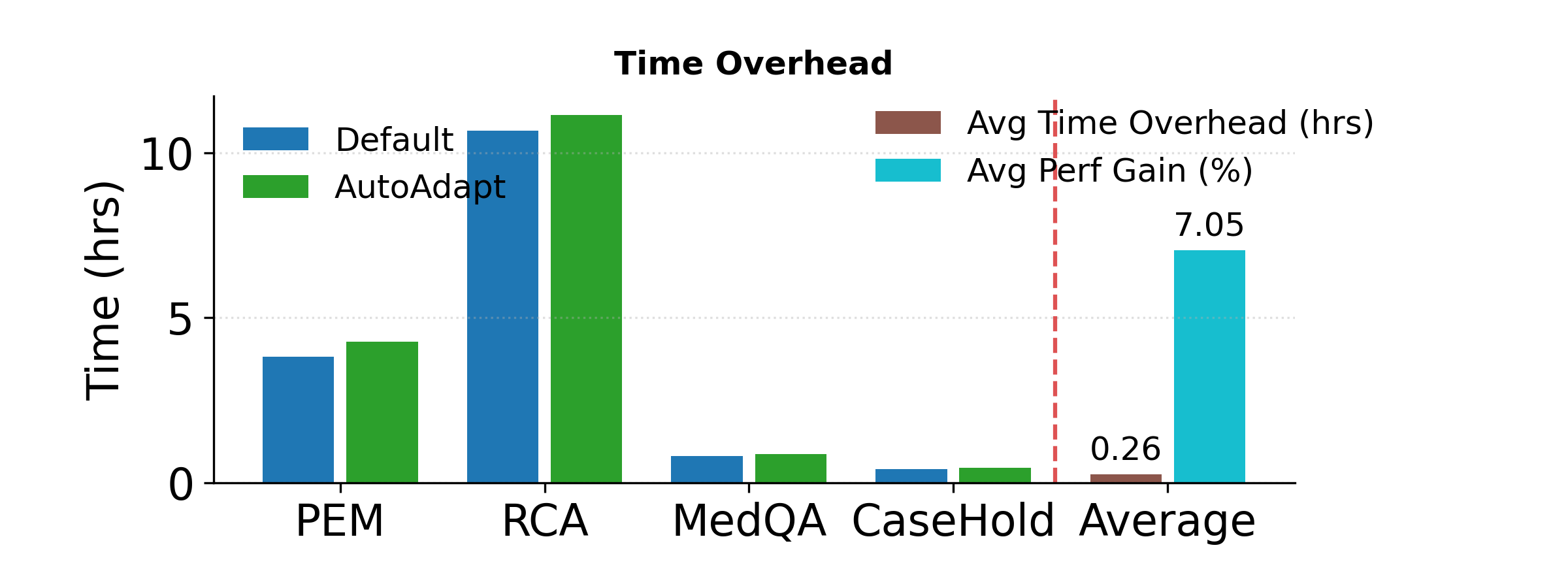}%
        \label{figure:time_overhead}
    }
    \subfloat[Cost Overhead]{
        \includegraphics[width=0.49\linewidth]{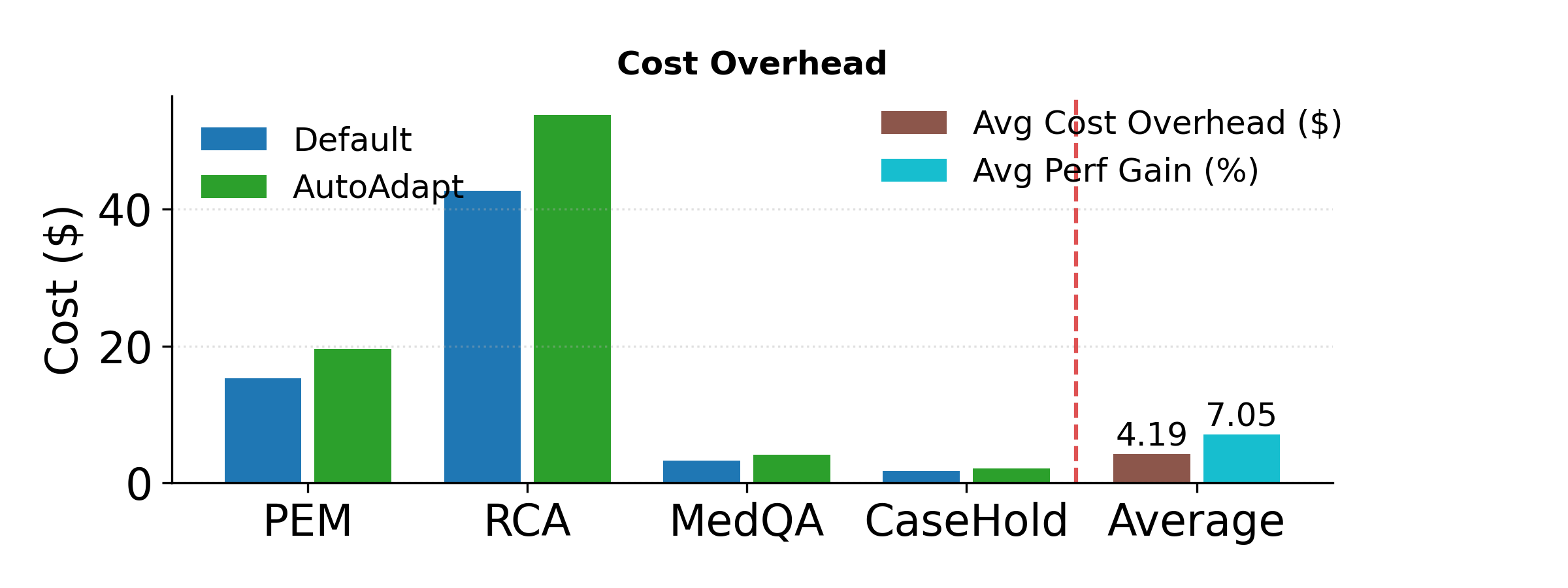}%
        \label{figure:cost_overhead}
    }
    \vspace*{-0.2cm}
    \caption{Top figure shows cost breakdown for MedQA across different agents in the MultiAgent planning phase. Bottom figures show the minimal time/money overhead and performance gains.} \label{fig:cost}
\end{figure}

\textbf{Time and Cost Analysis (RQ5)}
Figure \ref{fig:cost} quantifies the monetary and time cost incurred by AutoAdapt. 
In panel (a), the stacked breakdown shows the most fraction of time and LLM cost of the system are concentrated in the AutoRefine phase.
The reasoning is that AutoRefine runs over multiple iterations as compared to planning in order to find the optimal configuration. 
Panels (b) and (c) report the time and monetary overheads at the task level.
Averaged across datasets, \algname{}’s performance gains ($>7\%$) outpace its incremental time/cost ($\sim$ 15 mins / \$4), confirming a favorable 
trade-off,
validating the system’s efficacy for real‑world LLM domain adaptation.

\section{Conclusion}

In this work, we introduce \algname{}, an automated framework for domain adaptation of large language models.
\algname{} ensures reliable planning using a structured workflow graph. The proposed multi-agent pipeline grounded in best practices KB narrows the search space. Further, AutoRefine performs efficient optimization of the pipeline hyperparameters using a novel LLM guided algorithm grounded in probabilistic function estimation.
Through extensive experiments across diverse domains, we show that \algname{} consistently outperforms SoTA AutoML baselines.
\algname{} takes the first step towards end to end automation of LLM adaptation to real world scenarios.

\section*{Impact Statement}

\algname{} aims to lower the technical and resource barriers to domain adaptation of large language models by automating pipeline design, configuration selection, and refinement. By enabling efficient customization using smaller datasets and constrained budgets, \algname{} can reduce redundant experimentation, lower computational cost, and democratize access to high‑quality domain‑specific models for researchers and practitioners with limited resources. This can benefit applications in areas such as healthcare, education, legal analysis, and scientific research, where domain expertise exists but large‑scale machine learning infrastructure may not.

At the same time, \algname{} inherits the risks associated with adapting LLMs to sensitive domains, including the potential amplification of dataset biases or inappropriate behaviors if user‑provided data is flawed or unrepresentative. While \algname{} does not generate new data on its own in current state, its automation could accelerate deployment of domain‑adapted models without sufficient human review. To mitigate these risks, \algname{} is explicitly designed to surface configuration decisions, maintain modularity across stages, and enable human oversight through interpretable planning stages. We view \algname{} as a decision‑support system rather than a replacement for responsible model building evaluation.

Overall, \algname{} contributes to more efficient and transparent LLM adaptation workflows, with the potential to reduce environmental impact through compute savings while emphasizing the importance of careful downstream evaluation and governance in real‑world deployments.





\bibliography{arxiv/references}
\bibliographystyle{arxiv/icml2026}

\newpage
\appendix
\onecolumn

\section*{Appendix}\label{apndx}

The appendix of this paper is organized as follows. First, we introduce key preliminaries and notation in \autoref{apndx_prelims}. Next, we utilize this notation and present theoretical bounds for the \textit{AutoRefine }algorithm in \autoref{apndx_regret_bound}. We then describe our multi-agent system in detail, including the prompts used for individual agents, in \autoref{apndx_multiagent_system}. After that, we detail the datasets used in our study in \autoref{apndx_datasets}, followed by our study on best practices in \autoref{apndx_bp} to validate the relevance of the recommendations. Subsequently, we provide the detailed experimental setup in \autoref{apndx_expts} and the detailed numerical results in \autoref{apndx_res}. 
We also include a case study of AutoAdapt’s planning stage, highlighting how each agent contributes to the plan, in \autoref{apndx_casestudy}. Finally, we present a detailed description of the manual iterations performed for each baseline for the MedQA dataset for the sake of transparency in \autoref{apndx_sr}.

\section{Preliminaries}
\label{apndx_prelims}

\textbf{Multi-agent frameworks} have emerged as a powerful paradigm for complex decision-making tasks, enabling distributed reasoning through specialized agents that collaborate via a plethora of communication workflows. In the context of automated machine learning and domain adaptation, these frameworks allow decomposition of pipeline design into modular decisions, improving scalability and adaptability compared to monolithic approaches.

\subsection{Bayesian Optimisation}

Bayesian optimization (BO) \cite{frazier2018tutorialbayesianoptimization} is a sequential model-based approach for optimizing expensive black-box functions. It assumes an unknown objective function \( f(\mathbf{x}) \) and models it using a probabilistic surrogate, typically a Gaussian Process (GP) \cite{beckers2021introductiongaussianprocessmodels}. The GP defines a prior over functions and updates it with observed data \( \mathcal{D}_t = \{ (\mathbf{x}_i, y_i) \}_{i=1}^t \). The posterior predictive distribution at a new point \( \mathbf{x}_* \) is:

\[\mu(\mathbf{x}_*) = k(\mathbf{x}_*, X)^\top K^{-1} \mathbf{y},\]
\[\sigma^2(\mathbf{x}_*) = k(\mathbf{x}_*, \mathbf{x}_*) - k(\mathbf{x}_*, X)^\top K^{-1} k(X, \mathbf{x}_*) \]

where \( K \) is the kernel matrix with entries \( k(\mathbf{x}_i, \mathbf{x}_j) \).

The next query point is chosen by maximizing an \textit{acquisition function} \( \alpha(\mathbf{x}) \), which balances exploration and exploitation. Common choices include:
1) \textbf{Expected Improvement (EI):}
$
\alpha_{\text{EI}}(\mathbf{x}) = \mathbb{E}[\max(0, f(\mathbf{x}) - f^+)] = (\mu(\mathbf{x}) - f^+) \Phi(z) + \sigma(\mathbf{x}) \phi(z)
$
where \( z = \frac{\mu(\mathbf{x}) - f^+}{\sigma(\mathbf{x})} \), \( f^+ \) is the best observed value, and \( \Phi, \phi \) are the CDF and PDF of the standard normal distribution.
2) \textbf{Upper Confidence Bound (UCB):}
$
\alpha_{\text{UCB}}(\mathbf{x}) = \mu(\mathbf{x}) + \kappa \sigma(\mathbf{x})
$
where \( \kappa > 0 \) controls the exploration–exploitation trade-off.

\section{Related Works}\label{apndx_related_works}
Recent advances in LLM-powered AutoML frameworks can be broadly categorized into \emph{static} and \emph{dynamic} systems. Static recommendation approaches, such as AutoMMLab \cite{autommlab_yang2024autommlabautomaticallygeneratingdeployable} 
, select from a set of pre-prepared models and produce executable plans with limited adaptability to user-specific requirements. Dynamic systems, such as MLCopilot \cite{mlcopilot_zhang2024mlcopilotunleashingpowerlarge}, DS-Agent \cite{dsagent_guo2024dsagentautomateddatascience} and AutoML-Agent \cite{automlagent_trirat2025automlagent}, generate training plans and code
, offering better alignment with user tasks but often struggling with reliability and the complexity for long, multi-stage adaptation pipelines.

\textbf{Existing AutoML agents}: 
AutoGPT \cite{autoGPT_zhang2023automlgptautomaticmachinelearning} and HuggingGPT \cite{hugginggpt_shen2023hugginggptsolvingaitasks} connect LLMs with external tools and models to solve generic AI tasks. 
Agents such as \cite{dsagent_guo2024dsagentautomateddatascience, datainterpretor_hong2024datainterpreterllmagent, agenthpo_liu2025largelanguagemodelagent} focus on specific aspects of the ML training pipeline such as data preparation, analysis or HPO and do not consider the end to end pipeline for LLM adaptation.
\textbf{MLCopilot} \cite{mlcopilot_zhang2024mlcopilotunleashingpowerlarge} uses a two-stage process: an offline stage to build a knowledge pool from historical data and experience, and an online stage to retrieve demonstrations and generate solutions based on new task descriptions. 
\textbf{DS-Agent} proposes to use LLM agents along with case based reasoning (CBR) to come up with generic data science pipelines for a task. 
\textbf{AutoML-Agent} \cite{automlagent_trirat2025automlagent} is the closest work related to our method. It employs a multi-agent LLM framework using retrieved knowledge bases to 
automate generic ML training pipelines.
However, the method is not tailored to adaptation of LLMs which may not require finetuning (RAG, prompt optimization etc.) or go beyond finetuning (multi-stage RLHF, instruction tuning etc.). 
These methods fall short in reliably planning and executing complex, multi-stage LLM adaptation workflows. 

Different from the above works, AutoAdapt introduces a hybrid recommendation approach that combines the reliability of static planning (using ACG) with the adaptability of dynamic systems (multi-agent system guided by best practices KB). Unlike prior methods that typically support only single-stage fine-tuning, AutoAdapt provides a comprehensive framework for full domain adaptation pipelines, including multi-stage processes such as in-context prompting (eg. RAG), supervised fine-tuning (SFT), alignment (DPO) and so on.
AutoAdapt aims to \emph{reliably} and \emph{efficiently} generate executable, robust, and user-aligned \emph{LLM adaptation} plans, overcoming the limitations of existing AutoML systems.

\section{Theoretical Bounds on the Autorefine algorithm} \label{apndx_regret_bound}

This section details the cumulative regret of AutoRefine algorithm.

\paragraph{Setting.}
Let $\Theta$ be a compact hyperparameter domain equipped with a continuous--categorical relaxation, and let $\widehat{\Theta}\subset \Theta$ be the finite discretization used by the acquisition optimizer (e.g., multi-start LBFGS followed by discretization). 
Assume the black-box objective $f:\Theta\to\mathbb{R}$ is a sample from a zero-mean Gaussian process with a squared-exponential (SE/RBF) kernel 
\[
k_{\text{SE}}(x,x')
=\sigma_f^2\exp\!\left(-\frac{\|x-x'\|^2}{2\ell^2}\right),
\]
and that observations are noisy:
\(
y_t = f(\theta_t)+\varepsilon_t
\),
where $\varepsilon_t$ are $\sigma$-sub-Gaussian (independent). 
The AutoRefine posterior $\big(\mu_t,\sigma_t\big)$ is fit on the dataset $\mathcal{D}_t$ that combines ground-truth evaluations $\{(\theta_i,y_i)\}_{i=1}^{t}$ and LLM-derived trend pseudo-observations $\{(\tilde{\theta}_j,\tilde{y}_j)\}_{j=1}^{m_t}$, 
which are treated as \emph{unbiased, high-variance} measurements of $f$ (conditional on the GP), i.e.,
\[
\tilde{y}_j = f(\tilde{\theta}_j)+\tilde{\varepsilon}_j,
\quad 
\tilde{\varepsilon}_j\ \text{is } \tilde{\sigma}\text{-sub-Gaussian with } \tilde{\sigma}\gg \sigma,
\]
and thus incorporated with inflated noise variance so they do not over-confidently shrink $\sigma_t(\cdot)$.
At round $t$, AutoRefine algorithm selects
\[
\theta_t \in \arg\max_{\theta\in \widehat{\Theta}}\;
\mu_{t-1}(\theta)+\sqrt{\beta_t}\,\sigma_{t-1}(\theta),
\]
with a confidence schedule $\beta_t$ defined as \cite{iwazaki2025improvedregretboundsgaussian}. 
\begin{equation}\label{eq_beta_t}
\beta_t = 2 \ln \left( \frac{2 t^2 \pi^2}{3 \delta} \right)
+ 2d \ln \left( t^2 d b r \sqrt{ \ln \left( \frac{4 d a}{\delta} \right) } \right)    
\end{equation}
where $a,b$ are constants defined in lemma 1 of \cite{iwazaki2025improvedregretboundsgaussian}, $d$ is the dimension of the domain and $r$ is the length of the domain. Please refer \cite{iwazaki2025improvedregretboundsgaussian} for details. One subtle technical point to note is that in line 19 of Algorithm \ref{alg:pseudo_code}, the acquisition is done by an LLM that is prompted in-context using the optimal point from the GP-UCB algorithm. We prompt the LLM to decide between exploiting the optmized GP-UCB hyperparameters and exploring novel ones in the search space. As the algorithm progresses, there would be fewer novel points to explore and the algorithm would rely more on the GP-UCB optimized point. Thus this acquisition is equivalent to the epsilon-greedy policy with decay \cite{sutton2018reinforcement}, which asymptotically reduces to the GP-UCB acquisition itself. We need the LLM sampler as the acquisition as otherwise optimizing the function manifold may result in the sampler collapsing to the same optimal point (or one in the vicinity). The LLM sampler acts as a guardrail to prevent such trivial sampling using its domain knowledge.
Let the cumulative regret be $R_T:=\sum_{t=1}^T\big(f(\theta^\star)-f(\theta_t)\big)$, where $\theta^\star\in\arg\max_{\theta\in\Theta} f(\theta)$.

In this part of the proof, we bound the cumulative regret $R_T$ over time $T$. Here, the time $T$ represents the number of iterations over which the points in the domain (parameter space) are queried to obtain the function value (results). Note we use an LLM as a surrogate model to obtain intermediate results and do not perform evaluations by fitting the network over the sampled parameters and evaluating each iteration. The complexity of an LLM call is negligible compared to actually finetuning an LLM over the sampled parameters and evaluating to obtain the results at that point. As such we only consider the "actual function evaluations` for the complexity. Define $F_a$ to be the augmentation factor i.e. the fraction of the number of data points queried (by actual function evaluation and LLM surrogate) to the number of actual model evaluation.

We make use of the below result from \cite{srinivas2010gaussian} (provides bounds on the function value w.r.t. UCB):

\begin{lemma}[Confidence bounds \cite{srinivas2010gaussian}]\label{srinivas_confidence_bounds}
With probability $1-\delta$ ($\delta \in (0,1)$), For any $t \in \mathbb{N}_{+}$, the function value at the queried point $x_t$ satisfies
\[
\mu(x_t; X_{t-1}, y_{t-1}) - \beta_t^{1/2} \, \sigma(x_t; X_{t-1}) \leq f(x_t).
\]
Furthermore, for any $t \in \mathbb{N}_{+}$, any function value $f(x)$ on $\mathcal{X}_t$ satisfies
\[
f(x) \leq \mu(x; X_{t-1}, y_{t-1}) + \beta_t^{1/2} \, \sigma(x; X_{t-1}).
\]
\end{lemma}

Using the above result we bound the regret for Autorefine as below:

\paragraph{Claim I.}
For the SE/RBF kernel, the cumulative regret of the Autorefine algorithm (\ref{alg:pseudo_code}), with probability at least $1-\delta$ (for any $\delta\in(0,1)$) is,
\[
R_{T_a} \;=\; \mathcal{O}\!\left(\sqrt{\frac{T}{F_a}\,\ln^2 \left(\frac{T}{F_a}\right)}\right),
\]

\begin{proof}
The proof follows the GP--UCB template \cite{iwazaki2025improvedregretboundsgaussian} with two ingredients:
(i) high-probability confidence envelopes over the posterior; and
(ii) summability of posterior uncertainties at the points chosen by UCB.

\textbf{(1) Confidence Bound:}
By standard sub-Gaussian concentration and union bounds on the finite set $\widehat{\Theta}$ with the above $\beta_t$, we have w.h.p. (probability $\ge 1-\delta$) for all $t$ and all $\theta\in\widehat{\Theta}$ (c.f. lemma 1 in \citep{srinivas2010gaussian} and \ref{srinivas_confidence_bounds}):

\begin{equation}\label{f_lcb}
f(\theta)-\mu_{t-1}(\theta) \;\le\; \sqrt{\beta_t}\,\sigma_{t-1}(\theta).
\end{equation}

\begin{equation}\label{f_ucb}
f(\theta)-\mu_{t-1}(\theta) \;\ge\; -\sqrt{\beta_t}\,\sigma_{t-1}(\theta).
\end{equation}

The inclusion of unbiased pseudo-observations with inflated noise variance preserves the validity of ~\eqref{f_lcb} and ~\eqref{f_ucb}, as they are treated as additional high-variance measurements; such points only affect constants in the posterior covariance and thus in $\sigma_{t-1}(\cdot)$.

\textbf{(2) Instantaneous regret under UCB:}
Let $\Delta_t:=f(\theta^\star)-f(\theta_t)$, where $f$ is the actual evaluation (with possible finetuning using the parameters) of the network using selected parameters and the given data points.

Since the acquisition function ($\arg \max_{\theta \in \Theta} f(\theta)$) at time $t$ selected $\theta_t$, we have 
\begin{align*}
    \mu_{t-1}(\theta_t)+\sqrt{\beta_t}\sigma_{t-1}(\theta_t) &\ge \mu_{t-1}(\theta^\star)+\sqrt{\beta_t}\sigma_{t-1}(\theta^\star) \\
    &\ge \mu_{t-1}(\theta^\star)-\sqrt{\beta_t}\sigma_{t-1}(\theta^\star) \\
    \therefore \sqrt{\beta_t}\sigma_{t-1}(\theta_t)&\ge  \mu_{t-1}(\theta^\star) -\sqrt{\beta_t}\sigma_{t-1}(\theta^\star) \\
    &- \mu_{t-1}(\theta) \\
    \therefore f(\theta^\star) - \mu_{t-1}(\theta) &\le \sqrt{\beta_t}\sigma_{t-1}(\theta_t) \\
    &\text{, \{...From UCB eq. \}}\\
    f(\theta^\star) - (\mu_{t-1}(\theta) &- \sqrt{\beta_t}\sigma_{t-1}(\theta_t)) \le 2\sqrt{\beta_t}\sigma_{t-1}(\theta_t) \\
    f(\theta^\star) - f(\theta) &\le 2\sqrt{\beta_t}\sigma_{t-1}(\theta_t) \\
    & \text{, \{...From LCB eq. \}} \\
\end{align*}

Thus from above analysis we have
\begin{equation}
\label{eq:inst}
\Delta_t \;\le\; 2\sqrt{\beta_t}\,\sigma_{t-1}(\theta_t).
\end{equation}

\textbf{(3) Summation and improved SE bound.}
Since we are only interested in the regret at every iteration where the function is actually evaluated at the query point and the GP updated, we discard the iterations with LLM surrogate queries. Summing~\eqref{eq:inst} over $t=1,\dots,\frac{T}{F_a}$ gives
\begin{align*}
    R_{T_a} &\;\le\; 2\sum_{t=1}^{\frac{T}{F_a}} \sqrt{\beta_t}\,\sigma_{t-1}(\theta_t) \\
\end{align*}
For the SE/RBF kernel, a refined analysis of the \emph{realized} input sequence under UCB shows that the cumulative uncertainty weighted by $\sqrt{\beta_t}$ grows as $\mathcal{O}(\sqrt{T\,\ln^2 T})$ (c.f. \cite{iwazaki2025improvedregretboundsgaussian}). 
This yields
\begin{equation}\label{cumulative_regret}
    R_{T_a} \;=\; \mathcal{O}\!\left(\sqrt{\frac{T}{F_a}\,\ln^2 \left(\frac{T}{F_a}\right)}\right)
\end{equation}
with high probability. 

\textbf{Effect of pseudo-observations.}
If we consider the LLM pseudo observations in the regret, the bound comes out to be the following (same as in \cite{iwazaki2025improvedregretboundsgaussian})
\begin{equation}\label{eq_regret_T} 
    R_T = \mathcal{O}\!\left(\sqrt{T\,\ln^2 T}\right)
\end{equation}
Because pseudo-observations are incorporated as unbiased high-variance points, they do not invalidate ~\eqref{f_ucb}, ~\eqref{f_lcb}; they only alter constants in the posterior computation and hence in the final constant multiplying $\sqrt{T\,\ln^2 T}$. 
Therefore the rate is unchanged, concluding the proof of Part I bounding the regret from above.
\end{proof}

For the part II on the lower bound for the number of iterations, consider the \emph{simple regret} after $T$ rounds to be
\[
r_T := f(\theta^\star) - \max_{1 \le t \le T} f(\theta_t),
\]
where $\theta^\star$ is the global maximiser of $f$. 
From equation \ref{eq_regret_T} we obtain the average simple regret (over $T$ iterations) as 
\[
r_T = \tilde{\mathcal{O}}\!\Bigg(\sqrt{\frac{\ln^2 T}{T}}\Bigg).
\]

\paragraph{Claim II.}
To guarantee $r_T \le \varepsilon$ for some $\varepsilon > 0$, the number of actual function evaluation iterations $T_a$ must satisfy
\[
T_a = \Omega\!\Bigg(\frac{\ln^2(1/\varepsilon)}{F_a \varepsilon^2}\Bigg).
\]

\begin{proof}
From the regret rate,
\[
r_T \le c \sqrt{\frac{\ln^2 T}{T}}
\]
for some constant $c > 0$. To enforce $r_T \le \varepsilon$, we require
\[
c \sqrt{\frac{\ln^2 T}{T}} \le \varepsilon
\quad \Longrightarrow \quad
\sqrt{\frac{\ln^2 T}{T}} \le \frac{\varepsilon}{c}.
\]
Squaring both sides gives
\[
\frac{\ln^2 T}{T} \le \frac{\varepsilon^2}{c^2}
\quad \Longrightarrow \quad
T \ge \frac{c^2 \ln^2 T}{\varepsilon^2}.
\]
This is an implicit inequality because $T$ appears inside $\ln^2 T$. For large $T$, $\ln T$ grows slowly compared to $T$, so we approximate $\ln T \approx \ln(1/\varepsilon)$ (since $T$ will scale like $1/\varepsilon^2$ in lower bound). Substituting yields
\begin{align*}
    T &= \Omega\!\Bigg(\frac{\ln^2(1/\varepsilon)}{\varepsilon^2}\Bigg) \\
    F_a T_a &= \Omega\!\Bigg(\frac{\ln^2(1/\varepsilon)}{\varepsilon^2}\Bigg) \\
    T_a &= \Omega\!\Bigg(\frac{\ln^2(1/\varepsilon)}{F_a\varepsilon^2}\Bigg)
\end{align*}
which establishes the lower bound up to polylogarithmic factors.
\end{proof}

\begin{remark}[On LLM guidance and constants]
If trend points are slightly misspecified with bounded bias $b(\theta)$ but are re-centered (or modeled with hyperparameters that absorb the bias), the confidence envelopes remain valid up to a constant slack, which is absorbed into the choice of $\beta_t$; the asymptotic $\sqrt{T\,\ln^2 T}$ rate is preserved while constants improve when LLM guidance focuses the search in informative regions. 
If a fixed bias $B>0$ were injected without correction, an additional linear term $\mathcal{O}(BT)$ would appear from $\Delta_t\le 2\sqrt{\beta_t}\sigma_{t-1}(\theta_t)+2B$; thus, bias calibration or high-variance downweighting is essential to maintain sublinear regret.
\end{remark}

\begin{remark}[On iteration complexity]
The iteration complexity grows quadratically with $1/\varepsilon$ and only polylogarithmically with the domain size (via $\ln^2(1/\varepsilon)$). For example, to achieve $\varepsilon = 0.05$ (for e.g. within 5\% of the optimal accuracy) with augmentation factor $F_a=100$, we need
\begin{align*}
    T_{\min} &\approx \mathcal{O}\left(\frac{\ln^2(20)}{100\times(0.05)^2}\right) \\
    &\approx \mathcal{O}\left(\frac{(3)^2}{0.25}\right) \\&
    \approx \mathcal{O}\left(10\right) \text{ iterations.}
\end{align*}
Note, our algorithm does not change the asymptotic complexity compared to the GP-UCB algorithm but only changes the constant factors in the complexity. However, the gains are practically significant as our algorithm only needs few tens of iterations vs perhaps thousands to converge to the optimal point, which is desirable considering the costly training of LLMs. 
\end{remark}
\section{Adaptation Configuration Graph}\label{apndx_acg}
In this section, we describe in detail the design of the ACG utilized to run the multi agent planning pipeline.
Figure \ref{fig:apndx_acg} illustrates the Adaptation Configuration Graph (ACG) used in our framework to structure and constrain the space of large language model (LLM) domain adaptation workflows. The ACG is a directed acyclic graph that explicitly encodes high‑level adaptation strategies, their dependencies, and the corresponding parameterized training configurations, enabling systematic and reliable planning of multi‑stage adaptation pipelines.

At the highest level, the ACG branches from a root adapt node into alternative adaptation strategies, including retrieval‑augmented generation (RAG), supervised fine‑tuning (SFT), and preference‑based alignment (DPO). Each strategy node captures a semantically distinct mode of adaptation and is connected via requires edges to the concrete configuration blocks needed to instantiate that strategy. This design allows the planner to reason explicitly over which adaptation techniques are applicable under given data, resource, and user constraints, rather than treating the pipeline as an unstructured sequence of hyperparameters.

For RAG, the graph specifies the required retrieval configuration, including the embedding model, retriever parameters (e.g., top‑k), and the base model used for generation. This separation allows retrieval‑specific choices to be optimized independently of downstream training decisions.
For SFT and DPO, the ACG links the corresponding strategy nodes to detailed training configurations that define the base model, optimization hyperparameters (learning rate, batch size, gradient accumulation), logging and evaluation intervals, and training duration. By explicitly modeling these configurations as graph nodes, the ACG exposes the hierarchical structure of LLM adaptation decisions.

The graph further captures parameter‑efficient fine‑tuning (PEFT) as an intermediate abstraction. Both SFT and DPO may route through a PEFT node, which in turn branches into concrete techniques such as LoRA and prefix tuning. Each PEFT variant is associated with its own parameter block, specifying method‑specific hyperparameters (e.g., LoRA rank and scaling factors, prefix length) as well as shared efficiency settings such as quantization type. This explicit factoring enables the planner to reuse shared structure while still exploring method‑specific configurations.

Overall, the ACG serves two key purposes. First, it narrows the effective search space by enforcing valid structural dependencies between adaptation stages and parameters, thereby preventing ill‑formed or non‑executable pipelines. Second, it provides a transparent and interpretable representation of the adaptation workflow, which can be traversed and refined by downstream optimization components. By grounding pipeline construction in the ACG, our framework achieves reliable execution while retaining flexibility to support diverse adaptation strategies and training regimes.

\begin{figure}
    \centering
    \includegraphics[width=0.95\linewidth]{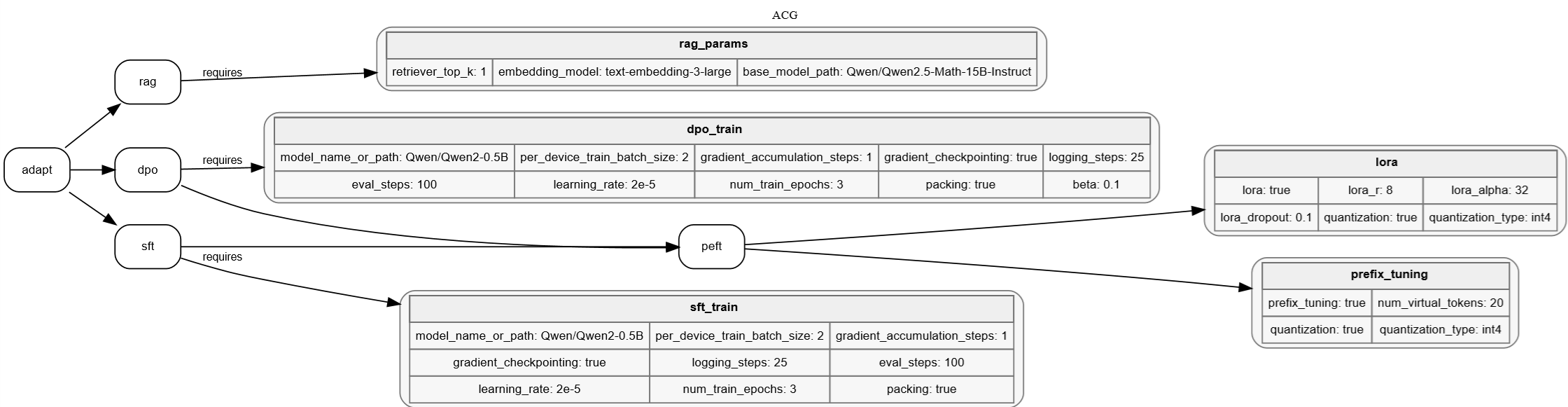}
    \caption{Adaptation Configuration Graph (ACG) for the LLM adaptation pipeline with example parameter values at each parameter node. }
    \label{fig:apndx_acg}
\end{figure}

\section{Multi Agent System}\label{apndx_multiagent_system}

In this section we detail the agents used in the multi-agent pipeline and the corresponding prompts. All agents use OpenAI's GPT-4.1-mini as the underlying model.
Below are the agents used in the system:

\subsection{Proposal Agents}

\textbf{Best Practice Agent}
 consolidates empirical insights from prior training pipelines, model cards, and benchmark results. It provides principled priors over adaptation strategies, hyperparameter ranges, and architectural choices, enabling the system to initialize planning in a region of the configuration space that historically yields strong performance.
 
\begin{promptBox}{Best Practice Agent Recommendation Prompt} 
\begin{lstlisting}
You are an expert in machine learning. We are working in a multi-agent framework for finding the best LLM model/
config for a user dataset/task. An orchestrating agent would ask certain questions from you. In order to provide 
the response you need to read through the historic message content and the current ask from the 
orchestrator. You will be given user task, the orchestrating LLM ask and also the historic message content. 
Please provide the answer to the orchestrator request adhering to the user request in the history  in order 
to solve the broader user task.

Following are the historic messages: {chat_history}
Following is the user task: {task_description}
Following is the Orchestrator ask: {orchestrator_ask} 

Below is the context from the retrieved database of the datasets/tasks and models best practices based on the 
user query:
{best_practice}

### Last round of proposal from the solver agents
{last_round_proposals}

### Last round of critique from the critique agents
{last_round_critique}

We need to get back to the user with a concrete recommendation and a reasoning.
Please answer with reason the following Orchestrator ask keeping the history and best practices in context:
{orchestrator_ask}
\end{lstlisting}
\end{promptBox}

\textbf{Research Agent}
    processed the literature, technical reports, and model documentation (for eg. gathered by the retrieval agent) to identify emerging techniques or underexplored alternatives. This ensures that the system incorporates innovations beyond those reflected in curated best‑practice corpora.

\begin{promptBox}{Research Agent Prompt} 
\begin{lstlisting}
### Task Description
{task_description}

### Retrieved Knowledge
{retrieved_knowledge}

### Last round of proposal from the solver agents
{last_round_proposals}

### Last round of critique from the critique agents
{last_round_critique}

### Instruction for ResearchAgent
You are a research agent with deep domain knowledge in machine learning training pipelines. \
You can reason using both the retrieved knowledge and your own LLM capabilities.

Provide a concise and structured response that includes:
1. Justification of the recommended approach using both the retrieved knowledge and logical inference.
2. A final recommendation with key configuration or design suggestions.
\end{lstlisting}
\end{promptBox}

\textbf{Knowledge Retriever Agent} dynamically retrieves and aggregates knowledge from multiple web sources (arXiv papers, Hugging Face model/data cards, and web search relevant to a user's task. It queries these sources, then summarizes the collected information into actionable insights using an LLM. This agent helps provide up-to-date, real-world context about state-of-the-art models, techniques, and best practices to support downstream decision-making in the multi-agent framework.

\begin{promptBox}{Web Search Query Generation Prompt} 
\begin{lstlisting}
You are tasked with generating web search queries based on finetuning LLM for the given user problem. Give me a 
specific query for web search focusing on the LLM model, configuration pipeline and technique that should 
be used for finetuning the LLM model for the specified user task and requiement. Please give me a single 
sentence within 10 words. The answer should not exceed 10 words.

The user's requirements are summarized as follows:
{user_requirement_summary}
\end{lstlisting}
\end{promptBox}

\begin{promptBox}{arXiv Summary Prompt} 
\begin{lstlisting}
I searched the arXiv papers using the keywords: {task_kw} and {domain_kw}. Here is the result:
=====================
{context}
=====================

Please summarize the given pieces of arXiv papers into a single paragraph of useful knowledge and insights. We 
aim to use your summary to address the following user's requirements.    
# User's Requirements
{user_requirement_summary}
\end{lstlisting}
\end{promptBox}

\begin{promptBox}{Hugging Face Model Cards Summary Prompt} 
\begin{lstlisting}
I searched the huggingface model cards using the task domain: {user_requirements["domain"]}. Here is the result:
=====================
{all_model_card_details}
=====================

Please summarize the given pieces of huggingface model cards into a single paragraph of useful knowledge and 
insights. We aim to use your summary to address the following user's requirements.    
# User's Requirements
{user_requirement_summary}
\end{lstlisting}
\end{promptBox}

\begin{promptBox}{Hugging Face Data Cards Summary Prompt} 
\begin{lstlisting}
I searched the huggingface data cards using the task domain: {user_requirements["domain"]}. Here is the result:
=====================
{all_data_card_details}
=====================

Please summarize the given pieces of huggingface data cards into a single paragraph of useful knowledge and 
insights. We aim to use your summary to address the following user's requirements.    
# User's Requirements
{user_requirement_summary}
\end{lstlisting}
\end{promptBox}

\begin{promptBox}{Final Knowledge Aggregation Prompt} 
\begin{lstlisting}
Please extract and summarize the following group of contents collected from different online sources into a chunk of insightful knowledge. Please format your answer as a list of suggestions. I will use them to address the user's requirements in machine learning tasks.

# Source: Web search
{search_summary}
=====================

# Source: arXiv Papers
{arxiv_summary}
=====================

# Source: Huggingface model cards
{hf_model_card_summary}
=====================

# Source: Huggingface data cards
{hf_data_card_summary}
=====================

The user's requirements are summarized as follows.
{user_requirement_summary}
\end{lstlisting}
\end{promptBox}

\subsection{Critic Agents}

\textbf{User Preference Agent}
encodes user‑specified constraints such as compute budgets, training preferences, method restrictions, and domain‑specific requirements and enforces them throughout the adaptation process. It operationalizes user intent by filtering infeasible configurations and steering the system toward personalized solutions.

\begin{promptBox}{User Preference Agent Prompt} 
\begin{lstlisting}
You are a strict reviewer for AI training proposals. A user has provided their preferences:
{user_preferences}

Below are proposals from other agents:
{proposals}

Your job is to:
- Identify where each proposal violates user preferences (if any)
- Recommend the most aligned proposal (if any)

Output a structured review in markdown with per-agent critique.

Note that the user preferences describe the global training pipeline and may not be directly relevant to the 
current proposal. 
If you determine that the preferences do not apply to this specific node, you may state that no critique is 
needed.
\end{lstlisting}
\end{promptBox}

\textbf{Data Agent}
analyzes the statistics of the provided dataset, including size, token distributions, domain characteristics, entropy measures, label properties etc. It transforms these signals into actionable quantitative summaries of the dataset. This agent also critiques the proposal based on data constraints.

\begin{promptBox}{Data Agent Prompt} 
\begin{lstlisting}
You are a strict reviewer for AI training proposals. The user has provided the following information:
{data_statistics}

Below are proposals from other agents:
{proposals}

Your job is to:
- Identify where each proposal violates user data assumptions, such as distribution mismatch, inappropriate 
model choice, incorrect input formatting, or unrealistic training configurations.
- Recommend the most aligned proposal (if any)

Output a structured review in markdown with:
1. Per-agent critique
2. Final recommendation
\end{lstlisting}
\end{promptBox}

\subsection{Aggregator Agent}
\textbf{Aggregator Agent}
 integrates outputs from all upstream agents into a unified context representation. It resolves conflicts, normalizes heterogeneous signals, and constructs the information state consumed by the multi‑agent planner. This aggregation step ensures coherent and consistent decision‑making across the system.

\begin{promptBox}{Aggregator Agent Prompt} 
\begin{lstlisting}
### Proposals from the Solver Agents
{proposals}

### Critiques from the Critique Agents
{critiques}

### Instructions for Aggregator Agent
You are an expert agent with deep domain knowledge in machine learning training pipelines. Your role is to 
aggregate the proposals from solver agents using your own reasoning capabilities. No external or retrieved 
knowledge is provided.
Always prioritize preferences or constraints specified by the data agent or user above all other considerations.

Provide a concise and structured response that includes:
1. Justification for the recommended approach based on logical reasoning and synthesis of the given proposals.
2. A final recommendation with key configuration or design suggestions aligned with user or data agent 
preferences.
\end{lstlisting}
\end{promptBox}

\section{Datasets}\label{apndx_datasets}
We conduct experiments on diverse domains to demonstrate the generalizability of \algname{}.
The domains are as follows:
1) Mathematical Reasoning
2) Medical Question Answering
3) Legal 
4) Root Cause Analysis of Cloud Incidents (Proprietary Dataset)
5) Natural Language
6) Ecommerce
7) Programming
8) Generic reasoning

Specifically we select the following benchmark and real world datasets:

(1) \textbf{MATH} \cite{hendrycksmath2021} dataset by UC Berkeley focused on mathematical reasoning, covering algebra, calculus, and advanced problem-solving. 

(2) \textbf{MedQA} \cite{medqa_jin2020disease} is a medical question answering dataset requiring domain-specific knowledge and reasoning.

(3) \textbf{CaseHold} \cite{casehold_zheng2021doespretraininghelpassessing} is a legal reasoning dataset for predicting case outcomes, testing adaptation to specialized legal language

(4) \textbf{PEM} is a customer dataset containing pre-engineering level math questions and answers.

(5) \textbf{RCA} is a proprietary dataset containing the root cause analysis of cloud incidents in a large scale cloud service provider.

(6) \textbf{ARC} \cite{allenai_arc} consists of multiple-choice science questions spanning elementary to middle-school curricula, designed to assess commonsense reasoning and knowledge integration across diverse scientific domains.

(7) \textbf{MBPP} \cite{mbpp_austin2021programsynthesislargelanguage} is a code generation benchmark composed of python problems

(8) \textbf{Ecomm} \cite{automlagent_trirat2025automlagent} is a task of classification of e-commerce products into relevant categories.

(9) \textbf{Entail} \cite{automlagent_trirat2025automlagent} is a task of natural language inference given premise and hypothesis.

(10) \textbf{When2Call} \cite{when2call-ross-etal-2025} is designed to deduce tool invocation ability.

\section{Best-Practices Knowledge Base}
\label{apndx_bp}
In this section we explain the construction of the  best practices KB and provide studies on the best practices retrieved and the consistency and relevance of the recommendations provided using these.

\subsection{KB Construction}

To ground the adaptation planning process in empirical evidence and widely adopted practices, we construct a Best‑Practices Knowledge Base (KB) that aggregates and structures information from diverse public sources on large language model (LLM) training, adaptation, and evaluation. The primary goal of the KB is to expose reliable priors over model choices, adaptation strategies, and hyperparameter configurations, which can be leveraged by the multi‑agent planner to generate executable and effective adaptation pipelines.

\textbf{Knowledge Sources}
We curate the KB from multiple complementary sources that reflect both community standards and state‑of‑the‑art empirical results. Specifically, we collect information offline from: (i) model cards and dataset cards hosted on Hugging Face, which provide structured descriptions of training regimes, datasets, parameter choices, evaluation metrics, and reported performance; (ii) fine‑tuning code repositories and notebooks from platforms such as GitHub and Unsloth, which expose concrete implementation details and hyperparameter settings used in practice; and (iii) leaderboards and evaluation summaries from Papers with Code, which link datasets to top‑performing models and report comparative metrics. These sources provide complementary perspectives, ranging from high‑level summaries to low‑level configuration details.

\textbf{Knowledge Extraction and Structuring}
Raw information from the above sources is heterogeneous and largely unstructured. To address this, we employ an LLM‑based extraction pipeline to convert documents and code into structured representations. For each model or dataset entry, we extract salient attributes such as dataset domain, size, task type, input–output characteristics, model family, parameter count, training pipeline (e.g., SFT, PEFT, DPO), and reported hyperparameter ranges. In the case of code repositories, we further extract explicit parameter instantiations (e.g., learning rate, batch size, number of epochs, PEFT‑specific settings) from training scripts and notebooks. All extracted items are normalized and stored as structured records. We manually validate a handful of the extracted data, however due to the scale of the data manual verification of the entire KB is left to future works.

\begin{promptBox}{Example Model Card derived during the offline KB construction phase} 
\begin{lstlisting}
 {      "base_model": "unsloth/Llama-3.2-3B-Instruct",
        "training_technique": "sft",
        "training_datasets": {
                "mlabonne/FineTome-100k": 100000
        },
        "dataset_structure": "text",
        "instruct": 1,
        "task": [],
        "evaluation": {},
        "quantization": 1,
        "peft": 1,
        "lr": 0.0002,
        "epochs": null,
        "batch_size": 2,
        "gradient_accumulation_steps": 4,
        "optimizer": "adamw_8bit",
        "scheduler": "linear",
        "weight_decay": 0.01,
        "mixed_precision": "fp16",
        "quantization_method": null,
        "lora_r": 16,
        "lora_alpha": 16,
        "lora_dropout": 0,
        "target_modules": [
                "q_proj",
                "k_proj",
                "v_proj",
                "o_proj",
                "gate_proj",
                "up_proj",
                "down_proj"
        ],
        "dpo_beta": null,
        "training_hardware": null,
        "library": [
                "unsloth",
                "datasets",
                "transformers",
                "trl",
                "peft"
        ],
        "domain": "general"}
\end{lstlisting}
\end{promptBox}

\textbf{Indexing and Retrieval}
Each record in the KB is indexed using a combination of semantic descriptors (task type, domain, dataset characteristics) and model attributes (architecture family, training method, parameter scale). At planning time, the user task description and dataset metadata are used to query the KB and retrieve a set of relevant model–dataset–configuration tuples. This retrieval step provides the multi‑agent system with concrete empirical precedents drawn from similar tasks and domains, serving as an evidence‑based starting point for decision making.

\textbf{Leveraging Best Practices for Adaptation Planning}
The KB is directly consumed by the Best‑Practice Agent during the planning stage. Given a candidate adaptation decision (e.g., selecting PEFT over full fine‑tuning), the agent uses retrieved KB entries to justify choices and to propose bounded hyperparameter ranges rather than arbitrary or overly broad search spaces. For example, learning rate ranges, batch size options, or training epoch counts are derived from configurations that have been demonstrated to work on comparable tasks and model scales. These KB‑derived ranges are subsequently refined during the AutoRefine stage based on user data and observed performance. Figure \ref{best_practice_retrieval} shows the retrieval schema by best practice agent from the KB.

\begin{figure}
    \centering
    \includegraphics[width=0.8\linewidth]{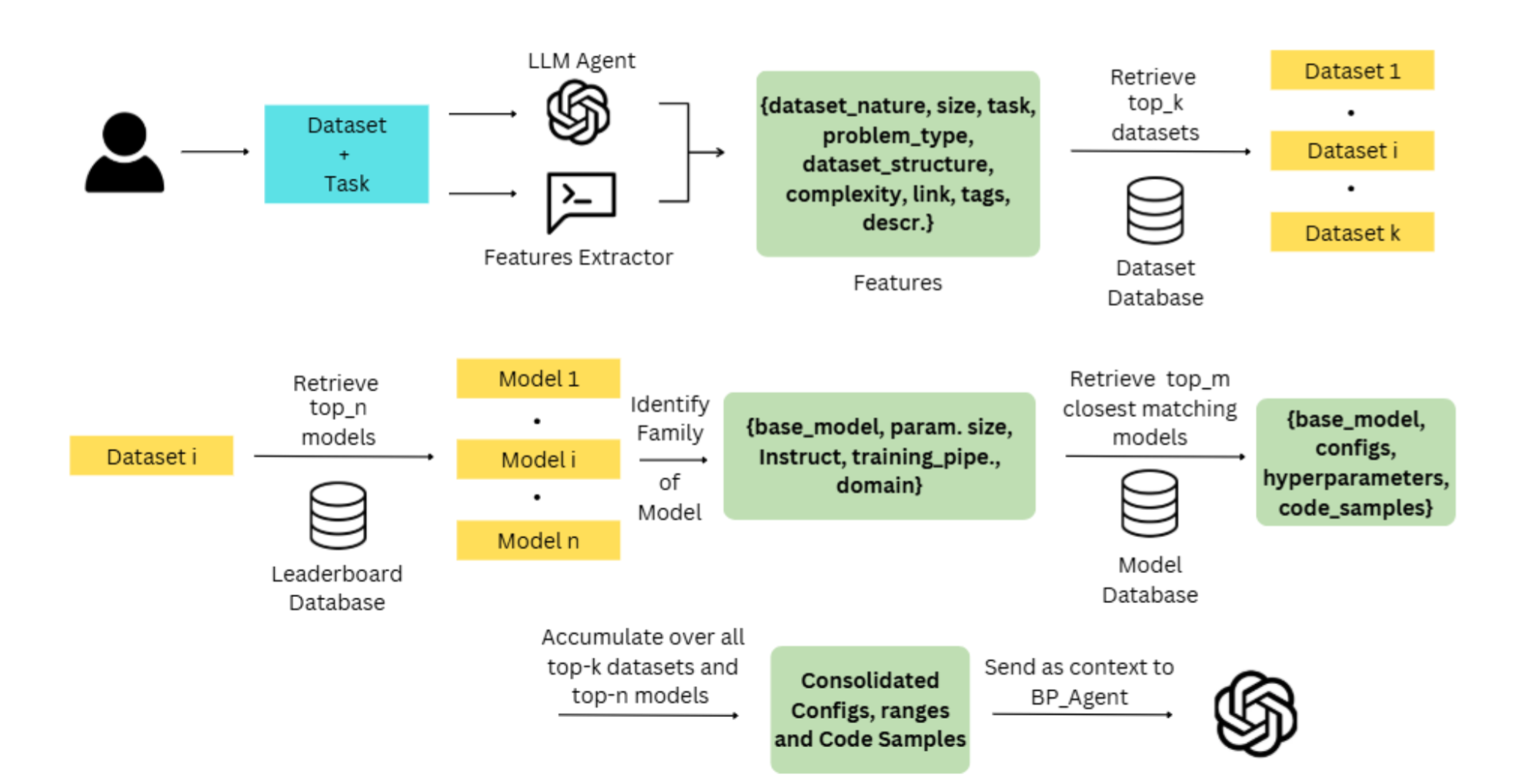}
    \caption{Best practices retrieval pipeline from KB.}
    \label{fig:best_practice_retrieval}
\end{figure}

\textbf{Scope and Limitations}
The KB is not intended to be exhaustive. Instead, it functions as a practical, curated prior over adaptation decisions for the domains and tasks considered in our experiments. We additionally support dynamic, query‑time retrieval to incorporate recent or task‑specific information beyond the offline corpus when needed. This design strikes a balance between coverage, accuracy, and computational efficiency.
Overall, the Best‑Practices Knowledge Base enables \algname{} to ground its recommendations in established empirical evidence in LLM adaptation, reducing reliance on ad‑hoc heuristics or unconstrained exploration, while maintaining flexibility to adapt to user‑specific data and requirements.

\subsection{Dispersion Index of Recommended Parameter Ranges}

To quantify the diversity of hyperparameter range recommendations produced by our multi-agent framework, we introduce the \textit{Dispersion Index (DI)}. The DI captures the degree of variability across a set of configurations by jointly considering categorical and numeric dimensions. Unlike pairwise similarity measures that focus on individual comparisons, our metric summarizes the overall spread of the recommendation set, providing a single interpretable score.

Formally, let $\mathcal{C} = \{c_1, c_2, \dots, c_N\}$ denote the set of recommended configurations. Each configuration comprises categorical parameters (e.g., activation functions, optimizer choices) and numeric parameters (e.g., learning rate, batch size). For any pair $(c_i, c_j)$, we compute:
\begin{equation}
\text{Sim}(c_i, c_j) = \alpha \cdot J(c_i, c_j) + (1-\alpha) \cdot \overline{\text{IoU}}(c_i, c_j),
\end{equation}
where $J(\cdot,\cdot)$ is the Jaccard similarity over categorical tokens, $\overline{\text{IoU}}(\cdot,\cdot)$ is the weighted average Intersection-over-Union across numeric ranges, and $\alpha \in [0,1]$ balances the contributions of categorical and numeric components. The \textit{Dispersion Index} is then defined as:
\begin{equation}
\text{DI} = \frac{2}{N(N-1)} \sum_{i<j} \big(1 - \text{Sim}(c_i, c_j)\big).
\end{equation}

Intuitively, DI approaches zero when all configurations are nearly identical and increases as recommendations become more diverse. By incorporating both discrete and continuous dimensions, DI provides a holistic measure of variability that is particularly suited for mixed-type hyperparameter spaces. This metric enables us to assess whether the optimizer explores the search space broadly or converges to a narrow region, which is critical for evaluating exploration--exploitation trade-offs in automated tuning.

\begin{figure}[t]
    \centering
    \captionsetup{font=small}

    \begin{subfigure}[t]{0.49\linewidth}
        \centering
        \includegraphics[width=\linewidth]{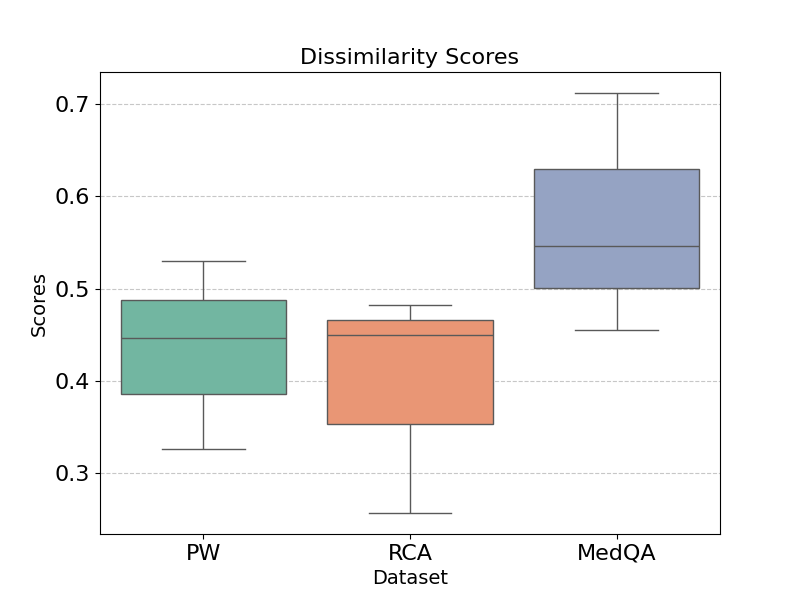}
        \caption{Dispersion index by dataset.}
        \label{fig:dispersion}
    \end{subfigure}\hfill
    \begin{subfigure}[t]{0.49\linewidth}
        \centering
        \includegraphics[width=\linewidth]{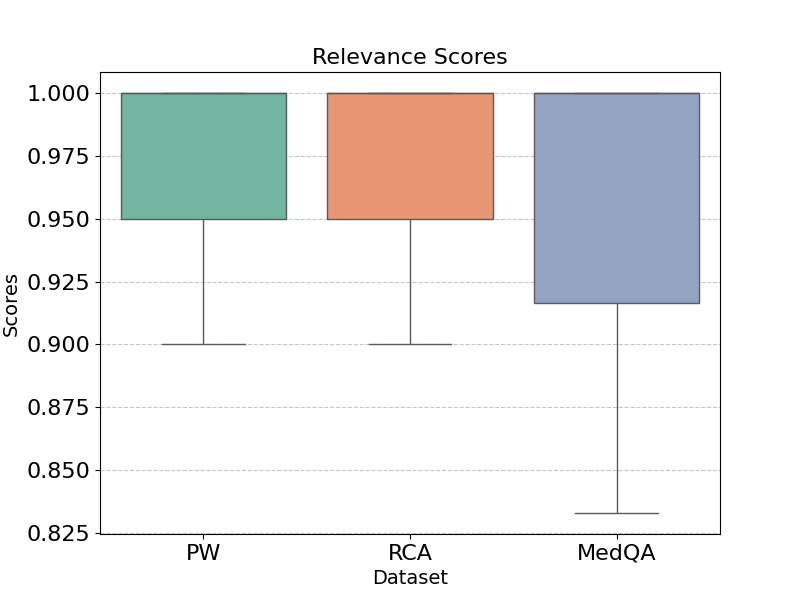}
        \caption{Relevance scores by dataset.}
        \label{fig:relevance_boxplot}
    \end{subfigure}

    \caption[Dispersion and relevance boxplots across datasets]{
    (Left) Boxplot showing the dispersion index metric for each dataset. We observe a median score around $0.5$, indicating a balance between consistency and diversity in the recommendations.
    (Right) Boxplot showing relevance scores of the retrieved best practices. We observe consistently high relevance across datasets.}
    \label{fig:dispersion_relevance_boxplots}
\end{figure}

Figure~\ref{fig:dispersion} reports the distribution of dissimilarity scores across three datasets: PEM, RCA, and MedQA. We observe that the median dispersion for PEM and RCA lies around 0.45, while MedQA exhibits a higher median near 0.55 with a broader inter-quartile range. Overall, the dispersion values cluster around 0.5, indicating a balance between consistency and variability in the recommended hyperparameter configurations. This suggests that our optimization strategy avoids excessive convergence to a single configuration while maintaining sufficient overlap to ensure stability. Notably, MedQA shows greater variability, reflecting the increased complexity and heterogeneity of this task, whereas PEM and RCA remain comparatively more consistent.

\subsection{Recommendation Relevance}

In order to study the relevance of the best practice sources used for recommendations, we perform a human annotation. Specifically we employ two human experts to review the best practice sources retrieved as well as the task and final recommendations. The annotators then score each retrieved source in a binary manner i.e. 0 if it is not relevant to the task and 1 if it is relevant. We then average these scores across all sources for a given task and run. We experiment on the three datasets: PEM, RCA, MedQA and run the best practice agent 10 times per dataset. The results are reported in figure \ref{fig:relevance_boxplot}, where we see a high relevance across all datasets confirming that the retrieved sources are indeed relevant to the task. In order to eliminate the human bias, we compute the Cohen's Kappa score (\cite{CohenKappa}) and find a high inter rater agreement ( $> 0.9$) between the annotators.

\subsection{Example Recommendations}

Figure ~\ref{fig:param-json} summarizes an example hyperparameter range recommended by our multi-agent pipeline for the MedQA dataset. The configuration includes both categorical and numeric range parameters, reflecting flexibility in batch size, accumulation steps, and learning rate, while maintaining structural flags such as gradient checkpointing.

The proposed hyperparameter search space is designed to achieve an optimal balance between computational efficiency and model performance. By constraining the search to a compact yet expressive set of configurations, we enable systematic exploration while mitigating the risk of combinatorial explosion.

\textbf{Model Selection.} Fixing the backbone to \texttt{meta-llama/Llama-3.2-3B-Instruct} ensures architectural consistency and leverages a strong instruction-tuned foundation, thereby reducing variance attributable to model choice.

\textbf{Categorical Parameters.} The discrete sets for \texttt{per\_device\_train\_batch\_size} (\{2, 4, 8, 12\}) and \texttt{gradient\_accumulation\_steps} (\{4, 8, 16\}) provide flexibility in effective batch size scaling. These choices are informed by practical hardware constraints, ensuring that each candidate configuration is feasible for typical GPU memory budgets.

\textbf{Continuous Parameters.} The learning rate interval $[10^{-5}, 5\times10^{-4}]$ captures the regime most relevant for fine-tuning large language models: sufficiently small to prevent catastrophic forgetting yet broad enough to accommodate task-specific convergence behaviour. Similarly, the epoch range of $[3, 10]$ offers adaptability across datasets of varying complexity while avoiding overfitting.

\textbf{Efficiency and Robustness.} By prioritising high-impact parameters and fixing structural flags such as \texttt{gradient\_checkpointing}, the search space remains tractable without sacrificing diversity. This design accelerates convergence, reduces computational overhead, and supports reproducibility across heterogeneous tasks.

In summary, the recommended parameter ranges exemplify an efficient search space that balances exploration and exploitation, enabling robust performance gains under realistic resource constraints.

\subsection{Example Best Practice Retrievals}
To ensure robust and efficient fine-tuning of large language models for medical question answering (MedQA), we consolidated a set of authoritative resources from Hugging Face and related repositories. These resources provide guidance on model configuration, parameter-efficient fine-tuning (PEFT), and domain-specific adaptation strategies.

\begin{tcolorbox}[colframe=black,colback=white,boxrule=0.5pt,width=\textwidth,title=Example of Retrieved Best-Practice Resources]
\centering
\begin{itemize}
    \item \textbf{Meta Llama 3 HF Repo:} \url{https://huggingface.co/meta-llama/Llama-3.2-3B}
    \item \textbf{Hugging Face PEFT for LoRA Adapters:} \url{https://github.com/huggingface/peft}
    \item \textbf{PEFT Finetuning Guide:} \url{https://huggingface.co/docs/peft/index}
    \item \textbf{Hugging Face Finetuning Guide:} \url{https://huggingface.co/docs/transformers/main/en/main_classes/trainer}
    \item \textbf{Hugging Face Llama 3 Finetuning Guide:} \url{https://huggingface.co/docs/transformers/model_doc/llama3}
    \item \textbf{QLoRA: Efficient Finetuning:} \url{https://github.com/artidoro/qlora}
    \item \textbf{Medical QA Datasets:} \url{https://huggingface.co/datasets?search=medical+qa}
\end{itemize}
\end{tcolorbox}

\textbf{Relevance for MedQA Task.} These resources collectively address the critical aspects of domain adaptation for medical question answering. The Llama 3 configuration and Hugging Face finetuning guides provide foundational instructions for model initialization and training workflows. PEFT and LoRA-based methods enable parameter-efficient adaptation, reducing memory footprint while preserving performance.
QLoRA introduces quantization-aware optimization, further improving efficiency without compromising accuracy. Domain-specific insights from the Medical LLMs overview and curated medical QA datasets ensure that the fine-tuned models capture clinical reasoning patterns and terminology. Note we ensure that the target dataset link is not present in the retrieved resources to prevent data leakage. Together, these best practices form a comprehensive toolkit for achieving high-quality, resource-efficient adaptation in downstream applications.

\begin{figure}[t!]
\centering
\begin{tcolorbox}[colframe=black,colback=white,boxrule=0.5pt,width=\textwidth]
\begin{verbatim}
{
    "model_name_or_path": 
    "meta-llama/Llama-3.2-3B-Instruct",
    "per_device_train_batch_size": {
        "type": "categorical",
        "choices": [2, 4, 8, 12]
    },
    "gradient_accumulation_steps": {
        "type": "categorical",
        "choices": [4, 8, 16]
    },
    "gradient_checkpointing": true,
    "learning_rate": {
        "type": "float",
        "low": 1e-05,
        "high": 0.0005
    },
    "num_train_epochs": {
        "type": "int",
        "low": 3,
        "high": 10
    }
}
\end{verbatim}
\end{tcolorbox}
\caption{Example hyperparameter range recommended by the multi-agent pipeline. Numeric parameters are expressed as intervals, while categorical parameters include discrete choices.}
\label{fig:param-json}
\end{figure}

\section{Experimental Setup}\label{apndx_expts}

In this section, we highlight the detailed experimental setup and justify its completeness in proving the efficiency of \algname{} over existing LLM-based AutoML based methods. To make the comparison fair across all baselines, all agents are implemented using the same model as AutoAdapt (i.e OpenAI's GPT-4.1-mini). In order to prevent data leakage, we remove entries of the exact benchmark dataset being evaluated, if present in KB. More concretely, experiments are conducted within the following two setups:
\begin{itemize}
    \item \textbf{Template Free (TF) Setting}: Within this setup, we evaluate each method's proficiency at generating an end-to-end pipeline under no explicit constraints, except for a simple prompt describing user requirements with respect to domain, accuracy and inference speed - mimcing the requirements of a user with no domain knowledge of LLMs. This gives the candidates full autonomy over the choice of model, adaptation technique and parameters. The prompt is kept consistent across all baselines. However, despite a 100\% success rate demonstrated by \algname{}, other methods fail to produce final code templates that execute without any errors. To quantify the ability of the baselines to produce executable code, we manually fix errors in the code templates produced by baseline methods, and report the Success Rate of each method as:
    $\text{Success Rate (SR)} = \frac{1}{1 + \text{number of iterations}}$. The prompt format is as follows:
\begin{promptBox}{Input Prompt for Template Free Setting} 
\begin{lstlisting}
We require a model which has been trained to answer {task_domain} questions. The model should maintain an
accuracy of over 90%. Please ensure that the total number of model parameters does not exceed 10 billion.
Futhuremore, the model should be deployed on a GPU and process at least 200 words per minute.
\end{lstlisting}
\end{promptBox}

    \item \textbf{Template Aware (TA) Setting}: This setup is meant to offer a fairer comparison between the baselines. In this setup, we augment the other baselines with the following AutoAdapt resources - (a) Adaptation Configuration Graph (ACG), (b) best practices knowledge base (KB), and (c) respective templates for SFT, DPO and RAG. This ensures a default success rate of 100\% and judges the baselines purely on their ability to make an informed decision regarding the parameter choices. Since performance on Natural Language benchmarks is highly contingent on the base model and duration of training, we fix the model across a task along with the number of epochs set to 1. The prompt format is as follows:
\begin{promptBox}{Input Prompt for Template Aware Setting} 
\begin{lstlisting}
You are an expert system tasked with choosing the most suited adaptation technique and respective 
hyperparameters for a given question answering task to be solved by adapting a Large Language Model. 
The base model and number of training epochs is fixed. Your goal is to select the best configuration and 
technique for training a {task_domain} question answering model by refering to the following:

###TASK: 
Train a model specialized for answering {user_domain} questions. Accuracy must be greater than 90%. 
Deployment must be on GPU. Inference speed must be greater than 200 words per minute.

###DECISION GRAPH: 
{ACG}

###CONTEXT:
{Best Practices Knowledge Base}

###PARAMETERS:
base_model = {fixed_base_model}
learning_rate = "<FILL_ME>"      
num_train_epochs = 1
lora = "<FILL_ME>"                
lora_r = "<FILL_ME>"              
lora_alpha = "<FILL_ME>"          
lora_dropout = "<FILL_ME>"        
per_device_training_batch_size = "<FILL_ME>"  
gradient_accumulation_steps = "<FILL_ME>"     
quantization = "<FILL_ME>" 

###SFT TEMPLATE:
{sft_template}

###DPO TEMPLATE:
{dpo_template}

###RAG TEMPLATE:
{rag_template}
\end{lstlisting}
\end{promptBox}
\end{itemize}

We utilize the proposed method for each dataset to evaluate the trained models. However, for datasets that lack a well-defined method of evaluation, we adopt an LLM-based equivalence matching with the correct labels to evaluate the models. The evaluation method and metric for each dataset is summarized in the following table, followed by the prompt formats:

\begin{table}[ht]
\centering
\caption{Evaluation setup}
\label{tab:evaluation_no_prompt}
\begin{tabular}{l l c}
\toprule
\textbf{Dataset} & \textbf{Method of Evaluation} & \textbf{Metric} \\
\midrule
UC Berkeley MATH & \href{https://github.com/openai/simple-evals}{OpenAI/simple\_evals} & Accuracy \\
MedQA            & LLM-based equivalence matching (P1) & Accuracy \\
CaseHold         & LLM-based equivalence matching (P1) & Accuracy \\
PEM               & LLM-based answer extraction (P3) + equivalence matching (P1) & Accuracy \\
RCA              & LLM-as-a-Judge (P2)  & Average Score  \\
ARC              & LLM-based equivalence matching (P1) & Accuracy \\
MBPP+            & Evalplus \cite{mbppplus} & Accuracy \\
When2Call        & LLM-based equivalence matching (P1) & Accuracy \\
E-commerce       & Direct Comparison & Accuracy \\
Textual Entailment        & Direct Comparison & Accuracy \\
\bottomrule
\end{tabular}
\end{table}

\begin{promptBox}{Prompt Format for LLM-based equivalence matching (P1)} 
\begin{lstlisting}
System Prompt
----------------------------------------
Given a reasoning domain problem, the correct answer, and the student's answer, determine if the student's answer 
is correct. You do not need to solve the problem; simply compare the student's answer with the correct 
answer and return 'yes' if they match and 'no' if they do not. The correct_answer and student_answer need 
not exactly match, even if they mean the same thing then return 'yes'. Also if either the option or the 
answer match between the two, then return 'yes'.


User Prompt
----------------------------------------

Problem: {problem}

Correct Answer: {correct_answer}

Student Answer: {student_answer}
\end{lstlisting}
\end{promptBox}

\begin{promptBox}{Prompt Format for LLM-as-a-Judge evaluation for RCA (P2)} 
\begin{lstlisting}
System Prompt
----------------------------------------
## Task Description
- You are an LLM response evaluator and scorer that assesses how well a generated string 
  matches a reference string.
- The reference strings you will be provided with are from an on-call software engineer
  responsible for maintaining a cloud data center.
- You will also be provided with a context string that the engineer used to write the 
  response string to help with scoring.
- You should assign a score from 1 to 5 based on how comprehensively the generated 
  string covers the information in the reference string.
- Use the following scoring criteria:
  5: The generated string perfectly captures all the key information in the reference 
     string, including details and nuances.
  4: The generated string covers most of the information in the reference string 
     accurately but may miss some minor details.
  3: The generated string partially covers the information in the reference string 
     but lacks depth or accuracy in some areas.
  2: The generated string contains only a few relevant points from the reference 
     string and misses significant details.
  1: The generated string is largely unrelated to the reference string or contains 
     mostly incorrect information.
You should also provide a brief explanation or justification for your score.


Example User Prompt
----------------------------------------
Reference String:
The history of artificial intelligence dates back to ancient civilizations such 
as Greece and China, with notable advancements in the 20th century leading to 
modern AI technologies.

Generated String:
Artificial intelligence has evolved significantly over time, especially in the 
20th century, thanks to advancements in technology.


Example System Scoring
----------------------------------------
Assigned Score: 4
Justification: The generated string accurately covers the evolution of AI and 
mentions key advancements in the 20th century, but it lacks specific references 
to ancient civilizations like Greece and China.


Final User Prompt
----------------------------------------
Context:
{context}

Reference String:
{gt_answer}

Generated String:
{generated_text}
\end{lstlisting}
\end{promptBox}

\begin{promptBox}{Prompt Format for LLM-based answer extraction for PEM dataset (P3)}
\begin{lstlisting}[mathescape]
System Prompt
----------------------------------------
Extract the final answer from the given solution


User Prompt
----------------------------------------
You are an answer extraction system. Your task is to process a dataset of science and mathematics problems and extract final answers. Follow these precise instructions:

1. Iterate through each entry in the dataset.

2. For each entry:
   a. Search for a final numerical value, equation, or multiple choice option.
   b. If a "Text Solution" section exists, prioritize finding the answer there.
   c. Ignore any URLs or video links.
   d. If sentences such as "statement i is correct/incorrect" are present then count this as an answer.
      Return: statement i is correct/incorrect.
   e. If multiple answers are present then return all answers.
   f. For math answers, return final expression if answer is not explicilty mentioned.

3. Output format:
   - If a final answer is found:
     Return "Answer: answer"
   - If no clear final answer is present:
     Return "No Answer"

4. Formatting rules:
   - Preserve exact notation and units of extracted answers.
   - Do not include explanations or workings.

5. Edge cases:
   - For incomplete or cut-off entries, still attempt to extract an answer if possible.
   - If an entry only contains hints without a clear final answer, return "No Answer".

6. Example usage:

   - Example 1:
     text = Text Solution:
            Let the velocities at the ends A and B be v_1 and v_2
            Using, equation of continuity 6v_1 = 10v_2
            $\Rightarrow$ v_1 = 5/3 v_2
            Now, using Bernoullis equation:
            P_1 + 1/2 $\rho$ v_1^2 = P_2 + 1/2 $\rho$ v_2^2
            P_1 - P_2 = 1/2 $\rho$ (16/9 v_2^2)
            0.05 $\rho$ g = 1/2 $\rho$ (16/9 v_2^2)
            Hence, v_2 = 0.75 m/s = 75 cm/s
            Hence, the total flow rate = A v = $75 \times 10^{-1} cm^3/s$ = 7.5 cc/s
            
     Answer = 7.5 cc/s

   - Example 2:
     text = Hint:
            - Determine the value of one main scale division (MSD) by taking the reciprocal of the number of divisions per cm on the main scale.
            - Calculate the value of one vernier scale division (VSD) by dividing the value of (n-1) main scale divisions by n, since n VSD coincide with (n-1) MSD.
            - The least count (LC) of the vernier calipers can be found by subtracting the value of one MSD from one VSD.
            Video Solution:
            https://www.youtube.com/embed/NTrY16gq84w

     Answer = No Answer


Execute this process systematically for each entry, focusing solely on answer extraction.

text = {solution}
\end{lstlisting}
\end{promptBox}

\section{Detailed Results}
\label{apndx_res}

In this section, we present in detail the specific technique employed by baselines and \algname{}, and the corresponding hyperparameters choice under the 2 experimental setups. 

\subsection{Template Aware Results}
\begin{table}[ht]
{\centering
\caption{Performance comparison across datasets and methods in the template aware (TA) setting.}
\label{tab:apndx_template_aware_results}
\resizebox{\linewidth}{!}{
\begin{tabular}{l l c c c c c c}
\toprule
\textbf{Dataset} & \textbf{Base Model} & \textbf{AutoAdapt} & \textbf{AutoMLAgent} & \textbf{MLCopilot} & \textbf{HF Defaults} & \textbf{DS-Agent} \\
\midrule
UC Berkeley MATH & Phi-4-mini-Instruct          & $27.8 \pm 3.5$ & $26.8 \pm 2.4$ & $27.2 \pm 2.3$ & $26.2 \pm 0.06$ & $\mathbf{30.0 \pm 3.7}$ \\
MedQA            & Phi-4-mini-Instruct          & $\mathbf{61.27 \pm 1.1}$ & $58.84 \pm 1.3$ & $57.50 \pm 1.2$ & $58.13 \pm 0.04$ & $59.86 \pm 1.0$ \\
CaseHold         & Llama3.2-3B                  & $\mathbf{89.22 \pm 0.9}$ & $84.72 \pm 1.1$ & $82.63 \pm 1.0$ & $86.21 \pm 0.08$ & $87.28 \pm 0.9$ \\
PEM               & Qwen2.5-Math-7B-Instruct     & $\mathbf{67.69 \pm 1.3}$ & $64.53 \pm 1.4$ & $63.71 \pm 1.5$ & $65.78 \pm 0.02$ & $63.40 \pm 1.6$ \\
RCA              & Phi-4                        & $\mathbf{2.30 \pm 0.76}$ & $2.23 \pm 0.81$ & $2.18 \pm 0.77$ & $2.23 \pm 0.81$ & $2.25 \pm 0.66$ \\
ARC              & Llama3.1-8B-Instruct         & $\mathbf{83.79 \pm 1.2}$ & $77.39 \pm 1.1$ & $77.56 \pm 1.0$ & $78.24 \pm 0.03$ & $77.47 \pm 1.4$ \\
MBPP+            & Qwen2.5-7B-Instruct          & $\mathbf{68.3 \pm 0.8}$ & $\mathbf{68.2 \pm 0.7}$ & $67.5 \pm 0.9$ & $\mathbf{68.0 \pm 0.07}$ & $66.9 \pm 0.8$ \\
When2Call        & Phi-4-mini-Instruct          & $\mathbf{54.5 \pm 0.9}$ & $\mathbf{45.4 \pm 1.2}$ & $48.1 \pm 1.1$ & $\mathbf{46.9 \pm 0.05}$ & $50.0 \pm 0.7$ \\
\bottomrule
\end{tabular}
}
}
\end{table}

\subsection{Template Free Results}
The raw results from the template free (TF) setup are presented below:
\begin{table}[ht]
\centering
\caption{Performance comparison across datasets and methods in the template free (TF) setting}
\label{tab:converted_results_rounded_bold}
\begin{tabular}{l c c c c}
\toprule
\textbf{Dataset} & \textbf{AutoAdapt} & \textbf{AutoMLAgent} & \textbf{MLCopilot} & \textbf{DS-Agent} \\
\midrule
UC Berkeley MATH & $\textbf{17.00} \pm 3.1$ & $8.60 \pm 2.7$ & $0.60 \pm 0.1$ & $0.00 \pm 0.00$ \\
MedQA            & $63.55 \pm 2.1$ & $46.97 \pm 1.8$ & $38.10 \pm 1.2$ & $\textbf{74.63} \pm 2.1$\\
CaseHold         & $\textbf{92.53} \pm 1.4$ & $85.90 \pm 1.2$ & $3.10 \pm 0.8$ & $3.16 \pm 0.8$ \\
PEM               & $\textbf{65.14} \pm 2.3$ & $28.03 \pm 1.9$ & $0.00 \pm 0.00$ & $0.00 \pm 0.00$ \\
RCA              & $\textbf{2.30} \pm 0.76$ & $1.44 \pm 0.64$ & $1.13 \pm 0.53$ & $1.23 \pm 0.72$ \\
ARC              & $\textbf{90.87} \pm 2.2$ & $76.71 \pm 1.5$ & $28.84 \pm 1.2$ & $14.85 \pm 0.9$ \\
MBPP+            & $\textbf{68.00} \pm 1.5$ & $39.90 \pm 1.2$ & $0.00 \pm 0.00$  & $0.00 \pm 0.00$\\
Entail           & $\textbf{82.1} \pm 1.4$ & $80.5 \pm 1.7$ & $24.3 \pm 2.2$ & $50.0 \pm 1.9$ \\
Ecom             & $94.2 \pm 1.7$ & $97.4 \pm 1.1$ & $\textbf{99.1} \pm 0.3$ & $0.00 \pm 0.00$ \\
When2Call        & $\textbf{58.52} \pm 1.3$ & $50.66 \pm 1.2$ & $52.30 \pm 0.9$ & $0.00 \pm 0.00$ \\
\bottomrule
\end{tabular}
\end{table}

We further normalize the accuracies and average scores from 0 to 1, and report it as NPS (Normalized Performance Scores). We report the Cumulative Score defined as: $\text{Cumulative Score} = \tfrac{1}{2} ( \text{NPS} + \text{SR} )$.
The detailed results used to construct the radar plots in \autoref{fig:main_results} are as follows:

\begin{table}[ht]
\centering
\caption{SR and NPS across datasets and baselines}
\label{tab:sr_nps_results}
\resizebox{\linewidth}{!}{
\begin{tabular}{l c c c c c c c c}
\toprule
\textbf{Dataset} &
\multicolumn{4}{c}{\textbf{SR} ($1/(1+\text{iterations})$)} &
\multicolumn{4}{c}{\textbf{NPS}} \\
\cmidrule(lr){2-5}\cmidrule(lr){6-9}
& \textbf{AutoAdapt} & \textbf{AutoMLAgent} & \textbf{MLCopilot} & \textbf{DS-Agent}
& \textbf{AutoAdapt} & \textbf{AutoMLAgent} & \textbf{MLCopilot} & \textbf{DS-Agent} \\
\midrule
UC Berkeley MATH & 1.00 & 0.11 & 0.17 & 1.00 & 0.17 & 0.09 & 0.01 & 0.00 \\
MedQA            & 1.00 & 0.11 & 0.14 & 1.00 & 0.64 & 0.47 & 0.38 & 0.75 \\
CaseHold         & 1.00 & 0.11 & 0.20 & 0.25 & 0.93 & 0.86 & 0.03 & 0.03 \\
PEM               & 1.00 & 0.11 & 0.17 & 0.25 & 0.65 & 0.28 & 0.00 & 0.00 \\
RCA              & 1.00 & 0.11 & 0.17 & 1.00 & 0.46 & 0.29 & 0.23 & 0.25 \\
ARC              & 1.00 & 0.11 & 0.17 & 0.33 & 0.91 & 0.77 & 0.29 & 0.15 \\
MBPP+            & 1.00 & 0.11 & 0.17 & 0.25 & 0.68 & 0.40 & 0.00 & 0.50 \\
Entail           & 1.00 & 1.00 & 1.00 & 0.70 & 0.82 & 0.80 & 0.24 & 0.50 \\
Ecom             & 1.00 & 1.00 & 0.50 & 0.00 & 0.94 & 0.97 & 0.98 & 0.00 \\
When2Call        & 1.00 & 0.11 & 0.17 & 0.53 & 0.59 & 0.51 & 0.52 & 0.53 \\
\bottomrule
\end{tabular}
}
\end{table}

\begin{table}[ht]
\centering
\caption{Cumulative Score (CS) across datasets and baselines}
\label{tab:cs_results}
\begin{tabular}{l c c c c}
\toprule
\textbf{Dataset} & \textbf{AutoAdapt} & \textbf{AutoMLAgent} & \textbf{MLCopilot} & \textbf{DS-Agent} \\
\midrule
UC Berkeley MATH & \textbf{0.58} & 0.10 & 0.09 & 0.50 \\
MedQA            & 0.82 & 0.29 & 0.26 & \textbf{0.87} \\
CaseHold         & \textbf{0.96} & 0.49 & 0.12 & 0.14 \\
PEM               & \textbf{0.83} & 0.20 & 0.08 & 0.12 \\
RCA              & \textbf{0.73} & 0.20 & 0.20 & 0.62 \\
ARC              & \textbf{0.95} & 0.44 & 0.23 & 0.24 \\
MBPP+            & \textbf{0.84} & 0.26 & 0.08 & 0.38 \\
Entail           & \textbf{0.91} & 0.90 & 0.62 & 0.60 \\
Ecom             & 0.97 & \textbf{0.99} & 0.74 & 0.00 \\
When2Call        & \textbf{0.79} & 0.31 & 0.34 & 0.53 \\
\bottomrule
\end{tabular}
\end{table}
\section{Case Study: Multi-Agent Debate Planning}
\label{apndx_casestudy}

This section presents in detail how the multi-agent debate planning proceeds for the 'MedQA' dataset. It highlights the individual contribution of the proposal and critic agents, and the mechanism behind the final choice of the adaptation technique and the hyperparameter ranges.

\begin{outputbox}{Agent: BestPracticeAgent --- Node Selection: \texttt{adapt}}
\begin{lstlisting}
Based on the user task and the contextual best practices, the goal is to select exactly one candidate node from ['sft', 'rag', 'dpo'] as the next phase in domain adaptation for the given user dataset/task.

**Reasoning:**

1. **User Task and Candidates:**
   - The user requires selecting one node (strategy) among:
     - **sft (Supervised Fine-Tuning)**
     - **rag (Retrieval-Augmented Generation)**
     - **dpo (Direct Preference Optimization)**

2. **Context and Best Practices:**
   - The best practices indicate multiple use cases with models trained primarily using **sft** across a variety of medical datasets and multimodal tasks (notably sequence classification, image description, symptom analysis).
   - Several base models like `gpt-4-vision-preview`, `microsoft/deberta-v3-small`, and `unsloth/Qwen2-VL-7B-Instruct-bnb-4bit` employ **supervised fine-tuning**.
   - Parameter-efficient fine-tuning techniques like LoRA and quantization approaches (4-bit, nf4, int4) are frequently paired with SFT, indicating it is well-established and compatible with efficient training pipelines.
   - Federated learning and instruction tuning appear as promising advanced training techniques but are not reflected in candidate choices.
   - There is no strong evidence or examples supporting the application or effectiveness of RAG or DPO in the provided best practice data for this particular medical/multimodal domain adaptation scenario.
   - DPO is mentioned but without concrete training or evaluation results, unlike SFT, which is consistently used and documented.
   - RAG typically demands additional infrastructure for retrieval components and was not part of the documented successful model adaptations or dataset discussions.

3. **Summary:**
   - Considering that supervised fine-tuning (SFT) is the proven and predominant choice in the referenced best practices.
   - It matches the user datasets/tasks aligned with medical and multimodal domains.
   - It supports effective parameter-efficient fine-tuning approaches enabling domain adaptation under constrained compute resources.

**Concrete Recommendation:**
- **Select: "sft" (Supervised Fine-Tuning)**

--- 
**Final Answer:**  
Select **"sft"** as the adaptation strategy node for the user's domain adaptation phase.

**References:**
https://github.com/deepmind/mathematics_dataset
https://github.com/Starscream-11813/Variational-Mathematical-Reasoning/
https://github.com/arkilpatel/SVAMP
https://github.com/TNI-playground/TNI-FlowerTune-LLM-Leaderboard-GeneralNLP
https://github.com/logikon-ai/cot-eval
https://github.com/dustalov/llmfao
\end{lstlisting}
\end{outputbox}

\begin{outputbox}{Agent: ResearchAgent  --- Node Selection: \texttt{sft}}
\begin{lstlisting}
Among the candidates 'sft' (Supervised Fine-Tuning), 'rag' (Retrieval-Augmented Generation), and 'dpo' (Direct Preference Optimization), 'sft' is the most standard and straightforward adaptation strategy for domain adaptation. It involves fine-tuning the model parameters using supervised data from the target domain, which typically yields direct improvements in performance on that domain. 'rag' requires additional retrieval mechanisms and data infrastructure, which may complicate the adaptation pipeline and is more suited when external knowledge retrieval is essential. 'dpo', focusing on preference optimization, is more suited for aligning model outputs with human preferences but might not directly improve domain-specific performance as reliably as 'sft'. Thus, selecting 'sft' ensures an effective and simple adaptation strategy aligned with standard practices.

next_nodes=['sft']
\end{lstlisting}
\end{outputbox}

\begin{outputbox}{Agent: UserPreferenceAgent --- Review of \texttt{sft} Proposals}
\begin{lstlisting}
# Proposal Review for User Preference Alignment

## Context
The user preference specifies selecting exactly one candidate node from among 'sft', 'rag', and 'dpo' as the next phase in domain adaptation for the user dataset/task. The user expects adherence to best practices and justifications aligned with domain adaptation strategies.

---

## Agent: BestPracticeAgent
### Analysis:
- The proposal explicitly addresses the user's task of choosing one candidate node from the specified set.
- The reasoning heavily references best practices documented for medical/multimodal domain adaptation.
- Emphasizes sft (Supervised Fine-Tuning) as the predominant, stable, well-documented, and efficient approach.
- Notes lack of strong evidence or successful precedents for rag or dpo in the given context.
- Considers infrastructure complexity, compatibility with parameter-efficient fine-tuning, and domain applicability.
- Provides concrete recommendation: select "sft".
- References numerous relevant sources and practical examples supporting the choice.

### Alignment with User Preferences:
- Fully aligned; directly tackles the user-specified candidate nodes and recommends one node.
- Justification rationale is detailed and supported by best practice examples.
- No observable violation of user preference.

### Conclusion:
- No violations.
- Proposal well-aligned and comprehensive.

---

## Agent: ResearchAgent
### Analysis:
- Provides a clear justification for choosing 'sft' over 'rag' and 'dpo'.
- Highlights sft as the standard and straightforward approach with direct domain performance benefits.
- Notes drawbacks and increased complexity for rag due to retrieval mechanisms.
- Notes dpo is better suited for preference alignment than domain adaptation.
- Recommends sft as the effective and simple solution according to standard practice.

### Alignment with User Preferences:
- Addresses exactly the user preference scenario.
- No contradiction or omission regarding user instruction.
- Justification is concise but sufficiently informative.
- No evidence of deviation from user requirements.

### Conclusion:
- No violations.
- Proposal is concise and appropriately aligned.

---

# Summary and Recommendation
| Agent             | Violations? | Notes |
|------------------|-------------|-------|
| BestPracticeAgent| None        | Detailed, well-supported, aligns perfectly with user preferences. |
| ResearchAgent    | None        | Concise and aligned, covers necessary points. |

Recommendation: Both proposals are aligned and recommend selecting "sft". Given higher detail in BestPracticeAgent, recommend accepting BestPracticeAgent's proposal.
\end{lstlisting}
\end{outputbox}

\begin{outputbox}{Agent: DataAgent --- Review of \texttt{sft} Proposals}
\begin{lstlisting}
# Review of AI Training Adaptation Proposals

## 1. Per-Agent Critique

### BestPracticeAgent
- Strengths:
  - Correctly identifies the three candidate strategies: SFT, RAG, and DPO.
  - Detailed reasoning anchored in best practices for medical domain and multimodal tasks.
  - Notes SFT is most frequently used and stable approach in similar datasets and tasks.
  - Recognizes RAG demands additional infrastructure not evidently supported given the dataset/task.
  - Points out lack of solid evidence for DPO in this medical adaptation scenario.
  - Concrete recommendation aligns with user's data profile and task: SFT well-suited for single-choice, structured I/O and moderate-length samples.
- No violations observed.

### ResearchAgent
- Strengths:
  - Concise and accurate rationale favoring SFT.
  - Correctly emphasizes straightforwardness and direct domain adaptation advantages of SFT.
  - Correctly identifies RAG introduces complexity unsuited to this dataset/task.
  - Points out DPO is more about user preference alignment than domain performance.
- Limitations:
  - Shorter and less detailed compared to BestPracticeAgent; lacks dataset specifics/medical context.
- No violations observed.

## 2. Final Recommendation
Both proposals identify "sft" as best strategy node. BestPracticeAgent is preferred due to thoroughness.

# Final Answer
Select "sft" (Supervised Fine-Tuning) as the domain adaptation strategy node.
Preferred Proposal: BestPracticeAgent
\end{lstlisting}
\end{outputbox}

\begin{outputbox}{Agent: BestPracticeAgent --- Final Recommendation: \texttt{sft}}
\begin{lstlisting}
Based on the user task to select exactly one adaptation strategy node from ['sft', 'rag', 'dpo'] for the given dataset/task, and considering the comprehensive contextual best practices and prior agent evaluations, the recommended choice is:

### Final Recommendation:
Select: "sft" (Supervised Fine-Tuning)

### Reasoning:
- Proven Effectiveness: SFT is predominant and successful in medical/multimodal datasets.
- Compatibility with Efficient Techniques: Works well with LoRA and quantization (4-bit, int4, nf4).
- Simplicity and Stability: Avoids retrieval infrastructure required by RAG.
- Limited Support for Alternatives: DPO targets preference alignment; RAG adds complexity; both lack demonstrated best-practice support here.
- Alignment with User Instructions: Select exactly one node; SFT is validated and strategically sound.

Please proceed with Supervised Fine-Tuning (sft) as the next phase of domain adaptation.

References:
https://github.com/deepmind/mathematics_dataset
https://github.com/Starscream-11813/Variational-Mathematical-Reasoning/
https://github.com/arkilpatel/SVAMP
https://github.com/TNI-playground/TNI-FlowerTune-LLM-Leaderboard-GeneralNLP
https://github.com/logikon-ai/cot-eval
https://github.com/dustalov/llmfao
\end{lstlisting}
\end{outputbox}

\begin{outputbox}{Agent: ResearchAgent --- Reiteration: \texttt{sft}}
\begin{lstlisting}
Among the candidate adaptation strategies 'sft' (Supervised Fine-Tuning), 'rag' (Retrieval-Augmented Generation), and 'dpo' (Direct Preference Optimization), supervised fine-tuning (sft) stands out as the most reliable and straightforward method for domain adaptation. It directly adjusts model parameters based on labeled data from the target domain, ensuring improved domain-specific performance. RAG requires additional components and infrastructure for knowledge retrieval, which may complicate adaptation and is less documented for the current domain scenario. DPO focuses on aligning model outputs with user preferences rather than explicit domain adaptation, making it less suitable here. Given the lack of contrary retrieved knowledge, and standard best practices, selecting 'sft' optimally balances effectiveness, simplicity, and alignment with domain adaptation goals.

next_nodes=['sft']
\end{lstlisting}
\end{outputbox}

\begin{outputbox}{Agent: UserPreferenceAgent --- Review of \texttt{sft} Proposals (Round 2)}
\begin{lstlisting}
# Proposal Review with Respect to User Preferences

## Summary of User Preferences
User requires selection of exactly one node from ['sft','rag','dpo'].

## Agent Proposals Review
### BestPracticeAgent
- Selects exactly one node: "sft".
- Strong domain/task justification; practical and stable training benefits.
- Violations: None.

### ResearchAgent
- Selects exactly one node: "sft".
- Clear domain adaptation justification; concise.
- Violations: None.

## Recommendation
Both fulfill requirement; BestPracticeAgent is more detailed and thus recommended.
\end{lstlisting}
\end{outputbox}

\begin{outputbox}{Agent: DataAgent (Critique) --- Review of \texttt{sft} Proposals (Round 2)}
\begin{lstlisting}
# Proposal Review for Adaptation Strategy Selection

## 1. Per-Agent Critique

### BestPracticeAgent
- Appropriate for medical QA dataset: well-formed examples, short outputs (~9 tokens), variable inputs.
- Avoids unnecessary RAG infrastructure; dismisses DPO for lack of advantage in this domain.
- Recommends LoRA/quantization; aligns with format and token-length distribution.
- Summary: sound, conservative, domain-aligned.

### ResearchAgent
- Also selects "sft"; agrees on straightforward supervised approach.
- Less detailed but consistent; no violations.

## 2. Final Recommendation
Both select "sft"; BestPracticeAgent is stronger due to detail and practical considerations.
Proceed with SFT (sft).
\end{lstlisting}
\end{outputbox}

\begin{outputbox}{Agent: BestPracticeAgent  --- Node Selection: \texttt{peft}}
\begin{lstlisting}
Given the user task to decide on selecting candidates for the next phase of domain adaptation with a focus on fine-tuned model-based adaptation, and considering that the only candidate available is "peft" (Parameter-Efficient Fine-Tuning), the recommendation is:

Recommendation: Select the candidate "peft" for the next phase.

Reasoning:
1. Alignment with best practices: strong endorsement of LoRA/FlexLoRA for efficient adaptation.
2. Model examples/configurations incorporate PEFT parameters (lora_r, lora_alpha, lora_dropout).
3. Resource efficiency for constrained hardware/VRAM.
4. Consistency with user instruction: only one candidate and no conflicting info.

Summary: Selecting "peft" leverages state-of-the-art resource-efficient fine-tuning techniques for domain adaptation.

References:
https://github.com/Starscream-11813/Variational-Mathematical-Reasoning/
https://github.com/deepmind/mathematics_dataset
https://github.com/arkilpatel/SVAMP
https://github.com/TNI-playground/TNI-FlowerTune-LLM-Leaderboard-GeneralNLP
https://github.com/rungalileo/agent-leaderboard
https://github.com/logikon-ai/cot-eval
\end{lstlisting}
\end{outputbox}

\begin{outputbox}{Agent: ResearchAgent  --- Node Selection: \texttt{peft}}
\begin{lstlisting}
Given that there is only one candidate node 'peft' and the task involves configuration settings for a fine-tuned model-based adaptation approach, selecting 'peft' (Parameter-Efficient Fine-Tuning) is appropriate. PEFT methods are effective for domain adaptation as they allow fine-tuning with fewer parameters and resources while retaining model performance. Since no other candidates are available, selecting 'peft' aligns well with the instruction and the goal of the adaptation process.

next_nodes=['peft']
\end{lstlisting}
\end{outputbox}

\begin{outputbox}{Agent: UserPreferenceAgent (Critique) --- Review of \texttt{peft} Proposals}
\begin{lstlisting}
# Proposal Review for Training Candidate Selection

## Agent Proposal Critiques

### BestPracticeAgent
- Recommends selecting "peft" citing best practices, resource efficiency, practical viability, and no conflicting info.
- Compliance: aligned with user context and task.
- Issues: none identified.

### ResearchAgent
- Selects "peft" as sole candidate; emphasizes efficiency.
- Compliance: aligned; less detailed but no violation.

## Recommendation
Most aligned: BestPracticeAgent (more comprehensive rationale and references).
\end{lstlisting}
\end{outputbox}

\begin{outputbox}{Agent: DataAgent (Critique) --- Review of \texttt{peft} Proposals}
\begin{lstlisting}
## Review of Proposals Based on Provided User Data and Task Assumptions

### User Data Summary
- Medical multiple-choice QA; 10,178 samples; inputs avg ~235 tokens, max 938; outputs avg ~8.7 tokens.

## 1. Per-Agent Critique

### BestPracticeAgent
Strengths:
- Identifies "peft" as sole candidate; emphasizes LoRA/FlexLoRA and efficiency.
Potential issues:
- Initially generic; should address short outputs and variable input lengths; confirm classification format and evaluation.

### ResearchAgent
Strengths:
- Correct default choice given sole candidate.
Potential issues:
- Too high-level; lacks task/data-specific considerations.

## 2. Final Recommendation
Select PEFT; prefer BestPracticeAgent with augmentation advice:
- Ensure setup supports classification outputs/MCQ.
- Account for variable input lengths; set appropriate loss/metrics; QC and validation.
\end{lstlisting}
\end{outputbox}

\begin{outputbox}{Agent: BestPracticeAgent  --- Final Recommendation: \texttt{peft}}
\begin{lstlisting}
Recommendation: Select "peft" (Parameter-Efficient Fine-Tuning) for the next phase of domain adaptation.

Reasoning:
1. Candidate exclusivity: only node is peft; fits fine-tuned adaptation.
2. Best-practice support: LoRA/FlexLoRA widely adopted; examples include deberta and Qwen2-VL variants.
3. Resource efficiency: suits variable-length inputs (up to ~938 tokens) and short outputs (~8.7 tokens) for classification/MCQ.
4. Dataset/task consistency: PEFT can adapt heads, handle token variability, and pair with proper loss/metrics.
5. No conflicting constraints.

Augmentation advice:
- Confirm PEFT supports short classification outputs and MCQ format.
- Normalize inputs; consider domain tokenization.
- Use classification-appropriate loss (cross-entropy) and metrics (accuracy/F1).
- Monitor domain shift/catastrophic forgetting.

References:
https://github.com/Starscream-11813/Variational-Mathematical-Reasoning/
https://github.com/deepmind/mathematics_dataset
https://github.com/arkilpatel/SVAMP
https://github.com/TNI-playground/TNI-FlowerTune-LLM-Leaderboard-GeneralNLP
https://github.com/rungalileo/agent-leaderboard
https://github.com/logikon-ai/cot-eval
\end{lstlisting}
\end{outputbox}

\begin{outputbox}{Agent: ResearchAgent  --- Reiteration: \texttt{peft}}
\begin{lstlisting}
The candidate 'peft' stands for Parameter-Efficient Fine-Tuning, which is well-aligned with the task of fine-tuned model-based adaptation, especially in scenarios requiring resource efficiency. Given that 'peft' is the sole candidate and no conflicting information advises against its use, selecting it is logical. PEFT methods enable effective adaptation of large models with fewer parameters, reducing computational resources while maintaining performance. The prior best practice analysis shows strong support for PEFT techniques such as LoRA and FlexLoRA. Therefore, selecting 'peft' optimally advances the adaptation process based on current best practices and the task requirements.

next_nodes=['peft']
\end{lstlisting}
\end{outputbox}

\begin{outputbox}{Agent: UserPreferenceAgent (Critique) --- Review of \texttt{peft} Proposals (Round 2)}
\begin{lstlisting}
# Review of AI Training Proposals with Respect to User Preferences

## Proposal Critiques

### BestPracticeAgent
- Aligns with implicit preference to select sole candidate peft.
- Comprehensive reasoning; accounts for dataset nuances and implementation.
- No violations.

### ResearchAgent
- Concise; selects peft; no conflicts.
- No violations.

## Summary Recommendation
Both conform; BestPracticeAgent recommended due to actionable detail and references.
\end{lstlisting}
\end{outputbox}

\begin{outputbox}{Agent: DataAgent (Critique) --- Review of \texttt{peft} Proposals (Round 2)}
\begin{lstlisting}
# Training Proposal Review

## 1. Per-agent Critique

### BestPracticeAgent
- Acknowledges dataset stats (variable input length, short outputs).
- Recommends loss/metrics and monitoring; aligns with data realities.
- No violations.

### ResearchAgent
- Correct but less detailed; still consistent with data characteristics.
- No violations.

## 2. Final Recommendation
Accept BestPracticeAgent proposal as most aligned and thorough choice for next phase using PEFT.
\end{lstlisting}
\end{outputbox}

\begin{outputbox}{Agent: BestPracticeAgent  --- Node: \texttt{sft\_train} (Parameters \& Reasoning)}
\begin{lstlisting}
{
  "parameters": {
    "model_name_or_path": "unsloth/Qwen2-VL-7B-Instruct-bnb-4bit",
    "learning_rate": {"type": "float", "low": 1e-5, "high": 3e-4},
    "num_train_epochs": {"type": "int", "low": 2, "high": 5},
    "per_device_train_batch_size": {"type": "categorical", "choices": [4, 8, 12]},
    "gradient_accumulation_steps": {"type": "categorical", "choices": [4, 8, 12]},
    "weight_decay": {"type": "float", "low": 0.0, "high": 0.05},
    "optimizer": {"type": "categorical", "choices": ["adamw_8bit", "adamw_torch_fused", "AdamW"]},
    "scheduler": {"type": "categorical", "choices": ["linear", "cosine", "constant"]},
    "mixed_precision": {"type": "categorical", "choices": ["fp16", "bf16"]},
    "quantization": {"type": "categorical", "choices": [1]},
    "quantization_method": {"type": "categorical", "choices": ["int4"]},
    "peft": {"type": "categorical", "choices": [1]},
    "lora_r": {"type": "categorical", "choices": [8, 16, 32]},
    "lora_alpha": {"type": "categorical", "choices": [16, 32, 64]},
    "lora_dropout": {"type": "categorical", "choices": [0, 0.05, 0.1]},
    "target_modules": {"type": "categorical", "choices": ["q_proj", "dense", "dense_h_to_4h", "out_proj", "o_proj"]},
    "gradient_checkpointing": true,
    "packing": true,
    "logging_steps": 25,
    "eval_steps": 100
  }
}

--- Reasoning (abridged):
- Model: unsloth/Qwen2-VL-7B-Instruct-bnb-4bit (<10B params), balances capacity and throughput.
- Strategy: SFT with PEFT (LoRA) + int4 quantization + mixed precision for efficiency.
- Search space: LR 1e-5..3e-4, epochs 2..5, batch/accum choices for memory/stability, WD 0..0.05.
- Optimizers/schedulers: AdamW variants; linear/cosine/constant.
- Memory: gradient checkpointing + packing; logging/eval steps set.

References:
https://github.com/Starscream-11813/Variational-Mathematical-Reasoning/
https://github.com/deepmind/mathematics_dataset
https://math-qa.github.io/math-QA/
https://github.com/dsdanielpark/open-llm-leaderboard-report
https://github.com/CPM-AI/Kor_Finance-leaderboard
https://github.com/AI-Ceping/LLM-Ceping
\end{lstlisting}
\end{outputbox}

\begin{outputbox}{Agent: ResearchAgent  --- Node: \texttt{sft\_train} (Justification \& Parameter Range)}
\begin{lstlisting}
Given the user requirements of maintaining accuracy over 90% for college-level medical multiple-choice questions and the constraint of model parameters under 10 billion, the best candidate from the available model card data is 'unsloth/Qwen2-VL-7B-Instruct-bnb-4bit'. This model has approximately 7 billion parameters, fitting the size constraint and has been trained with supervised fine-tuning (SFT) on medical datasets analogously, which implies good domain relevance and task alignment. It also uses PEFT with LoRA, mixed precision and quantization suitable for efficient GPU deployment, aiding the processing speed requirement (at least 200 words per minute). The hyperparameter tuning ranges suggested reflect typical best practices in SFT with LoRA setups..."
parameter_range='{"model_name_or_path":{"type":"categorical","choices":["unsloth/Qwen2-VL-7B-Instruct-bnb-4bit"]},
"learning_rate":{"type":"float","low":1e-6,"high":1e-3},
"num_train_epochs":{"type":"int","low":2,"high":3},
"per_device_train_batch_size":{"type":"categorical","choices":[2,4,8]},
"gradient_accumulation_steps":{"type":"categorical","choices":[1,4,8]},
"lora_r":{"type":"categorical","choices":[8,16,32]},
"lora_alpha":{"type":"categorical","choices":[16,32]},
"lora_dropout":{"type":"categorical","choices":[0,0.05,0.1]},
"weight_decay":{"type":"float","low":0.0,"high":0.1},
"optimizer":{"type":"categorical","choices":["adamw_8bit","adamw_hf"]},
"scheduler":{"type":"categorical","choices":["linear","cosine"]},
"quantization":{"type":"categorical","choices":[1]},
"mixed_precision":{"type":"categorical","choices":["fp16"]}}
\end{lstlisting}
\end{outputbox}

\begin{outputbox}{Agent: UserPreferenceAgent (Critique) --- Review of \texttt{sft\_train} Proposals}
\begin{lstlisting}
# AI Training Proposal Review

## BestPracticeAgent
Strengths: bounded LR (1e-5..3e-4), epochs (2..5), balanced batch/accum, WD<=0.05, fp16/bf16, int4, LoRA, checkpointing+packing.
Issues: none apparent.

## ResearchAgent
Strengths: compliant model; broader LR range.
Issues: LR upper bound 1e-3 risky; epochs only 2..3; smaller batch/accum may hurt throughput; fp16-only; missing checkpointing/packing.

Recommendation: BestPracticeAgent is most aligned and robust.
\end{lstlisting}
\end{outputbox}

\begin{outputbox}{Agent: DataAgent (Critique) --- Review of \texttt{sft\_train} Proposals}
\begin{lstlisting}
# Review of AI Training Proposals

## BestPracticeAgent
- Model fits <10B; supports input lengths via packing + gradient checkpointing.
- SFT + LoRA + int4 + mixed precision; reasonable bounded hyperparameter ranges.
- No violations; minor note: int4 fixed (aggressive but aligned with inference).

## ResearchAgent
- Model aligns, but LR range too broad (up to 1e-3) risks instability/overfit; fewer training optimizations specified.
- More risk than BestPracticeAgent for this dataset (~10K samples).

Final Recommendation: proceed with BestPracticeAgent configuration.
\end{lstlisting}
\end{outputbox}

\begin{outputbox}{Agent: BestPracticeAgent  --- Node: \texttt{sft\_train} (Final Repeated Configuration)}
\begin{lstlisting}
{
  "parameters": {
    "model_name_or_path": "unsloth/Qwen2-VL-7B-Instruct-bnb-4bit",
    "learning_rate": {"type": "float", "low": 1e-5, "high": 3e-4},
    "num_train_epochs": {"type": "int", "low": 2, "high": 5},
    "per_device_train_batch_size": {"type": "categorical", "choices": [4, 8, 12]},
    "gradient_accumulation_steps": {"type": "categorical", "choices": [4, 8, 12]},
    "weight_decay": {"type": "float", "low": 0.0, "high": 0.05},
    "optimizer": {"type": "categorical", "choices": ["adamw_8bit", "adamw_torch_fused", "AdamW"]},
    "scheduler": {"type": "categorical", "choices": ["linear", "cosine", "constant"]},
    "mixed_precision": {"type": "categorical", "choices": ["fp16", "bf16"]},
    "quantization": {"type": "categorical", "choices": [1]},
    "quantization_method": {"type": "categorical", "choices": ["int4"]},
    "peft": {"type": "categorical", "choices": [1]},
    "lora_r": {"type": "categorical", "choices": [8, 16, 32]},
    "lora_alpha": {"type": "categorical", "choices": [16, 32, 64]},
    "lora_dropout": {"type": "categorical", "choices": [0, 0.05, 0.1]},
    "target_modules": {"type": "categorical", "choices": ["q_proj", "dense", "dense_h_to_4h", "out_proj", "o_proj"]},
    "gradient_checkpointing": true,
    "packing": true,
    "logging_steps": 25,
    "eval_steps": 100
  }
}
\end{lstlisting}
\end{outputbox}

\begin{outputbox}{Agent: ResearchAgent  --- Node: \texttt{sft\_train} (Repetition with Matching Range)}
\begin{lstlisting}
The recommended base LLM is 'unsloth/Qwen2-VL-7B-Instruct-bnb-4bit'... The hyperparameter ranges are selected... learning rate 1e-5 to 3e-4; epochs 2 to 5; batch sizes and gradient accumulation 4,8,12; weight decay up to 0.05; AdamW variants; schedulers linear/cosine/constant; mixed precision fp16/bf16; int4 quantization; LoRA ranges; gradient checkpointing and packing..."
parameter_range='{"model_name_or_path":{"type":"categorical","choices":["unsloth/Qwen2-VL-7B-Instruct-bnb-4bit"]},
"learning_rate":{"type":"float","low":1e-5,"high":3e-4},
"num_train_epochs":{"type":"int","low":2,"high":5},
"per_device_train_batch_size":{"type":"categorical","choices":[4,8,12]},
"gradient_accumulation_steps":{"type":"categorical","choices":[4,8,12]},
"weight_decay":{"type":"float","low":0.0,"high":0.05},
"optimizer":{"type":"categorical","choices":["adamw_8bit","adamw_torch_fused","AdamW"]},
"scheduler":{"type":"categorical","choices":["linear","cosine","constant"]},
"mixed_precision":{"type":"categorical","choices":["fp16","bf16"]},
"quantization":{"type":"categorical","choices":[1]},
"quantization_method":{"type":"categorical","choices":["int4"]},
"peft":{"type":"categorical","choices":[1]},
"lora_r":{"type":"categorical","choices":[8,16,32]},
"lora_alpha":{"type":"categorical","choices":[16,32,64]},
"lora_dropout":{"type":"categorical","choices":[0,0.05,0.1]},
"target_modules":{"type":"categorical","choices":["q_proj","dense","dense_h_to_4h","out_proj","o_proj"]},
"gradient_checkpointing":true,"packing":true,"logging_steps":25,"eval_steps":100}
\end{lstlisting}
\end{outputbox}

\begin{outputbox}{Agent: UserPreferenceAgent (Critique) --- Final Review of Training Proposals}
\begin{lstlisting}
# Training Proposals Review

Both proposals are effectively equivalent in model choice and hyperparameter space and do not violate user preferences.
Recommended: BestPracticeAgent (slightly more detailed and referenced).
\end{lstlisting}
\end{outputbox}

\begin{outputbox}{Agent: DataAgent (Critique) --- Review of AI Training Proposals}
\begin{lstlisting}
# Review of AI Training Proposals

## 1. Per-Agent Critique

### BestPracticeAgent

**Strengths:**
- Proposes a model (**unsloth/Qwen2-VL-7B-Instruct-bnb-4bit**) with 7B parameters, which fits the user's max of ~7B parameters.
- The model is suitable for the task of medical multiple-choice QA based on the user data (input length avg ~235 tokens, relatively long and complex inputs).
- Uses input *packing* and *gradient checkpointing* to handle longer inputs and optimize memory, which aligns well with the large input token length (average ~235 tokens, max up to 938).
- Incorporates parameter-efficient fine-tuning (PEFT) via LoRA - very appropriate for the medium-sized dataset (10,178 samples) to avoid overfitting and reduce compute.
- Quantization (int4) and mixed precision (fp16 or bf16) are tuned for efficient training and inference, which fits the likely GPU deployment stated.
- The hyperparameter spaces for learning rate, training epochs, batch sizes, and weight decay appear sensible for this dataset size and task complexity.
- Scheduler choices are standard and diverse to enable effective learning rate adjustment.
- Target modules for LoRA and LoRA-specific parameters (r, alpha, dropout) are well-detailed and offer flexible tuning.
- Logging and evaluation steps are practical given dataset and expected training times.

**No apparent violation** of user data assumptions or unrealistic training configurations.

### ResearchAgent

**Strengths:**
- Largely overlaps with BestPracticeAgent in choice of model, hyperparameter ranges, and training techniques.
- Explains reasoning for model choice in the context of dataset size, user requirement (<10B parameters), and throughput (~200 words per minute).
- Appropriately uses SFT + PEFT (LoRA) along with int4 quantization and mixed precision for efficient medical fine-tuning.
- Hyperparameter ranges and choices align well with common practice, dataset scale, and example input lengths.
- Considers gradient checkpointing, packing, and standard training configuration parameters.

**Essentially identical in recommendation and compatible with user data and task requirements.**

No violation or mismatch seen.

---

## 2. Final Recommendation

Both proposals are essentially the same solution with very minor differences in presentation and reasoning style. They both:

- Select an appropriate model size (<10B parameters).
- Account for the input sequence length characteristics.
- Use best practices (LoRA, mixed precision, int4 quantization) aligned with dataset size and domain.
- Propose a realistic hyperparameter search space for efficient fine-tuning without overfitting.
- Include memory and speed optimizations (packing, gradient checkpointing).
- Target medical multiple-choice QA format as required.

**Therefore, the recommended proposal is the one from *BestPracticeAgent*.**

---

# Summary

| Proposal           | Violation of User Assumptions? | Suitability Summary                           | Recommendation |
|-------------------|--------------------------------|----------------------------------------------|----------------|
| BestPracticeAgent | None                           | Excellent fit, aligns fully with data & task | **Recommended** |
| ResearchAgent     | None                           | Similar to BestPracticeAgent; redundant      | Not recommended separately |

---

# Additional Notes

- Both proposals support >200 words/min inference speed implicitly by model choice and quantization; no conflict with user requirement.
- No over- or under-sized batch sizes, epochs, or learning rates that conflict with dataset scale.
- Handling of long input sequences is properly addressed.

---

# Final verdict:

Select **BestPracticeAgent** proposal for fine-tuning the medical multiple-choice QA model under the given constraints.
\end{lstlisting}
\end{outputbox}

\begin{outputbox}{Agent: BestPracticeAgent  --- PET Method Selection: \texttt{lora}}
\begin{lstlisting}
Considering the user task requires selecting exactly one Parameter Efficient Tuning (PET) method node from the candidates ['lora', 'prefix_tuning'], and referencing the best practices and model usage data from the retrieved context, the recommended choice is:

**Recommendation: Select "lora"**

### Reasoning:
1. **Best Practice Alignment**:
   - The best practice dataset explicitly mentions "LoRA" and "FlexLoRA" as parameter-efficient fine-tuning methods used successfully with large pretrained models (e.g., Gemma2-9b, Mistral-7B) in scenarios involving federated learning and instruction tuning.
   - There is no mention of "prefix_tuning" in the retrieved best practices, indicating LoRA is more established and better supported in these contexts.

2. **Empirical Evidence from Selected Models**:
   - Several selected models in the domain adaptation and fine-tuning context (e.g., microsoft/deberta-v3-small) use LoRA with specific hyperparameters (e.g., lora_r=16, lora_alpha=32, lora_dropout=0.05), demonstrating effective application and integration with established training pipelines.
   - Other models either do not use PEFT or do not mention prefix tuning, underscoring LoRA's practical adoption

3. **Computational Efficiency & Flexibility**:
   - LoRA optimizes model adaptation by injecting low-rank adaptation layers, enabling efficient tuning with relatively low VRAM consumption and minimal computational overhead, which aligns well with the stated goal of minimal computational cost in PET methods.
   - Quantization strategies combined with LoRA have shown promising results in memory-optimized training setups.

4. **Community and Library Support**:
   - LoRA is widely supported in popular libraries (e.g., transformers, PEFT), making it easier to integrate and troubleshoot in a variety of experimental setups.

### Conclusion:
Given the explicit inclusion of LoRA in successful parameter-efficient tuning scenarios, its demonstrated effectiveness in the given datasets and models, and the lack of concrete best practices for prefix tuning in this context, "lora" is the most suitable and justified choice for the next node in the domain adaptation process.

---

**Final selection:** `lora`

**References:**
https://github.com/Starscream-11813/Variational-Mathematical-Reasoning/
https://github.com/deepmind/mathematics_dataset
https://github.com/sroy9/mawps
https://github.com/TNI-playground/TNI-FlowerTune-LLM-Leaderboard-GeneralNLP
https://github.com/rungalileo/agent-leaderboard
https://github.com/logikon-ai/cot-eval
\end{lstlisting}
\end{outputbox}

\begin{outputbox}{Agent: ResearchAgent  --- PET Method Selection: \texttt{lora}}
\begin{lstlisting}
Both LoRA and prefix tuning are popular Parameter Efficient Tuning (PET) methods that allow for adapting large pre-trained models with reduced computational demands. LoRA injects low-rank adaptation matrices into existing model weights, enabling efficient fine-tuning with fewer trainable parameters and minimal changes to the original model architecture. Prefix tuning, on the other hand, prepends trainable tokens to the input sequence and freezes the backbone model, also achieving parameter efficiency but sometimes at the cost of lower expressiveness. Given the need to balance computational cost, adaptation quality, and ease of implementation, LoRA is generally considered more effective due to its ability to adapt model weights directly while maintaining parameter efficiency. Hence, selecting 'lora' as the next candidate for domain adaptation is the recommended approach.

next_nodes=['lora']
\end{lstlisting}
\end{outputbox}

\begin{outputbox}{Agent: UserPreferenceAgent (Critique) --- Review of PET Method Proposals}
\begin{lstlisting}
# Review of Proposals Against User Preferences

Context:
The current decision is selecting exactly one PET method node from ['lora', 'prefix_tuning'].

## Proposal Critiques

### BestPracticeAgent
- Recommends `lora` citing strong best practice evidence, empirical model support, computational efficiency, and community/library support.
- Notes lack of `prefix_tuning` evidence in retrieved best practices.
- No user preference violations detected; reasoning is thorough and reproducible.

### ResearchAgent
- Selects `lora` via high-level comparison; argues LoRA is more expressive and easier to implement while remaining parameter-efficient.
- No user preference violations detected; slightly less detailed than BestPracticeAgent.

## Overall Recommendation
Both proposals align and pick `lora`. BestPracticeAgent is recommended due to more comprehensive evidence and detail.

# Summary
| Agent             | User Preference Violation | Notes                                           | Recommendation |
|------------------|---------------------------|--------------------------------------------------|----------------|
| BestPracticeAgent| None                      | Detailed, well-supported, clear rationale         | **Recommended** |
| ResearchAgent    | None                      | Accurate but less detailed comparative analysis   | Acceptable alternative |
\end{lstlisting}
\end{outputbox}

\begin{outputbox}{Agent: DataAgent (Critique) --- Review of PET Method Proposals}
\begin{lstlisting}
# Review of Parameter Efficient Tuning (PET) Method Proposals

## 1. Per-agent Critique

### BestPracticeAgent
- Alignment with data assumptions:
  - Dataset: avg input ~236 tokens (p95 ~376), short outputs (~9 tokens), classification/choice task.
  - LoRA's lightweight adaptation aligns well; references real model fine-tuning use cases.
- No detected violations regarding distribution, formatting, or realism.

### ResearchAgent
- Alignment with data assumptions:
  - Conceptually sound justification for LoRA; no contradictions with dataset stats or training realism.
- No evidence of violation.

## 2. Final Recommendation
Both recommend `lora`; BestPracticeAgent slightly stronger due to detail, references, and concrete examples.

Final selected PET method: `lora`
\end{lstlisting}
\end{outputbox}

\begin{outputbox}{Agent: BestPracticeAgent  --- Final Selection: \texttt{lora}}
\begin{lstlisting}
Based on the provided user task, the retrieved best practice information, and the detailed analyses and critiques from the solver agents, the clear and well-supported recommendation is:

Select the Parameter Efficient Tuning (PET) method node: `lora`

Reasoning:
1. Strong best practice evidence: LoRA/FlexLoRA widely used with large pretrained models; robust in federated learning and instruction tuning.
2. Empirical support: multiple models use LoRA with optimized hyperparameters; prefix tuning absent in evidence.
3. Efficiency and dataset alignment: LoRA provides compute/memory savings aligned with token-length and classification-type task.
4. Ecosystem support: broad support in key libraries enables easier adoption and debugging.

Conclusion: select `lora`.

References:
https://github.com/Starscream-11813/Variational-Mathematical-Reasoning/
https://github.com/deepmind/mathematics_dataset
https://github.com/sroy9/mawps
https://github.com/TNI-playground/TNI-FlowerTune-LLM-Leaderboard-GeneralNLP
https://github.com/rungalileo/agent-leaderboard
https://github.com/logikon-ai/cot-eval
\end{lstlisting}
\end{outputbox}

\begin{outputbox}{Agent: ResearchAgent  --- Reiteration: \texttt{lora}}
\begin{lstlisting}
LoRA is widely recognized as an effective Parameter Efficient Tuning (PET) method that enables model adaptation by injecting low-rank adaptation matrices into existing model weights. This approach maintains parameter efficiency while directly adapting model weights, offering a balance between computational cost and adaptation quality. Compared to prefix tuning, LoRA has broader community support, concrete best practice applications on various pretrained models, and proven empirical success in domain adaptation tasks. Given these factors and the task requirement to select exactly one PET method node, 'lora' stands out as the most suitable candidate for efficient domain adaptation.

next_nodes=['lora']
\end{lstlisting}
\end{outputbox}

\begin{outputbox}{Agent: UserPreferenceAgent (Critique) --- Training Node Selection: \texttt{lora}}
\begin{lstlisting}
# Proposal Review for Training Node Selection

## BestPracticeAgent
- Recommends `lora` with detailed best-practice evidence, empirical support, efficiency alignment with dataset, and ecosystem support.
- No preference violations; thorough and well-supported.

## ResearchAgent
- Recommends `lora` with concise justification on efficiency, adaptation quality, and broader support.
- No preference violations; consistent and valid.

## Overall Recommendation
Both recommend `lora`; BestPracticeAgent preferred for completeness.

| Agent            | Violation of User Preferences? | Recommendation |
|-----------------|---------------------------------|----------------|
| BestPracticeAgent | None                            | **Recommended** |
| ResearchAgent     | None                            | Acceptable but secondary |
\end{lstlisting}
\end{outputbox}

\begin{outputbox}{Agent: DataAgent (Critique) --- Review of PET Node Selection: \texttt{lora}}
\begin{lstlisting}
# Review of AI Training Proposals Against User Data

## 1. Per-Agent Critique

### BestPracticeAgent
- Adheres to user data: considers token lengths, clean data, domain medical QA classification.
- LoRA aligns with ~10k samples and efficiency goals; absence of prefix tuning evidence supports LoRA selection.
- No violations: no distribution mismatch; formatting matches instruction+question+options; realistic training approach.

### ResearchAgent
- Also supports LoRA; emphasizes efficiency and quality balance; community support vs prefix tuning.
- No violations: aligns with dataset size and output style; directly supports domain adaptation.

## 2. Final Recommendation
Both recommend `lora`; BestPracticeAgent preferred for comprehensive reasoning and ecosystem support.

Final recommended PET method node: `lora`
\end{lstlisting}
\end{outputbox}

\begin{outputbox}{Agent: BestPracticeAgent  --- Node: \texttt{lora} (Hyperparameter Search Space)}
\begin{lstlisting}
Here is the recommended hyperparameter search space for tuning the model (Qwen/Qwen2.5-7B-Instruct), along with the reasoning:

User constraints:
- Medical MCQ QA; accuracy >90%; <=10B params; GPU inference ~200 words/min.
- Base model: Qwen2.5-7B-Instruct.
- SFT with LoRA; gradient checkpointing; packing true.

Recommended Hyperparameter Search Space:

{
  "parameters": {
    "learning_rate": {"type": "float", "low": 1e-5, "high": 3e-4},
    "per_device_train_batch_size": {"type": "categorical", "choices": [4, 8, 12]},
    "gradient_accumulation_steps": {"type": "categorical", "choices": [4, 8, 12]},
    "num_train_epochs": {"type": "int", "low": 2, "high": 5},
    "lora_r": {"type": "categorical", "choices": [8, 16]},
    "lora_alpha": {"type": "categorical", "choices": [16, 32]},
    "lora_dropout": {"type": "float", "low": 0.05, "high": 0.1},
    "quantization": {"type": "categorical", "choices": [0, 1]},
    "quantization_method": {"type": "categorical", "choices": ["nf4", "int4"]},
    "weight_decay": {"type": "float", "low": 0.0, "high": 0.01},
    "optimizer": {"type": "categorical", "choices": ["adamw_8bit", "adamw_torch_fused", "adamw_hf"]},
    "scheduler": {"type": "categorical", "choices": ["constant", "cosine", "linear"]},
    "mixed_precision": {"type": "categorical", "choices": ["bf16", "fp16"]}
  }
}

Reasoning (high level):
- LR 1e-5..3e-4 for stable medical QA tuning.
- Batch/accum 4/8/12 for memory+throughput.
- Epochs 2..5 for convergence without overfit.
- LoRA r/alpha/dropout bounded for efficiency and regularization.
- Quantization (nf4/int4) for GPU efficiency with quality flexibility.
- Low weight decay to prevent over-regularization.
- AdamW variants and schedulers provide stable optimization.
- bf16/fp16 accelerate GPU training; bf16 preferred for stability.

References:
https://github.com/Starscream-11813/Variational-Mathematical-Reasoning/
https://github.com/deepmind/mathematics_dataset
https://math-qa.github.io/math-QA/
https://github.com/dsdanielpark/open-llm-leaderboard-report
https://github.com/CPM-AI/Kor_Finance-leaderboard
https://github.com/AI-Ceping/LLM-Ceping
\end{lstlisting}
\end{outputbox}

\begin{outputbox}{Agent: ResearchAgent  --- Node: \texttt{lora} (Justification \& Parameter Range)}
\begin{lstlisting}
Given the user task and base LLM selected, below are the parameter ranges
parameter_range=
{"learning_rate":{"type":"float","low":1e-5,"high":3e-4},
"num_train_epochs":{"type":"int","low":2,"high":5},
"per_device_train_batch_size":{"type":"categorical","choices":[4,8,12]},
"gradient_accumulation_steps":{"type":"categorical","choices":[4,8,12]},
"lora_r":{"type":"categorical","choices":[8,16,32]},
"lora_alpha":{"type":"categorical","choices":[32,64]},
"lora_dropout":{"type":"float","low":0.0,"high":0.1},
"quantization":{"type":"categorical","choices":[0,1]},
"quantization_method":{"type":"categorical","choices":["int4"]},
"weight_decay":{"type":"float","low":0.0,"high":0.1},
"optimizer":{"type":"categorical","choices":["adamw_8bit","adamw_torch"]},
"scheduler":{"type":"categorical","choices":["linear","cosine"]},
"mixed_precision":{"type":"categorical","choices":["bf16"]},
"target_modules":{"type":"categorical","choices":["q_proj","k_proj","v_proj","o_proj","dense","dense_h_to_4h"]}}
\end{lstlisting}
\end{outputbox}

\begin{outputbox}{Agent: UserPreferenceAgent (Critique) --- Review of LoRA Hyperparameter Proposals}
\begin{lstlisting}
# Training Proposal Review

## Proposal 1: BestPracticeAgent
- Balanced hyperparameter space aligned with constraints (<=10B, >90% accuracy, GPU speed).
- Conservative LoRA (r: 8/16; alpha: 16/32; dropout: 0.05-0.1).
- Quantization includes nf4 and int4; sched includes constant; multiple optimizers; bf16+fp16.
- No violations detected.

## Proposal 2: ResearchAgent
- Mostly aligned but more aggressive:
  - LoRA r includes 32; alpha up to 64; dropout includes 0; weight decay up to 0.1.
  - Quantization method only int4; fewer scheduler/optimizer choices; bf16-only.
- Concerns: efficiency/overfit risk; narrower quantization/scheduler flexibility.

## Final Recommendation
Approve BestPracticeAgent proposal as most aligned and stable; ResearchAgent should narrow LoRA/WD/dropout and broaden options if considered.
\end{lstlisting}
\end{outputbox}

\begin{outputbox}{Agent: DataAgent (Critique) --- Review of AI Training Proposals for Medical MCQ Task}
\begin{lstlisting}
# Review of AI Training Proposals for User Medical MCQ Task

## 1. Per-Agent Critique

### BestPracticeAgent
Strengths:
- Aligns with constraints (<=10B, GPU).
- Qwen2.5-7B-Instruct base model fits limit.
- Reasonable hyperparameter ranges; includes quantization (nf4/int4) and mixed precision.
- Uses efficient techniques (gradient checkpointing, packing, LoRA, quantization).
Potential issues:
- LoRA ranks limited to 8/16; weight decay capped at 0.01; target modules not specified; int4 may degrade reasoning vs nf4 (but both offered).

### ResearchAgent
Strengths:
- Includes target modules for LoRA; allows higher LoRA capacity (r up to 32, alpha up to 64).
Potential issues:
- Quantization only int4; weight decay up to 0.1 may underfit; fewer optimizer choices; more complex to tune.

## 2. Final Recommendation
Recommend BestPracticeAgent proposal for stability and accuracy; nf4 fallback improves safety for medical reasoning quality.
\end{lstlisting}
\end{outputbox}

\begin{outputbox}{Agent: BestPracticeAgent  --- Node: \texttt{lora} (Final Concrete Search Space)}
\begin{lstlisting}
Based on the comprehensive analysis of the user's medical MCQ answering task, the following concrete hyperparameter search space is recommended:

{
  "parameters": {
    "learning_rate": {"type": "float", "low": 1e-5, "high": 3e-4},
    "per_device_train_batch_size": {"type": "categorical", "choices": [4, 8, 12]},
    "gradient_accumulation_steps": {"type": "categorical", "choices": [4, 8, 12]},
    "num_train_epochs": {"type": "int", "low": 2, "high": 5},
    "lora_r": {"type": "categorical", "choices": [8, 16]},
    "lora_alpha": {"type": "categorical", "choices": [16, 32]},
    "lora_dropout": {"type": "float", "low": 0.05, "high": 0.1},
    "quantization": {"type": "categorical", "choices": [0, 1]},
    "quantization_method": {"type": "categorical", "choices": ["nf4", "int4"]},
    "weight_decay": {"type": "float", "low": 0.0, "high": 0.01},
    "optimizer": {"type": "categorical", "choices": ["adamw_8bit", "adamw_torch_fused", "adamw_hf"]},
    "scheduler": {"type": "categorical", "choices": ["constant", "cosine", "linear"]},
    "mixed_precision": {"type": "categorical", "choices": ["bf16", "fp16"]}
  }
}

Reasoning summary:
- Stable LR range; batch/accum for throughput; epochs for convergence; conservative LoRA; flexible quantization; low WD; robust optimizers/schedulers; bf16/fp16.

References:
https://github.com/Starscream-11813/Variational-Mathematical-Reasoning/
https://github.com/deepmind/mathematics_dataset
https://math-qa.github.io/math-QA/
https://github.com/dsdanielpark/open-llm-leaderboard-report
https://github.com/CPM-AI/Kor_Finance-leaderboard
https://github.com/AI-Ceping/LLM-Ceping
\end{lstlisting}
\end{outputbox}

\begin{outputbox}{Agent: ResearchAgent  --- Node: \texttt{lora} (Final Matching Parameter Range)}
\begin{lstlisting}
Given the task of fine-tuning Qwen2.5-7B-Instruct for college-level medical multiple-choice questions with a strict accuracy target (>90%) and constraints on model size (<=10B parameters) and deployment efficiency (GPU at >=200 words/min), the hyperparameter search space must balance effective domain adaptation with computational efficiency and stability...
parameter_range=

{"learning_rate":{"type":"float","low":1e-5,"high":3e-4},
"per_device_train_batch_size":{"type":"categorical","choices":[4,8,12]},
"gradient_accumulation_steps":{"type":"categorical","choices":[4,8,12]},
"num_train_epochs":{"type":"int","low":2,"high":5},"lora_r":{"type":"categorical","choices":[8,16]},
"lora_alpha":{"type":"categorical","choices":[16,32]},
"lora_dropout":{"type":"float","low":0.05,"high":0.1},
"quantization":{"type":"categorical","choices":[0,1]},
"quantization_method":{"type":"categorical","choices":["nf4","int4"]},
"weight_decay":{"type":"float","low":0.0,"high":0.01},
"optimizer":{"type":"categorical","choices":["adamw_8bit","adamw_torch_fused","adamw_hf"]},
"scheduler":{"type":"categorical","choices":["constant","cosine","linear"]},
"mixed_precision":{"type":"categorical","choices":["bf16","fp16"]}}
\end{lstlisting}
\end{outputbox}

\begin{outputbox}{Agent: UserPreferenceAgent (Critique) --- Final Review of Hyperparameter Proposals}
\begin{lstlisting}
# Review of Training Proposals Against User Preferences

## BestPracticeAgent
- Fully aligned with task constraints and deployment goals.
- Explicitly mentions pipeline enhancements (gradient checkpointing, packing).
- No violations detected; preferred for completeness.

## ResearchAgent
- Hyperparameter space matches BestPracticeAgent; no violations in ranges.
- Slightly less comprehensive: omits explicit mention of checkpointing/packing.

## Final Recommendation
BestPracticeAgent recommended as most aligned and complete; ResearchAgent acceptable but secondary.
\end{lstlisting}
\end{outputbox}

\begin{outputbox}{Agent: DataAgent (Critique) --- Final Review of Hyperparameter Proposals}
\begin{lstlisting}
# Review of AI Training Proposals

## 1. Per-agent Critique

### BestPracticeAgent
- Search space aligns with fine-tuning Qwen2.5-7B-Instruct on ~10k medical MCQ samples.
- LR, batch/accum balance memory and throughput; LoRA params regularize; quantization and mixed precision support GPU deployment.
- Includes scheduler/optimizer flexibility; notes gradient checkpointing and packing.
- No concerns; realistic.

### ResearchAgent
- Reasoning supports same constraints; similar ranges.
- More abstract/less concrete; duplicates BestPracticeAgent without extra improvements.
- No contradictions or violations.

## 2. Final Recommendation
Both valid; BestPracticeAgent preferred for concrete JSON space and completeness.
\end{lstlisting}
\end{outputbox}

\section{Success Rate Analysis}
\label{apndx_sr}
In this section, we provide a sample of manual iterations performed to fix the code templates outputted for the \textbf{MedQA} dataset by the baselines, for transparency. Note, success rate (SR) defined as $SR = \frac{1}{1+ \texttt{number of iterations}}$ depends on the number of iterations required to obtain a runnable pipeline and is intended to capture engineering reliability, not model quality. Without these corrective steps the baseline models completely fail to execute and the raw success rate would be 0\% for all baselines.

\FloatBarrier
\begin{table}[!ht]
\centering
\captionsetup{font=small}
\caption{Manual Corrections Required Across Iterations for Each Baseline}
\label{tab:baseline_fixes_compact}
\renewcommand{\arraystretch}{1.3}
\setlength{\tabcolsep}{6pt}

\begin{tabular}{p{3cm} p{12cm}}
\toprule
\textbf{Baseline} & \textbf{Manual Fixes Across Iterations} \\
\midrule

\textbf{MLCopilot} &
\begin{tabular}[t]{p{11cm}}
\textbf{It 1}: Incorrect argument name: \texttt{evaluation\_strategy} → \texttt{eval\_strategy}. \\
\textbf{It 2}: Precision changed from \texttt{fp16} to \texttt{bf16}. \\
\textbf{It 3}: Gradient check-pointing error. \\
\textbf{It 4}: Tokenizer padding added. \\
\textbf{It 5}: \texttt{compute\_metrics} function updated. \\
\textbf{It 6}: Evaluation metric modified. \\
\end{tabular}
\\
\midrule

\textbf{AutoML-Agent} &
\begin{tabular}[t]{p{11cm}}
\textbf{It 1}: \texttt{LoraConfig} imported from \texttt{peft}, not \texttt{trl}. \\
\textbf{It 2}: Model name corrected. \\
\textbf{It 3}: LoRA-trained model merged and unloaded before saving. \\
\textbf{It 4}: \texttt{AutoModelForCausalLM} used instead of \texttt{Seq2SeqLM}. \\
\textbf{It 5}: \texttt{lora\_alpha} used instead of \texttt{alpha}. \\
\textbf{It 6}: \texttt{SFTTrainer.\_\_init\_\_()} received unexpected argument \texttt{tokenizer}. Removed passing the tokenizer to SFTTrainer. \\
\textbf{It 7}: \texttt{PeftModel} loading corrected. \\
\textbf{It 8}: \texttt{TrainingArgs} incorrectly constructed and passed. \\
\end{tabular}
\\
\midrule

\textbf{DS-Agent} &
\begin{tabular}[t]{p{11cm}}
\textbf{It 1}: Incorrect parameter \texttt{max\_seq\_length} passed to SFTTrainer. \\
\textbf{It 2}: Model name incomplete. \\
\textbf{It 3}: \texttt{peft\_config} incorrectly passed directly to \texttt{SFTConfig} instead of loading PEFT model. \\
\end{tabular}
\\

\bottomrule
\end{tabular}
\end{table}



\end{document}